\documentclass[journal]{IEEEtran}
\usepackage{amsmath,amsthm,amsfonts,amssymb}
\usepackage{algorithm}
\usepackage{algpseudocode}
\usepackage[caption=false,font=footnotesize,labelfont=rm,textfont=rm]{subfig}
\usepackage{textcomp}
\usepackage{stfloats}
\usepackage{url}
\usepackage{verbatim}
\usepackage{graphicx}
\usepackage{cite}
\usepackage{hyperref}
\usepackage{bm}
\usepackage{latexsym}
\usepackage{tikz}
\usepackage{booktabs,threeparttable,adjustbox,makecell,multirow,array}
\usepackage{orcidlink}
\usepackage{mathtools}

\def\x{{\mathbf x}}
\def\y{{\mathbf y}}
\def\z{{\mathbf z}}
\def\f{{\mathbf{f}}}
\def\u{{\mathbf u}}
\def\m{{\mathbf m}}
\def\vx{{\vec{\mathbf x}}}
\def\vy{{\vec{\mathbf y}}}
\def\vz{{\vec{\mathbf z}}}
\def\vf{{\vec{\mathbf f}}}
\def\vu{{\vec{\mathbf u}}}

\def\cX{{\cal X}}

\def\cD{{\cal D}}

\def\cN{{\cal N}}

\def\bzeta{{\bm{\zeta}}}
\def\btheta{{{\bm{\theta}}}}

\newtheorem{theorem}{Theorem}
\newtheorem{lemma}{Lemma}

\newtheorem{remark}{Remark}
\newtheorem{definition}{Definition}
\newtheorem{proposition}{Proposition}

\definecolor{orange}{RGB}{200,0,100}

\newcommand{\cred}[1]{{\color{red}#1}}
\newcommand{\cblue}[1]{{\color{black}#1}}

\begin{document}
\title{
Ensemble Kalman Filtering Meets Gaussian Process SSM for Non-Mean-Field and Online Inference
}

\author{Zhidi Lin$^{\orcidlink{0000-0002-6673-511X}}$,
% ~\IEEEmembership{Graduate Student Member,~IEEE}, 
Yiyong Sun$^{\orcidlink{0009-0006-4823-0202}}$, Feng Yin$^{\orcidlink{0000-0001-5754-9246}}$,~\IEEEmembership{Senior Member,~IEEE}, and Alexandre Hoang Thiéry$^{\orcidlink{0000-0002-9542-509X}}$ 

        % <-this % stops a space
% \thanks{}% <-this % stops a space
% \thanks{Manuscript updated \today.}
\thanks{
Z. Lin is with the School of Science \& Engineering, and the Future Network of Intelligence Institute, The Chinese University of Hong Kong, Shenzhen, and also with the Shenzhen Research Institute of Big Data, Shenzhen 518172, China (email: \href{mailto:zhidilin@link.cuhk.edu.cn}{zhidilin@link.cuhk.edu.cn}). Y. Sun and F. Yin are with the School of Science \& Engineering, The Chinese University of Hong Kong, Shenzhen, Shenzhen 518172, China  (email: \href{mailto:yiyongsun@link.cuhk.edu.cn}{yiyongsun@link.cuhk.edu.cn}, \href{mailto:yinfeng@cuhk.edu.cn}{yinfeng@cuhk.edu.cn}).
% Z. Lin, Y. Sun and F. Yin are with the School of Science \& Engineering, The Chinese University of Hong Kong, Shenzhen, Shenzhen 518172, China  (email: \href{mailto:zhidilin@link.cuhk.edu.cn}{zhidilin@link.cuhk.edu.cn}, \href{mailto:yiyongsun@link.cuhk.edu.cn}{yiyongsun@link.cuhk.edu.cn}, \href{mailto:yinfeng@cuhk.edu.cn}{yinfeng@cuhk.edu.cn}).
A. H. Thiéry is with the Department of Statistics \& Data Science, National University of Singapore, Singapore 117546 (email: \href{mailto:a.h.thiery@nus.edu.sg}{a.h.thiery@nus.edu.sg}). 
\textit{Corresponding author: Feng Yin}.}
% \vspace{-.12in}
}

%% -------------------- The paper headers  --------------------

%% --------------------  Remember, if you use this you must call \IEEEpubidadjcol in the second column for its text to clear the IEEEpubid mark.  --------------------

\maketitle

\begin{abstract}
The Gaussian process state-space models (GPSSMs) represent a versatile class of data-driven nonlinear dynamical system models. However, the presence of numerous latent variables in GPSSM incurs unresolved issues for existing variational inference approaches, particularly under the more realistic non-mean-field (NMF) assumption, including extensive training effort, compromised inference accuracy, and infeasibility for online applications, among others. In this paper, we tackle these challenges by incorporating the ensemble Kalman filter (EnKF), a well-established model-based filtering technique, into the NMF variational inference framework to approximate the posterior distribution of the latent states. This novel marriage between EnKF and GPSSM not only eliminates the need for extensive parameterization in learning variational distributions, but also enables an interpretable, closed-form approximation of the evidence lower bound (ELBO). Moreover, owing to the streamlined parameterization via the EnKF, the new GPSSM model can be easily accommodated in online learning applications. We demonstrate that the resulting EnKF-aided online algorithm embodies a principled objective function by ensuring data-fitting accuracy while incorporating model regularizations to mitigate overfitting. We also provide detailed analysis and fresh insights for the proposed algorithms. Comprehensive evaluation across diverse real and synthetic datasets corroborates the superior learning and inference performance of our EnKF-aided variational inference algorithms compared to existing methods. 
\end{abstract} \vspace{-.05in}
\vspace{-.05in}
\section{Introduction}
\label{sec:introduction}  %\vspace{-.05in}
\IEEEPARstart{S}{tate-space} models (SSMs) describe the underlying dynamics of latent states through a transition function and an emission function \cite{sarkka2013bayesian}. As a versatile tool for modeling dynamical systems, SSM finds successful applications in diverse fields, including control engineering, signal processing, computer science, and economics \cite{kullberg2021online, khan2008distributing, tobar2015unsupervised}. In SSM, a key task is to infer unobserved latent states from a sequence of noisy measurements. Established techniques, such as the Kalman filter (KF), extended Kalman filter (EKF), ensemble Kalman filter (EnKF), and particle filter (PF), have been widely employed over the decades for latent state inference \cite{sarkka2013bayesian}. However, these classic state inference methods heavily rely on precise knowledge of the underlying system dynamics \cite{revach2022kalmannet}, posing  tremendous challenges in complex and uncertain scenarios like model-based reinforcement learning \cite{yan2020gaussian} and disease epidemic propagation \cite{alaa2019attentive}. As a novel alternative, the underlying complex dynamics can be learned from noisy observations, leading to the emergence of data-driven SSMs. One class of prominent data-driven SSMs is the Gaussian process state-space model (GPSSM) \cite{frigola2013bayesian}, which utilizes Gaussian processes (GPs) \cite{williams2006gaussian} as the core learning modules to capture the complex underlying system dynamics. 

Gaussian processes, serving as the most prominent Bayesian non-parametric model \cite{williams2006gaussian}, provide the flexibility to model general nonlinear system dynamics without enforcing an explicit parametric structure. With the inherent regularization imposed by the GP prior, GPSSMs are able to mitigate overfitting and model generalization issues \cite{zhao2019cramer}, rendering them to be more effective in scenarios with limited data samples \cite{yin2020linear}. Moreover, owing to its Bayesian nature, GPSSMs maintain good interpretability and explicit uncertainty calibration for analyzing system dynamics. These superior properties have led to the extensive usage of GPSSMs in various live
applications, such as human pose and motion learning \cite{wang2007gaussian}, robotics and control learning \cite{deisenroth2013gaussian}, reinforcement learning \cite{arulkumaran2017deep}, target tracking and navigation \cite{xie2020learning}, and magnetic-field sensing \cite{berntorp2023constrained}.

Despite the popularity of GPSSMs, simultaneously learning the model and estimating the latent states in GPSSMs remains highly challenging, primarily due to the following two factors. First, the inference quality of the numerous latent states affects the model learning and vice versa, leading to heightened computational and statistical complexities. Second, the nonlinearity in GPSSMs prohibits tractable learning and inference processes \cite{theodoridis2020machine,tuncer2022multi,wahlstrom2015extended}. Hence, the main task in GPSSM is to accurately approximate the joint posterior over the latent states and the system dynamics represented by GPs. Multiple approaches have emerged towards achieving this objective in the last decade.

The seminal work utilizing particle Markov chain Monte Carlo (PMCMC) was proposed in \cite{frigola2013bayesian}. Subsequently, more advanced methods \cite{svensson2016computationally, svensson2017flexible, berntorp2021online,liu2023sequential} leveraged the reduced-rank GP approximation introduced in \cite{solin2020hilbert} to alleviate the substantial computational demands of GPs in \cite{frigola2013bayesian}. However, the computational burden of the involved PMCMC remains unaffordable, particularly when dealing with \cblue{long and high-dimensional latent state trajectories}. Consequently, there has been a paradigm shift towards variational inference methods \cite{frigola2014variational, frigola2015bayesian, mchutchon2015nonlinear,eleftheriadis2017identification, doerr2018probabilistic,ialongo2019overcoming,curi2020structured,lindinger2022laplace,fan2023free}, which adopted the classic sparse GP approximations with inducing points \cite{titsias2009variational}. 

Generally, the variational inference approaches can be classified into two main categories: the mean-field (MF) class \cite{frigola2014variational, frigola2015bayesian, mchutchon2015nonlinear, eleftheriadis2017identification} versus the non-mean-field (NMF) class \cite{doerr2018probabilistic, ialongo2019overcoming, curi2020structured, lindinger2022laplace,fan2023free}, depending on whether the statistical independence assumption is applied to the variational distribution of the latent states and GP dynamics. The first variational algorithm, incorporating the MF assumption and integrating PF for latent state inference, was introduced in \cite{frigola2014variational}. Subsequent MF endeavors \cite{frigola2015bayesian,mchutchon2015nonlinear,eleftheriadis2017identification} aimed at reducing computational and model complexities in \cite{frigola2014variational}. While the MF class methods may enable more manageable computations, they neglect the inherent dependencies between the latent states and the GP dynamics, potentially impairing the overall learning accuracy and yielding overconfident state estimates \cite{turner+sahani:2011a}.

Recent methods \cite{doerr2018probabilistic,ialongo2019overcoming,curi2020structured,lindinger2022laplace,fan2023free}, therefore, concentrate on the more complicated NMF approximation to address the inaccuracies arising from the MF assumption, yet they continue to struggle with computational challenges.
% The NMF approximation has long been acknowledged as more accurate, yet it introduces intractability in the inference process due to the substantially complicated variational distributions. Although some NMF methods have been proposed recently, 
Specifically, in \cite{doerr2018probabilistic}, the posterior distribution of underlying states conditioned on the GP dynamics was simply approximated by the subjectively selected prior distribution, while in \cite{ialongo2019overcoming} and \cite{curi2020structured}, the posterior was approximated by some parametric Markov-structured Gaussian variational distributions. However, optimizing the extensive variational parameters in the variational distributions is computationally demanding and inefficient. In contrast to the heavy parameterization as conducted in \cite{doerr2018probabilistic,ialongo2019overcoming,curi2020structured}, the works in \cite{lindinger2022laplace} and \cite{fan2023free} opted for directly characterizing the posterior distribution through the model evidence. Nonetheless, the Laplace approximation, employed in \cite{lindinger2022laplace}, can yield an oversimplified unimodal posterior distribution over the latent states \cite{theodoridis2020machine}, while the stochastic gradient Hamiltonian Monte Carlo method \cite{chen2014stochastic}, employed in \cite{fan2023free}, may exhibit very slow convergence ($50,000$ iterations as reported).  Furthermore, all these existing variational methods are incompatible with online applications, posing a notable limitation in their applicability.

To tackle the computational challenges mentioned above, we propose integrating the model-based filtering technique, EnKF, into the NMF variational inference framework. The primary benefit of this approach lies in the streamlined parameterization of the associated variational distributions, improving inference tractability and accelerating algorithm convergence. Additionally, this streamlined parameterization enables seamless accommodation of our method in online learning applications.
% By leveraging the EnKF, our method can fully exploit the dependencies between the latent states and the GP dynamics, while reducing the parameterization of associated variational distributions, thereby alleviating model and computational complexities and accelerating algorithm convergence. This streamlined parameterization also helps the easy accommodation of our method in an online setting. 
The main contributions are summarized as follows.
\begin{itemize}
    \item We corroborate the significance of incorporating well-established model-based filtering techniques to enhance the efficiency of GPSSM variational algorithms. In contrast to existing methods utilizing numerous variational parameters to parameterize the variational distribution over latent states, our proposed novel algorithm harnesses the EnKF to accurately capture the dependencies between the latent states and the GP dynamics, and eliminate the need to parameterize the variational distribution, significantly reducing the number of variational parameters and accelerating the algorithm convergence.  
    \item By leveraging EnKF, we also demonstrate an approximation of the evidence lower bound (ELBO) via simply summating multiple interpretable terms with readily available closed-form solutions. Leveraging the differentiable nature of the classic EnKF alongside the off-shelf automatic differentiation tools, we can optimize the ELBO and train the GPSSM efficiently.
    \item Without explicitly parameterizing the variational distribution over latent states, the proposed EnKF-aided algorithm gets readily extended to accommodate online learning applications. The resulting online algorithm relies on a principled objective by ensuring observation fitting accuracy while incorporating model regularizations to mitigate model overfitting, rendering it superior to state-of-the-art online algorithms in learning and inference efficiency.
    \item We offer in-depth analysis and novel insights into the proposed algorithms, and conduct extensive experiments on diverse real and synthetic datasets to assess their performance from various aspects. The results demonstrate that the proposed EnKF-aided variational learning algorithms consistently outperform the existing state-of-the-art methods, especially in terms of learning and inference.
\end{itemize}

The remainder of this paper is organized as follows. Some preliminaries related to GPSSMs are provided in Section \ref{sec:preliminaries}. The  computational challenges of the existing methods are summarized in Section \ref{sec:vGPSSMs}.
Section \ref{sec:EnVI} elaborates the proposed EnKF-aided variational learning algorithm. Section \ref{sec:online_EnVI} extends the proposed algorithm to accommodate online learning applications. The experimental results are presented in Section \ref{sec:experimental_results}, and Section \ref{sec:conclusion} concludes this paper. 
% Some key technical proofs and derivations are given in Appendix and 
More supportive results, proofs, etc., are relegated to the Appendix and supplementary materials \cite{lin2023ensemble}.

\vspace{-.05in}
\section{Preliminaries} \label{sec:preliminaries} 
In Section \ref{subsec:GP}, we briefly review the Gaussian process regression. Section \ref{subsec:GPSSM} is dedicated to introducing Gaussian process state-space models. 
\vspace{-.06in}
\subsection{Gaussian Processes (GPs)}
\label{subsec:GP} 
\vspace{-.01in}
A GP defines a collection of random variables indexed by $\bm{x} \in \mathcal{X}$, such that any finite subset of these variables follows a joint Gaussian distribution \cite{williams2006gaussian}.  Mathematically,  a real scalar-valued GP $f(\bm{x})$ can be represented as 
\begin{equation}
    f(\bm{x}) \sim \mathcal{GP}\left(\mu(\bm{x}), \ k(\bm{x}, \bm{x}^\prime);  \ \bm{\theta}_{gp}\right),
\end{equation}
where $\mu({\bm{x}})$ is a mean function typically set to zero in practice, and $k(\bm{x}, \bm{x}^\prime)$ is the kernel function that provides insights about the nature of the underlying function \cite{williams2006gaussian}; and $\bm{\theta}_{gp}$ is a set of hyperparameters that needs to be tuned for model selection. By placing a GP prior over the function $f(\cdot): \cX \mapsto \mathbb{R}$ in a general regression model, 
\begin{equation}
	y = f(\bm{x}) + {e},  \quad 	{e} \sim \cN(0, \sigma_{e}^2), \quad y \in \mathbb{R},
	\label{eq:reg_model}
\end{equation}
we get the salient Gaussian process regression (GPR) model.
% The task in GPR model is to infer the mapping function $f(\cdot)$ using an observed dataset  $\mathcal{D}\triangleq\{\bm{x}_i, y_i\}_{i = 1}^{n} \triangleq \{\bm{X}, \bm{y}\}$ consisting of $n$ input-output pairs. 

Given an observed dataset, $\mathcal{D} \!\triangleq\! \{\bm{x}_i, y_i\}_{i = 1}^{n} \!\triangleq\! \{\bm{X}, \bm{y}\}$ consisting of $n$ input-output pairs, the posterior distribution of the mapping function,  $p(f(\bm{x}_*) \vert \bm{x}_*, \cD)$, at any test input $\bm{x}_* \!\in \! \cX$, is Gaussian \cite{williams2006gaussian}, fully characterized by the posterior mean $\xi$ and the posterior variance $\Xi$.  Concretely, 
\begin{subequations}
    \label{eq:GP_posterior}
    \begin{align}
        &\! \xi(\bm{x}_*)  \!=\! \bm{K}_{\bm{x}_*, \bm{X}} \left(\boldsymbol{K}_{\bm{X},\bm{X}}+ \sigma_{e}^{2} \boldsymbol{I}_{n} \right)^{\!-\!1} { \bm{y}}, 
        \label{eq:post_mean}\\
        &\! \Xi(\bm{x}_*) \!=\! k(\bm{x}_*, \bm{x}_*)  \!-\! \bm{K}_{\bm{x}_*, \bm{X}} \left( \boldsymbol{K}_{\bm{X},\bm{X}}\!+\! \sigma_{e}^{2} \boldsymbol{I}_{n} \right)^{\!-\!1} \bm{K}_{\bm{x}_*, \bm{X}}^\top,
        \label{eq:post_cov}
    \end{align}
\end{subequations}
where $\boldsymbol{K}_{\bm{X},\bm{X}}$ denotes the covariance matrix evaluated on the training input $\bm{X}$, and each entry is $[\boldsymbol{K}_{\bm{X},\bm{X}}]_{i,j} = k({\bm{x}}_i, \bm{x}_j)$; $\bm{K}_{\bm{x}_*, \bm{X}}$ denotes the cross covariance matrix evaluated on the test input $\bm{x}_*$ and the training input $\bm{X}$;  
the zero-mean GP prior is assumed here and will be used in the rest of this paper if there is no further specification.  
Note that the posterior distribution $p(f(\bm{x}_*) \vert \bm{x}_*, \cD)$ gives not only a point estimate, i.e., the posterior mean, but also an uncertainty region of such estimate quantified by the posterior variance.
It is also noteworthy that here we denote the variables in the GPR using mathematical mode italics, such as $\bm{x}_i$ and $y_i$; these variables should not be confused with the latent state $\x_t$ and observation $\y_t$ in SSM (cf. Eq.~(\ref{eq:SSM})).
%, which are represented by lower case Roman letters.

% \vspace{-.05in}
\subsection{Gaussian Process State-Space Models (GPSSMs)}\label{subsec:GPSSM}
A generic SSM describes the probabilistic dependence between latent state $\x_t \in \mathbb{R}^{d_x}$ and observation $\y_t \in \mathbb{R}^{d_y}$. Mathematically, it can be expressed by the following equations:
\begin{subequations}
\label{eq:SSM}
    \begin{align}
        \text{(\textit{Transition})}  \qquad  \x_{t+1} &= f(\x_t) + \mathbf{v}_t,  & \mathbf{v}_t \sim \cN(\bm{0}, \mathbf{Q}), \\
        \text{(\textit{Emission})} \qquad  \quad \y_{t} &= \bm{C} \x_t + \mathbf{e}_t,	& \mathbf{e}_t \sim \cN(\bm{0}, \mathbf{R}),
    \end{align}
\end{subequations}
where the latent states form a Markov chain. That is, for any time instance $t \! \in \! \mathbb{N}$, the next state $\x_{t+1}$ is generated by conditioning on only $\x_t$ and the transition function $f(\cdot)\!:\! \mathbb{R}^{d_x} \! \mapsto \! \mathbb{R}^{d_x}$.  
The emission is assumed to be linear with a known coefficient matrix, $\bm{C} \in \mathbb{R}^{d_y \times d_x}$, hence mitigating the system non-identifiability\footnote{
Nonlinear emissions can be addressed by augmenting the latent state to a higher dimension. This augmentation helps mitigate/eliminate the significant non-identifiability issues commonly encountered in GPSSMs \cite{lin2023towards,frigola2015bayesian}}. Both the states and observations are corrupted by zero-mean Gaussian noise with covariance matrices $\mathbf{Q}$ and $\mathbf{R}$, respectively. 

The GPSSM incorporates a GP prior to model the \cblue{time-invariant} transition function $f(\cdot)$ in Eq.~\eqref{eq:SSM}. 
% After observing data, the GP posterior not only captures the system dynamics but also provides a quantification of the associated uncertainty. 
Specifically, Fig.~\ref{fig:graphical_model} presents the graphical model of GPSSM, while the following equations express its mathematical representation:
\begin{subequations}
\label{eq:gpssm}
	\begin{gather}
		 f(\cdot)  \sim \mathcal{G} \mathcal{P}\left(\mu(\cdot), k(\cdot, \cdot); \btheta_{gp} \right), \\
		 \mathbf{x}_{0} \sim p\left(\mathbf{x}_{0} \right), \\
		 {\f}_{t} =f\left(\mathbf{x}_{t-1}\right), \\
		 \mathbf{x}_{t} \mid {\f}_{t}  \sim \mathcal{N}\left( \x_t \mid {\f}_{t}, \mathbf{Q}\right), \\
		 \mathbf{y}_{t} \mid  \mathbf{x}_{t} \sim \cN \left(\mathbf{y}_{t} \mid \bm{C} \mathbf{x}_{t}, \mathbf{R}\right).
	\end{gather}
\end{subequations}
where $\mathbf{f}_t$ represents the GP transition function value evaluated at the previous state $\mathbf{x}_{t-1}$. For multidimensional state spaces ($d_x \! > \! 1$), the transition function $f(\cdot): \mathbb{R}^{d_x} \!\mapsto\! \mathbb{R}^{d_x}$ is represented using a multi-output GP. In this context, the $d_x$ functions are typically modeled with $d_x$  mutually independent GPs \cite{lin2022output}.
\cblue{The prior distribution of the initial state, $p(\x_0)$, is assumed to be Gaussian and known for simplicity; however, it can also be learned from observed data in the absence of prior information \cite{doerr2018probabilistic}.} 
Additionally, it is noteworthy that the GPSSM illustrated in Fig.~\ref{fig:graphical_model} can be extended to accommodate control systems incorporating a deterministic control input $\bm{c}_t \in \mathbb{R}^{d_c}$ by augmenting the latent state with $[\x_t, \bm{c}_t] \in \mathbb{R}^{d_x + d_c}$. For the sake of brevity, however, we omit explicit reference to $\bm{c}_t$ in our notation throughout this paper.
\cblue{
\begin{remark}
    It is important to note the distinction between GPSSMs and the state-space representation of GPs (SSGP). The SSGP is a method that converts a GP into a linear state-space form, enabling the use of efficient inference techniques such as the Kalman filter and smoother, thereby facilitating the computationally efficient handling of large-scale problems. For more details on SSGP, we direct readers to \cite{martino2021joint, hartikainen2010kalman}.
    In contrast, a GPSSM utilizes the non-parametric flexibility of GPs to model the state transitions in SSMs, allowing for the capture of complex, nonlinear system dynamics. This results in a more flexible but typically computationally intensive SSM.
\end{remark}
}
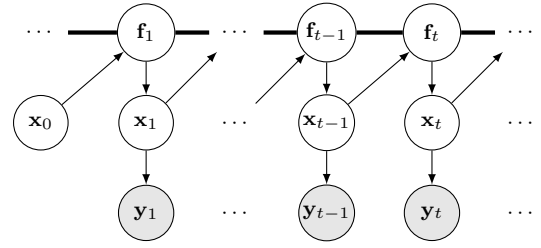
\begin{figure}[t!]
	\centering
	\footnotesize
	\begin{tikzpicture}[align = center, latent/.style={circle, draw, text width = 0.45cm}, observed/.style={circle, draw, fill=gray!20, text width = 0.45cm}, transparent/.style={circle, text width = 0.45cm}, node distance=1.2cm]
		\node[latent](x0) {${\x}_0$};
		\node[latent, right of=x0, node distance=1.4cm](x1) {${\x}_{1}$};
		\node[transparent, right of=x1](x2) {$\cdots$};
		\node[latent, right of=x2](xt-1) {$\!\!{\x}_{t-1}\!\!$};
		\node[latent, right of=xt-1, node distance=1.4cm](xt) {${\x}_{t}$};
		\node[transparent, right of=xt](xinf) {$\cdots$};
		\node[transparent, above of=x0](f0) {$\cdots$};
		\node[latent, above of=x1](f1) {${\f}_{1}$};
		\node[transparent, right of=f1](f2) {$\cdots$};
		\node[latent, above of=xt-1](ft-1) {$\!\!{\f}_{t-1}\!\!$};
		\node[latent, above of=xt](ft) {${\f}_{t}$};
		\node[transparent, right of=ft](finf) {$\cdots$};
		\node[observed, below of=x1](y1) {${\y}_{1}$};
		\node[transparent, below of=x2](y2) {$\cdots$};
		\node[observed, below of=xt-1](yt-1) {$\!\!{\y}_{t-1}\!\!$};
		\node[observed, below of=xt](yt) {${\y}_{t}$};
		\node[transparent, right of=yt](yinf) {$\cdots$};
		\draw[-latex] (x0) -- (f1);
		\draw[-latex] (f1) -- (x1);
		\draw[-latex] (x1) -- (f2);
		\draw[-latex] (ft-1) -- (xt-1);
		\draw[-latex] (xt-1) -- (ft);
		\draw[-latex] (x2) -- (ft-1);
		\draw[-latex] (ft) -- (xt);
		\draw[-latex] (xt) -- (finf);
		\draw[-latex] (x1) -- (y1);
		\draw[-latex] (xt-1) -- (yt-1);
		\draw[-latex] (xt) -- (yt);
		\draw[ultra thick]
		(f0) -- (f1)
		(f1) -- (f2)
		(f2) -- (ft-1)
		(ft-1) -- (ft)
		(ft) -- (finf)
		;
	\end{tikzpicture}
\caption{Graphical model of GPSSM. The white circles represent the latent variables, while the gray circles represent the observable variables. The thick horizontal bar represents a set of fully connected nodes, i.e., the GP.}
\label{fig:graphical_model}
\vspace{-.1in}
\end{figure}

Given the aforementioned model, see Eq.~\eqref{eq:gpssm}, the joint density function of the GPSSM can be expressed as:
\begin{equation}
	\begin{aligned}
		p(\vf, {\vx}, \vy \vert \bm{\theta}) =  p(\mathbf{x}_{0}) p(\f_{1:T}) \prod_{t=1}^{T} p(\mathbf{y}_{t} \vert \mathbf{x}_{t})  p(\mathbf{x}_{t} \vert {\f}_{t}), 
	\end{aligned}
	\label{eq:joint_density}
\end{equation}
where $p(\f_{1:T}) \!\!=\!\! p(f(\x_{0:T-1})) \!\!=\!\! \prod_{t=1}^{T} p(\f_t \vert \f_{1:t-1}, \x_{0:t-1})$ corresponds to a finite dimensional ($T$-dimensional) GP distribution \cite{frigola2015bayesian}, and we define the short-hand notations $\vy \triangleq \y_{1:T} =\{\y_t\}_{t=1}^T$,  $\vf \triangleq \f_{1:T}= \{\f_t\}_{t=1}^T$, and  $\vx \triangleq \x_{0:T}= \{\x_t\}_{t=0}^T$.
The model parameters $\bm{\theta}$ includes the noise covariance and GP hyper-parameters, i.e., $\bm{\theta} \!\!=\!\! \{\mathbf{Q}, \mathbf{R}, \bm{\theta}_{gp}\}$. 
The challenging task in GPSSM is to learn $\btheta$, and simultaneously infer the latent states of interest, which involves the marginal distribution $p(\vy|\btheta)$. However, due to the nonlinearity of GP, a closed-form solution for $p(\vy|\btheta)$ is unavailable. Hence, it becomes necessary to utilize approximation methods.

\section{Problem Statement}
\label{sec:vGPSSMs}
% This section examines and summarizes the issue in the existing variational GPSSM methods \cite{courts2021gaussian,cheng2022rethinking}.
 
To overcome the intractability of the marginal distribution $p(\vy \vert \bm{\theta})$, variational GPSSMs involve constructing a model evidence lower bound (ELBO) to the logarithm of the marginal likelihood. Specifically, the ELBO, denoted by $\mathcal{L}$, is constructed such that the difference between $\log p(\vy \vert \btheta)$ and $\mathcal{L}$ is equal to the Kullback-Leibler (KL) divergence between the variational approximation, $q(\vx, \vf)$, and the true posterior, $p(\vx, \vf \vert \vy)$, i.e., 
\begin{equation}
    \log p(\vy \vert \btheta) - \underbrace{\mathbb{E}_{q(\vx, \vf)} \! \left[\log \frac{p(\vf, \vx, \vy)} {q(\vx, \vf) }\right]}_{ \triangleq \mathcal{L}}  = \underbrace{\operatorname{KL} \!\! \left[ q(\vx, \vf) \| p(\vx, \vf \vert \vy)\right]}_{\ge 0}.
    \label{eq:ELBO_general}
\end{equation}
Detailed derivations can be found in Supplement~\ref{appx_subsec:ELBO}. Maximizing $\mathcal{L}$ with respect to (w.r.t.) $\bm{\theta}$ fine-tunes the model parameters to fit the observed data in the model learning process; while maximizing $\mathcal{L}$ w.r.t. the variational distribution $q(\vx, \vf)$ is equivalent to minimizing the KL divergence, $\operatorname{KL} [ q(\vx, \vf) \| p(\vx, \vf \vert \vy)]$. That is, it enhances the quality of the variational distribution approximation, bringing it closer to the underlying posterior distribution, $p(\vx, \vf \vert \vy)$, which corresponds to the model inference process \cite{cheng2022rethinking}.  The capacity of $q(\vx, \vf)$ to approximate $p(\vx, \vf \vert \vy)$ is thus of crucial importance \cite{courts2021gaussian}. 

Based on the model defined in Eq.~\eqref{eq:joint_density}, a generic factorization of $q(\vx, \vf)$ is as follows:
\begin{equation}
    q(\vx, \vf) = q(\x_0) q(\vf) \prod_{t=1}^T q(\x_t \vert \f_t), 
    \label{eq:generic_vi_dist}
\end{equation}
where $q(\vf)$ represents the variational distribution of the GP and $q(\x_0) \prod_{t=1}^T q(\x_t \vert \f_t)$ corresponds to the variational distribution of the latent states \cite{krishnan2017structured,lin2023towards}. It is noteworthy that the factorization of the variational distribution presented in Eq.~(\ref{eq:generic_vi_dist}) is recognized as an NMF approximation within the GPSSM literature (see formal definition in Appendix \ref{appendix:MF_NMF_definition}) \cite{ialongo2019overcoming}, because it explicitly establishes the dependencies between the latent states and the GP transition function values, as manifested by the terms $\prod_{t=1}^T q(\x_t \vert \f_t)$. 

The works most closely related to this paper are probably \cite{doerr2018probabilistic} and \cite{ialongo2019overcoming}. In \cite{doerr2018probabilistic}, variational distribution $q(\x_t \vert \f_t)$ is subjectively set equal to the prior distribution $p(\x_t \vert \f_t)$, i.e., $q(\x_t \vert \f_t) = p(\x_t \vert \f_t)$. But using the prior distribution as an approximation results in no filtering or smoothing effect on latent states, although there is a dependence between the latent states $\vx$ and the transition function values $\vf$. Instead, the work in \cite{ialongo2019overcoming} assumes a parametric Markov-structured Gaussian variational distribution over the temporal states, i.e.,
\begin{equation}
q(\x_t \vert \f_t) = \cN(\x_t \vert \bm{A}_t \f_t +\bm{b}_t, \bm{S}_t),
\label{eq:varidist_ialongo}
\end{equation}
where $\{\bm{A}_t, \bm{b}_t, \bm{S}_t\}_{t=1}^T$ are free variational parameters.  This choice, however, introduces a significant drawback: the number of variational parameters grows linearly with the length of the time series. Although this issue can be partially addressed by incorporating an \textit{inference network}, such as a bidirectional recurrent neural network \cite{eleftheriadis2017identification,krishnan2017structured}, to learn the variational parameters and make the variational distribution learning with a constant model complexity, fine-tuning these variational parameters of the inference network still requires substantial effort. Consequently, despite the flexibility of the variational distribution described in Eq.~\eqref{eq:varidist_ialongo} allowing for approximation of the true posterior distribution, its empirical performance often falls short of its theoretical expressive capacity \cite{lindinger2022laplace}. Moreover, it is noteworthy that the inference network commonly takes the input of the observations, $\vy$, rendering the existing inference network-based methods predominantly trained in an offline manner, thus posing challenges when attempting to adapt them for online applications.  

To address these limitations, we step away from the heavy parameterization (e.g., the black-box inference networks) and propose a novel interpretable EnKF-aided algorithm for variational inference in GPSSMs. Our algorithm can exploit the dependencies between $\vx$ and $\vf$ and alleviate the model and computational complexities while being easily extended to an online learning algorithm.

\section{EnVI: EnKF-Aided Variational Inference} \label{sec:EnVI}
This section presents our novel variational inference algorithm for GPSSMs. We begin by introducing the model-based EnKF in Section \ref{subsec:EnKF}, which serves as the foundation for our algorithm. Following that,  Section \ref{subsec:EnVI} details the proposed EnKF-aided variational inference algorithm. Lastly, extensive discussions about the properties of our proposed algorithm are given in Section \ref{subsec:EnVI_discussion}.

\subsection{Ensemble Kalman Filter (EnKF)} \label{subsec:EnKF}
The EnKF is a Monte Carlo-based method that excels in handling nonlinear systems compared to the KF and EKF \cite{roth2017ensemble}. Given an SSM, see Eq.~\eqref{eq:SSM}, the EnKF sequentially approximates the filtering distributions using $N \! \in \! \mathbb{N}$ equally weighted particles \cite{roth2017ensemble}. Specifically, at the prediction step, EnKF first samples particles, $\x_{t-1}^{1:N} \triangleq \{\x_{t-1}^n\}_{n=1}^N$, from the filtering distribution, $p(\x_{t-1} \vert \y_{1:t-1})$, at time $t-1$. Then, the sampled particles are propagated using the state transition function $f(\cdot)$, i.e.,
\begin{equation}
\label{eq:EnKF_predictive}
\begin{aligned}
\bar{\x}_t^n = f(\x_{t-1}^n) + \mathbf{v}_t^n, \ \mathbf{v}_t^n \sim \cN(0, \mathbf{Q}), \ \forall n.
\end{aligned}
\end{equation} 
The prediction distribution is then approximated by a Gaussian distribution, i.e. $p(\x_t \vert \y_{1:t-1})  \approx \cN( \x_t \vert \bar{\mathbf{m}}_t, \bar{\mathbf{P}}_t)$,
where 
\begin{subequations}
\label{eq:prediction_monmenets}
\begin{align}
        \bar{\mathbf{m}}_t  & =  \frac{1}{N} \sum_{n=1}^N  \bar{\x}_t^n , \\ 
        \bar{\mathbf{P}}_t &= \frac{1}{N\!-\!1} \sum_{n=1}^N (\bar{\x}_t^n - \bar{\mathbf{m}}_t )(\bar{\x}_t^n - \bar{\mathbf{m}}_t )^\top.
\end{align}
\end{subequations}
With the linear Gaussian emission in Eq.~\eqref{eq:SSM}, the joint distribution of $\x_t$ and $\y_t$ can be readily obtained as follows:
%\footnote{It is worth noting that EnKF is capable of handling SSMs with nonlinear and non-Gaussian emissions. However, for the purpose of demonstration, we focus on showcasing the case of linear Gaussian emission.}:
%
\begin{equation}
\label{eq:joint_y_x_enkf}
    \begin{bmatrix}
        \x_t \vert \y_{1:t-1}\\
        \y_t \vert \y_{1:t-1}
    \end{bmatrix} \sim \cN \left( 
    \begin{bmatrix}
        \bar{\mathbf{m}}_t\\
        \bm{C}\bar{\mathbf{m}}_t
    \end{bmatrix}, 
    \begin{bmatrix}
        \bar{\mathbf{P}}_t, & \bar{\mathbf{P}}_t \bm{C}^{\top} \\
        \bm{C} \bar{\mathbf{P}}_t, & \bm{C} \bar{\mathbf{P}}_t \bm{C}^{\top} + \mathbf{R}
    \end{bmatrix} 
    \right).
\end{equation}
Thus, the filtering distribution, $p(\x_t \vert \y_{1:t}) \!=\! \cN( \x_t \vert \mathbf{m}_t, \mathbf{P}_t),$ can be obtained using the conditional Gaussian identity at the update step,  
where
\begin{subequations}
\label{eq:EnKF_update_distr}
    \begin{align}
            \mathbf{m}_t & = \bar{\mathbf{m}}_t + \bar{\mathbf{G}}_t (\y_t - \bm{C} \bar{\mathbf{m}}_t), \\
            \mathbf{P}_t & = \bar{\mathbf{P}}_t - \bar{\mathbf{P}}_t \bm{C}^{\top} \bar {\mathbf{G}}_t^\top,
    \end{align}
\end{subequations}
and $\bar{\mathbf{G}}_t = \bar{\mathbf{P}}_t \bm{C}^{\top} (\bm{C} \bar{\mathbf{P}}_t \bm{C}^{\top} + \mathbf{R})^{-1}$ is the Kalman gain \cite{roth2017ensemble}. Each filtered particle ${\x}_t^n$ is then obtained from a Kalman-type update, i.e.,
\begin{equation}
 {\x}_t^n = \bar{\x}_t^n  + \bar{\mathbf{G}}_t (\y_t + \mathbf{e}_t^n - \bm{C} \bar{\x}_t^n), \ \mathbf{e}_t^n \sim \cN(0, \mathbf{R}), \ \forall n.
 \label{eq:updated_samples}
\end{equation}

\cblue{
It is crucial to note that the EnKF inherently exhibits differentiability. Specifically, if we utilize the reparameterization trick \cite{kingma2019introduction} to sample the process and observation noises as shown in Eq.~\eqref{eq:EnKF_predictive} and \eqref{eq:updated_samples}, i.e.,
\begin{subequations}
\label{eq:reparameterization_trick}
    \begin{align}
        \mathbf{v}_t^n = \bm{0} + \mathbf{Q}^{\frac{1}{2}} \bm{\epsilon}, \ \ \  \bm{\epsilon} \sim \cN(\bm{0}, \bm{I}_{d_x}), \\
        \mathbf{e}_t^n = \bm{0} + \mathbf{R}^{\frac{1}{2}} \bm{\epsilon}, \ \ \  \bm{\epsilon} \sim \cN(\bm{0}, \bm{I}_{d_y}), 
    \end{align}
\end{subequations}
then the sampled latent states become differentiable w.r.t. the state transition function $f$ and the noise covariances $\mathbf{Q}$ and $\mathbf{R}$. This differentiable nature is crucial for enabling the use of gradient-based optimization methods in learning algorithms. Moreover, although we present the linear emission model in this paper, the EnKF can be extended to nonlinear emission models. The difference lies in the computation of the predicted observation covariance and the cross-covariance between the state and the observation predictions in Eq.~\eqref{eq:joint_y_x_enkf}, which are done using ensembles. For more details, see e.g. \cite{roth2017ensemble}.
}

\begin{remark} \label{remark:EnKF_vs_PF}
    Compared to another Monte Carlo-based method, PF, EnKF demonstrates computational efficiency, particularly in exceedingly high-dimensional spaces \cite{katzfuss2016understanding}, though the PF may offer more flexibility in addressing highly non-Gaussian and nonlinear aspects. % Compared to PF, the PF might handle strongly nonlinear and non-Gaussian systems better, but EnKF remains computationally efficient even in extremely high-dimensional scenarios \cite{katzfuss2016understanding}, unlike PF which demands substantial computational resources for high dimensions \cite{naesseth2019high}. 
    Moreover, the inherent differentiable nature of EnKF (see Eqs.~\eqref{eq:EnKF_predictive}--\eqref{eq:updated_samples}) enables seamless integration with off-the-shelf automatic differentiation tools (e.g., PyTorch \cite{paszke2019pytorch}) for model parameter optimization (see more discussions in Section \ref{subsec:EnVI_discussion}). This stands in contrast to PF, where the discrete distribution resampling significantly poses challenges for the utilization of the reparameterization trick \cite{kingma2019introduction}, resulting in significantly higher computational complexity for the parameters gradient computations \cite{naesseth2019high,rosato2022efficient,chen2022autodifferentiable}.
\end{remark}

In the following subsection, we will describe our novel EnKF-aided NMF variational inference algorithm for GPSSMs, which significantly reduces the number of variational parameters and ultimately enhances the learning and inference performance.

\subsection{EnKF-Aided Variational Inference (EnVI)} 
\label{subsec:EnVI}
% In Section \ref{subsubsec:sparse_GPSSM}, we introduce the sparse GP \cite{titsias2009variational}, a widely utilized technique in various GP variational approximations, followed by constructing the variational lower bound of GPSSM in Section \ref{subsubsec:ELBO_Sparse_GPSSM}. Subsequently, we elucidate how our EnKF seamlessly integrates into the variational inference framework in Section \ref{subsubsec:EnVI_algorithm} to compute the variational lower bound efficiently.
\subsubsection{\textbf{Sparse GPSSMs}} \label{subsubsec:sparse_GPSSM}
Before introducing the utilization of EnKF in the variational inference framework, let us first introduce the sparse GP \cite{titsias2009variational}, a widely utilized technique in various GP variational approximations. This method serves as a scalable approach to model the corresponding GP component within the approximate posterior distribution (see Eq.~\eqref{eq:generic_vi_dist}), thereby guaranteeing inherent scalability in GPSSM. The main idea of the sparse GP is to introduce a small set of inducing points $\vec{\z} \triangleq \{\z_{i}\}_{i=1}^M$ and $\vec{\u} \triangleq \{\u_{i}\}_{i=1}^{M}$, $M \! \ll \! T$, to serve as the surrogate of the associated GP, where the inducing inputs, $\z_i \!\in\! \mathbb{R}^{d_x}, \forall i$, are placed in the same space as the latent states $\x_t$, while the corresponding inducing outputs $\u_{i} \!=\! f(\z_{i})$ follow the same GP prior as $\vf$. With the augmentation of the inducing points, the joint distribution of the GPSSM becomes 
\begin{equation}
 p(\vf, \vu, \vx, \vy) = p(\x_{0}) p(\vf, \vu) \prod_{t =1}^{T}  p(\y_{t} \vert \x_{t}) p(\x_{t} \vert \f_{ t}),
\label{eq:joint_dist}
\end{equation}
where $p(\vf, \vu) = p(\vu) p(\vf \vert \vu) =  p(\vu) \prod_{t=1}^T p(\f_{t} \vert \x_{t-1}, \vu)$ is the augmented GP prior and 
$p(\f_{t} \vert \x_{t-1}, \vu)$ is the noiseless GP prediction whose mean and covariance can be computed similarly to Eq.~(\ref{eq:GP_posterior}). 
\cblue{The introduced inducing inputs $\vz$ will be treated as variational parameters and jointly optimized with model parameters \cite{hensman2013gaussian}. We will further describe this later. }

Suppose that the inducing outputs $\vu$ serve as sufficient statistics for the GP function values $\vf$, such that given $\vu$, the GP function values $\vf$ and any novel $\f_*$ are independent \cite{titsias2009variational}, i.e., $p(\f_* \vert \vf, \vu) = p(\f_* \vert \vu)$ for any $\f_*$. We can integrate out $\vf$ in Eq.~\eqref{eq:joint_dist}, and the transition function is fully characterized using only the inducing points. Consequently, we have:
\begin{equation}
\setlength{\abovedisplayskip}{4.5pt}
\setlength{\belowdisplayskip}{4.5pt}
    p(\vu, \vx, \vy) = p(\x_0) p(\vu) \prod_{t=1}^T p(\y_t \vert \x_t) p(\x_t \vert \vu, \x_{t-1}),
    \label{eq:joint_dist_ips}
\end{equation}
where 
\begin{equation}
\label{eq:GP_transition_enkf} 
\begin{aligned}
    p(\x_t \vert \vu, \x_{t-1}) & = \int p(\x_t \vert \f_t) p(\f_t \vert \x_{t-1}, \vu, \vz) \mathrm{d}\f_t  \\
    & = \cN(\x_t ~\vert~ \bm{\xi}_t, \ \bm{\Xi}_t), 
\end{aligned}
\end{equation}
and with a bit abuse of notation,
\begin{subequations}
\begin{align}
    & \bm{\xi}_t = \bm{K}_{\x_{t-1}, \vz}(\bm{K}_{\vz,\vz} + \mathbf{Q})^{\!-\!1} \vu, \\
    & \bm{\Xi}_t = \bm{K}_{\x_{t\!-\!1}, \x_{t\!-\!1}}  \!+\! \mathbf{Q} \!-\! \bm{K}_{\x_{t\!-\!1}, \vz}(\bm{K}_{\vz, \vz} + \mathbf{Q})^{\!-\!1}\bm{K}_{\x_{t\!-\!1}, \vz}^\top.
\end{align}
\end{subequations}
That is to say, with the aid of sparse GPs, the computational complexity of GPSSMs can be reduced to $\mathcal{O}(d_x T M^2)$, comparing to the original $\mathcal{O}(d_x T^3)$ \cite{lin2023towards}.

\subsubsection{\textbf{ELBO for sparse GPSSM}} \label{subsubsec:ELBO_Sparse_GPSSM}
In the context of the GPSSM described in Eq.~\eqref{eq:joint_dist_ips}, we first assume a generic variational distribution for the latent variables, $\{\vu, \vx\}$,  factorized as follows:
\begin{equation}
	\begin{aligned}
            q(\vu, \vx) & =  q(\vu) q(\x_0) \prod_{t= 1}^T  \int q(\x_t \vert \f_t) p(\f_t \vert \vu, \x_{t-1}) \mathrm{d} \f_t\\
            & = q(\vu) q(\x_0) \prod_{t= 1}^T  q(\x_{t} \vert \vu, \x_{t-1}). 
        \end{aligned}
	\label{eq:generic_vi_dist_enKF}
\end{equation}
Here, the variational distribution over the inducing outputs, $q(\vu)$, is explicitly assumed to be a free-form Gaussian, i.e.,
\begin{equation}
        q(\vu) = \prod_{d=1}^{d_x} \cN(\{\u_{i, d}\}_{i=1}^M \vert \m_d, \mathbf{L}_{d} \mathbf{L}_{d}^\top ) = \cN(\vu ~\vert~ \mathbf{m}, \mathbf{S}),
	\label{eq:qu_variational}
\end{equation}
where $\mathbf{m} \!\!=\!\! [\m_1^\top, \ldots, \m_{d_x}^\top]^\top \in \mathbb{R}^{M d_x}$ and $\mathbf{S} \!\!=\!\! \operatorname{diag}(\mathbf{L}_1 \mathbf{L}_{1}^\top, \ldots, \mathbf{L}_{d_x} \mathbf{L}_{d_x}^\top) \in \mathbb{R}^{Md_x \times Md_x}$ are free variational parameters.  This explicit representation of the variational distribution enables scalability through the utilization of stochastic gradient-based optimization \cite{hoffman2013stochastic}, as it allows for the independence of individual GP predictions given the explicit inducing points \cite{hensman2013gaussian}. 
\cblue{The corresponding inducing inputs, $ \vz $, are treated as variational parameters as well, as described in \cite{titsias2009variational,hensman2013gaussian}. These parameters, $\{\vz, \m, \mathbf{S}\}$, collectively define the variational distribution that approximates the true GP posterior.}
We also assume that the variational distribution over the initial state is $q(\x_{0}) = \mathcal{N}(\x_0 \vert \m_{0}, \mathbf{L}_{0}\mathbf{L}_{0}^\top)$, where ${\m}_{0} \in \mathbb{R}^{d_x}$ and lower-triangular matrix $\mathbf{L}_{0} \in \mathbb{R}^{d_x \times d_x}$ are free variational parameters of $q(\x_0)$.   
The variational distribution of the latent states, $q(\x_t \vert \vu, \x_{t-1})$, obtained by integrating out $\f_t$ as shown in Eq.~\eqref{eq:generic_vi_dist_enKF}, will be implicitly modeled by resorting to the EnKF technique. This modeling approach will help eliminate the need to parameterize $q(\x_t \vert \vu, \x_{t-1})$, a common requirement in the existing works \cite{eleftheriadis2017identification,doerr2018probabilistic,ialongo2019overcoming, curi2020structured}, thus overcoming the challenges associated with optimizing a large number of variational parameters, as discussed in Section \ref{sec:vGPSSMs}.
%
% which helps overcome the challenges associated with optimizing a large number of variational parameters, as discussed in Section \ref{sec:vGPSSMs}.
%
% For the variational distribution of the latent states, $q(\x_t \vert \vu, \x_{t-1})$, which obtain from the integrating out $\f_t$, see Eq.~\eqref{eq:generic_vi_dist_enKF}, we will resort to the conventional EnKF technique to obtain the corresponding approximations, which help effectively remove the requirement for making parameterized assumptions about the variational distribution, as commonly observed in the existing works \cite{eleftheriadis2017identification,doerr2018probabilistic,ialongo2019overcoming}. Consequently, we are able to circumvent the challenges associated with optimizing a substantial number of parameters in black-box inference networks, as mentioned in Section \ref{sec:vGPSSMs}.

Before elucidating the methodology of employing EnKF to eliminate the parameterization for the variational distribution $q(\x_t \vert \vu, \x_{t-1})$, we first undertake the derivation of the ELBO, 
% induced by Eq.~\eqref{eq:joint_dist_ips} and Eq.~\eqref{eq:generic_vi_dist_enKF}, 
which is succinctly summarized as follows.
% After performing certain algebraic calculations, the ELBO is obtained as
\begin{theorem} \label{them:general_elbo}
Upon the augmentation of the inducing points into the GPSSM (see Eq.~\eqref{eq:joint_dist_ips}) and under the NMF assumption of the variational distribution (see Eq.~\eqref{eq:generic_vi_dist_enKF}), the model evidence lower bound for joint learning and inference is:
    \begin{equation}
    \label{eq:ELBO_NMF_enkf}
	\begin{aligned}
	\!\! \mathcal{L} 
  %           &= \mathbb{E}_{q(\vu, \vx)} \left[\log \frac{p( \vu, \vx, \vy)} {q(\vu, \vx)}\right]\\
		% & = \mathbb{E}_{q(\vu, \vx)}  \! \left[\log \frac{p(\mathbf{x}_{0})  p(\vu) \prod_{t=1}^{T}  p(\mathbf{x}_{t} \vert {\vu}, \x_{t-1}) p(\mathbf{y}_{t} \vert \mathbf{x}_{t})} {q(\vu) q(\x_0) \prod_{t= 1}^T q(\x_{t} \vert \vu, \x_{t-1}) }\right]  \\
		& = \underbrace{ \mathbb{E}_{q(\vu, \vx)}  \left[ \sum_{t=1}^T \log p(\y_{t} \vert \x_t) \right]}_{\mathrm{term~1} }   \
		-  \ \underbrace{\operatorname{KL}[q(\x_0) \| p(\x_0)]}_{\mathrm{term~ 2}} \\ 
        & \quad \!-\!\underbrace{\operatorname{KL}[q(\vu) \| p(\vu)]}_{\mathrm{term~3}}  - \underbrace{\mathbb{E}_{q(\vu, \vx)}  \left[ \sum_{t=1}^T \log \frac{q(\x_{t} \vert \vu, \x_{t-1})}{p(\x_t \vert \vu, \x_{t-1})}\right]}_{\mathrm{term~4}}.
	\end{aligned}
    \end{equation}
\end{theorem}
\begin{proof}
    The proof can be found in Appendix \ref{appendix_elbo_proof}
\end{proof}
\noindent An important observation lies in the interpretability of each component within the ELBO. Specifically, when maximizing the ELBO for learning and inference:
 \begin{itemize}
  	\item  Term 1 corresponds to the data reconstruction error, which encourages any state trajectory $\vx$ sampled from the variational distribution, $q(\vu, \vx)$, to accurately reconstruct the observed data.
 	\item Terms 2 and 3 serve as regularization terms for $q(\x_0)$ and $q(\vu)$, respectively. They discourage significant deviations of the variational distributions from the corresponding prior distributions.
 	\item  Term 4 represents a regularization for $q(\x_t \vert \vu, \x_{t-1})$, which discourages significant deviations of $q(\x_t \vert \vu, \x_{t-1})$ from the prior $p(\x_t \vert \vu, \x_{t-1})$.
 \end{itemize}

\subsubsection{\textbf{EnVI algorithm}} \label{subsubsec:EnVI_algorithm}
The NMF assumption applied to the variational distribution (Eq.~\eqref{eq:generic_vi_dist_enKF}) typically results in an intractable evaluation of the first and fourth terms within the ELBO \cite{ialongo2019overcoming}. Therefore, approximations are needed to help efficiently evaluate the ELBO. One existing method is to simply set $q(\x_t \vert \vu, \x_{t-1}) \!=\! p(\x_t \vert \vu, \x_{t-1})$, resulting in the same ELBO as in \cite{doerr2018probabilistic}, where no filtering or smoothing effect is feasible. To circumvent this limitation, we examine the difference between term 1 and term 4 in Eq.~\eqref{eq:ELBO_NMF_enkf} and propose our variational lower bound approximation. The main result is summarized in the following proposition.
\cblue{
\begin{proposition}
\label{prop:approx_ELBO_}
Under the approximations that:
\begin{itemize}
    \item[1)] $p(\x_{t-1} \vert \vu, \y_{1:t-1}) \approx p(\x_{t-1} \vert \vu, \y_{1:t})$,
    \item[2)] $q(\x_{t} \vert \vu, \x_{t-1}) \approx p(\x_{t} \vert \vu, \x_{t-1}, \y_{1:t})$, 
\end{itemize}
the ELBO presented in Theorem \ref{them:general_elbo} can be reformulated as a summation over several simple terms:
\begin{equation}
\label{eq:EnVI_lower_bound}
\begin{aligned}
    \! \! \mathcal{L} \!\approx\! \mathbb{E}_{q(\vu)} \! \left[ \sum_{t=1}^T \log  p(\y_t \vert \vu, \y_{1:t-1}) \right] & \! - \operatorname{KL}[q(\x_0) \| p(\x_0)] \\ & \! - \operatorname{KL}[q(\vu) \| p(\vu)],
\end{aligned}
\end{equation}
where the log-likelihood, $\log p(\y_t \vert \vu, \y_{1:t-1})$ in the first term can be analytically evaluated using the EnKF (discussed below). The two KL divergence terms can also be computed in closed form, due to the Gaussian nature of the prior and variational distributions \cite{theodoridis2020machine}.
\end{proposition}
}
%%--------------------------------------
%
\begin{proof}
   The proof can be found in Appendix~\ref{appendix_elbo_proof}.
\end{proof}

\cblue{
Proposition~\ref{prop:approx_ELBO_} demonstrates that the newly derived ELBO is significantly more tractable under two mild approximations, which are justified and explained as follows.
\begin{itemize}
    \item Approximation 1) posits that the state estimation at time $t\!\!-\!\!1$ using observations from time $1$ to $t\!\!-\!\!1$ (i.e., $p(\x_{t-1} \vert \vu, \y_{1:t-1})$) is approximately equal to the estimation using observations from time $1$ to $t$ (i.e., $p(\x_{t-1} \vert \vu, \y_{1:t})$). This approximation is generally reasonable, particularly when $t$ is large (i.e., with a long observation sequence), as the increased information aids in more accurately inferring the latent state, thereby reducing the discrepancy between the two posterior distributions.
    \item Approximation 2) states that the variational distribution $q(\x_t \vert \vu, \x_{t-1})$ is approximately equal to the posterior distribution $p(\x_t \vert \vu, \x_{t-1}, \y_{1:t})$. In the variational inference framework \cite{hoffman2013stochastic}, the variational distribution $q(\x_t \vert \vu, \x_{t-1})$ is intended to approximate the true but unknown distribution $p(\x_t \vert \vu, \x_{t-1}, \y_{1:T})$. That is to say, in Approximation 2), we are essentially using the filtering distribution $p(\x_t \vert \vu, \x_{t-1}, \y_{1:t})$ to approximate the smoothing distribution $p(\x_t \vert \vu, \x_{t-1}, \y_{1:T})$. While this approximation may introduce some estimation loss, it remains reasonable when the observation series is sufficiently long, such that additional future observations do not significantly alter the state estimates. On the other hand, given our priority on computational efficiency, a minor loss in accuracy is acceptable and often necessary as part of this trade-off.
\end{itemize}
}

We next proceed to outline the evaluation of the log-likelihood, $\log p(\y_t \vert \vu, \y_{1:t-1})$ in Eq.~\eqref{eq:EnVI_lower_bound} using the EnKF.  which allows us to evaluate the log-likelihood recursively and analytically. 

Building upon the EnKF outlined in Section~\ref{subsec:EnKF} and assuming that we have acquired the posterior distribution, $p(\x_{t-1} \vert \vu, \y_{1:t-1}) \!=\!\cN(\x_{t-1} \vert \mathbf{m}_{t-1}, \mathbf{P}_{t-1})$, at time $t\!-\!1$, we can employ the GP transition,  presented in Eq.~\eqref{eq:GP_transition_enkf}, to perform the prediction step and generate $N$ predicted samples $\bar{\x}_t^{1:N}$, i.e.,
\begin{equation}
\label{eq:EnKF_predictive_GP}
      \bar{\x}_t^n \sim p(\x_t \vert \vu, \mathbf{x}_{t-1}^n),  \ \  n = 1, 2, \ldots, N,
\end{equation}
where $\x_{t-1}^{1:N}$ are $N$ particles obtained from $p(\x_{t-1} \vert \vu, \y_{1:t-1})$ with equal weights. With the predicted samples $\bar{\x}_t^{1:N}$, we thus can approximate the prediction distribution as a Gaussian:
\begin{align}
    p(\x_t \vert \vu, \y_{1:t-1}) & = \int p(\x_{t} \vert \vu, \x_{t-1}) p(\x_{t-1} \vert \vu,  \y_{1:t-1}) \mathrm{d}\x_{t-1} \nonumber \\
    % & = \int p(\x_{t} \vert \vu, \x_{t-1}) q(\x_{t-1} \vert \vu) \mathrm{d}\x_{t-1}  \nonumber \\
    & \approx \cN(\x_t \vert \bar{\mathbf{m}}_t, \bar{\mathbf{P}}_t) \label{eq:predictive_distri_enkf_ip}
\end{align}
where $\bar{\mathbf{m}}_t$ and $\bar{\mathbf{P}}_t$ can be computed using Eq.~\eqref{eq:prediction_monmenets}. 

Similarly, during the update step, utilizing Eqs.~\eqref{eq:EnKF_update_distr} and \eqref{eq:updated_samples}, we can derive the filtering distribution at time step $t$, 
\begin{equation}
\label{eq:var_q_x_u}
p(\x_t \vert \vu, \y_{1:t}) = \cN(\x_t \vert \mathbf{m}_t, \mathbf{P}_t)  
\end{equation} 
and obtain the set of $N$ updated samples $\x_{t}^{1:N}$. \cblue{We can then recursively obtain the samples and posterior distributions in the subsequent time steps.} Note that in the context of GPSSMs, both the prediction and update steps are conditioned on the inducing points, $\{\vu, \vz\}$, which act as a surrogate for the transition function.

Leveraging the EnKF outlined above,  the log-likelihood, $\log p(\y_t \vert \vu, \y_{1:t-1})$ in Eq.~\eqref{eq:EnVI_lower_bound}, can be evaluated recursively and analytically. Specifically, at each time step $t$, we have
\begin{equation}
    \begin{aligned}
        \log p(\y_t \vert \vu, \y_{1:t-1}) &\!=\! \log \int p(\y_t \vert \x_t) p(\x_t \vert \vu, \y_{1:t-1}) \mathrm{d} \x_{t}\\
        & \!=\! \log \cN(\y_t \vert \bm{C} \bar{\mathbf{m}}_t, \ \bm{C} \bar{\mathbf{P}}_t \bm{C}^\top ),
    \end{aligned}
    \label{eq:evaluation_log_likelihood}
\end{equation}
due to the Gaussian prediction distribution, see Eq.~\eqref{eq:predictive_distri_enkf_ip}, and the linear emission model.

Now we can evaluate our approximate variational lower bound, $\mathcal{L}$ in Eq.~\eqref{eq:EnVI_lower_bound}. We first utilize the reparameterization trick (see also e.g. Eq.~\eqref{eq:reparameterization_trick}) \cite{kingma2019introduction} to sample $\vu$, i.e.,
\begin{equation}
    \cblue{\vu = \mathbf{m} + \mathbf{S}^{\frac{1}{2}} \bm{\epsilon}, \  \bm{\epsilon} \sim \cN(\bm{0}, \bm{I}),}
\end{equation}
and numerically get an unbiased evaluation of the expected log-likelihood, $\mathbb{E}_{q(\vu)} \! \left[ \sum_{t=1}^T \log  p(\y_t \vert \vu, \y_{1:t-1}) \right]$. 
Due to the reparameterization trick \cite{kingma2019introduction}, $\mathcal{L}$ is differentiable w.r.t. the model parameters $\btheta \!=\! \{\bm{\theta}_{gp}, \mathbf{Q}, \mathbf{R} \}$ and the variational parameters $\bzeta = \{ \mathbf{m}_0, \mathbf{L}_0, \mathbf{m}, \mathbf{S}, \vz\}$. \cblue{Therefore, we can use modern differentiation tools, such as PyTorch, to automatically compute the gradient through backpropagation through time (BPTT) and apply gradient-based methods (e.g., Adam) to maximize $\mathcal{L}$ \cite{paszke2019pytorch, kingma2015adam}.}
Detailed routine for implementing the EnKF-aided variational learning and inference algorithm, termed as EnVI, is summarized in Algorithm \ref{alg:EnKF_version2}.

\begin{algorithm}[t!]
\caption{EnKF-aided variational learning and inference}
\label{alg:EnKF_version2}
\begin{algorithmic}[1]
\Statex {\bf Input}:  $\btheta = \{\btheta_{gp}, \mathbf{Q}, \mathbf{R}\}, ~\bzeta, ~\y_{1:T}$, $\x_0^{1:N} \sim  q(x_0)$
\While{\textit{iterations not terminated}}
\State $\vu \sim q(\vu)$, $L_\ell = 0$
\For {$t= 1, 2, \ldots, T$}
\State Get prediction samples using Eq.~\eqref{eq:EnKF_predictive_GP}
\State Get empirical moments $\bar{\mathbf{m}}_t, \bar{\mathbf{P}}_t$ using Eq.~\eqref{eq:predictive_distri_enkf_ip}
\State Get Kalman gain: $\bar{\mathbf{G}}_t \!=\! \bar{\mathbf{P}}_t \bm{C}^{\top} (\bm{C} \bar{\mathbf{P}}_t \bm{C}^{\top} \!\!+\! \mathbf{R})^{-1}$  
\State  Get updated samples using Eq.~\eqref{eq:updated_samples}
\State Evaluate the log-likelihood using Eq.~\eqref{eq:evaluation_log_likelihood}, and 
$$
L_\ell = L_\ell + \log p(\y_t \vert\vu, \y_{1:t-1})
$$  
\EndFor
\State  $\mathcal{L} = L_\ell - \operatorname{KL}(q(\x_0) \| p(\x_0)) - \operatorname{KL}(q(\vu) \| p(\vu)) $
\State Maximize $\mathcal{L}$ and update $\btheta$, $\bzeta$ using Adam \cite{kingma2015adam}
\EndWhile
\Statex {\bf Output}: EnKF particles $\x_{0:T}^{1:N}$,  model parameters $\btheta$, and variational parameters $\bzeta$.
\end{algorithmic}
\end{algorithm}

It is noteworthy that the newly derived ELBO, $\mathcal{L}$ circumvents the explicit evaluation of the first and fourth terms in Eq.~\eqref{eq:ELBO_NMF_enkf}, and sidesteps the computational challenges posed by the heavy parameterization of $q(\x_t \vert \vu, \x_{t-1})$. Consequently, EnVI can substantially improve the efficiency of model learning and inference. 
In addition, maximizing the ELBO in Eq.~\eqref{eq:EnVI_lower_bound} can be interpreted as follows:  The objective is to optimize the model parameters and variational parameters such that the GPSSM can fit the observed data well at each step (indicated by the first term); simultaneously, the KL regularization terms impose constraints to prevent model overfitting (indicated by the second and third terms).

\begin{remark} 
\label{remark:comput_complexity}
The computational complexity of the EnVI algorithm predominantly lies in the evaluation of $N$ independent particles on the GP transition during the prediction step (see  Eq.~\eqref{eq:EnKF_predictive_GP}). Recall that the number of inducing points is significantly smaller than the length of the data sequence, i.e., $M \! \ll \! T$, and assume that $M  \! \ge \! d_x$. In this context, the computational complexity of Algorithm \ref{alg:EnKF_version2} scales as $\mathcal{O}(N T d_x M^2)$. In practice, $N$ is often a small number, and the computation of the $N$ particles on the GP transition can be run in parallel \cite{paszke2019pytorch}, resulting in the computational complexity in real-world deployments scaling as $\mathcal{O}(T d_x M^2)$. This cost matches that of existing works. Yet,  the streamlined model complexity in EnVI enhances its computational robustness and accelerates convergence compared to the existing works.
\end{remark}
\vspace{-.05in}

\subsection{More Discussions and Insights} 
\label{subsec:EnVI_discussion}
This subsection presents more detailed insights into the proposed EnVI and discusses its connections to existing works.

% Indeed, there are various well-established model-based filtering and smoothing techniques that can be potentially integrated into the variational inference framework for GPSSMs. Such as KF, EKF, PF, and EnKF, to name a few \cite{sarkka2013bayesian}. Among these options, EnKF, as a Monte Carlo-based method, exhibits superior capability in handling nonlinear systems when compared to KF and EKF \cite{roth2017ensemble}. While analytical solutions can be obtained for nonlinear systems using e.g. EKF methods, they do not align with the principle advocated by John Tukey that \textit{an approximate solution to the right problem is worth more than a precise solution to the wrong problem} \cite{tukey1962future}. In comparison to PF, EnKF may have certain limitations in dealing with strongly nonlinear and non-Gaussian systems. However, sampling in PF is generally more computationally demanding, particularly for high-dimensional systems \cite{naesseth2019high}, whereas EnKF can be effectively employed even in extremely high-dimensional systems \cite{katzfuss2016understanding}. 

First of all, the importance of the differentiable nature within EnKF, as mentioned in Remark \ref{remark:EnKF_vs_PF}, to the EnVI becomes more apparent. The inherent differentiability in EnVI, spanning from parameters ($\{\btheta, \bzeta\}$) to latent states and extending to the objective function (the ELBO), enables principled joint learning and inference using modern off-the-shelf automatic differentiation tools (e.g., PyTorch \cite{paszke2019pytorch}). This contrasts with most existing EnKF-based dynamical systems learning methods (see, e.g., \cite{chen2023reduced, ishizone2020ensemble} and the references therein), which employ the expectation-maximization (EM) algorithm \cite{theodoridis2020machine} to iteratively update the model parameters $\btheta$ and latent state trajectory $\vx$. 
% More specifically, these methods first fix the model parameter $\btheta$ and employ the EnKF to estimate the latent state trajectory $\vx$; Subsequently, they maintain the fixed trajectory $\vx$ to calculate the model marginal likelihood and update the model parameter $\btheta$.
It has been reported that such EM-based methods disregard the gradient information from the parameters $\btheta$ to the state trajectory $\vx$, which can potentially degrade the learning performance \cite{chen2022autodifferentiable, courts2023variational}. In contrast, by regarding the latent state trajectory as a function of both the model parameters $\btheta$ and variational parameters $\bzeta$, our method can jointly optimize these parameters in a principled way, resulting in enhanced performance.

Second, it is worth noting that the proposed EnVI algorithm falls under the NMF category (see definition in Appendix \ref{appendix:MF_NMF_definition}), given the exploitation of the dependencies between latent states and the transition function, as evident in the filtering distribution, $p(\x_t \vert \vu, \y_{1:t})$ in Eq.~\eqref{eq:var_q_x_u}. This indicates that EnVI inherits the favorable characteristics of NMF algorithms \cite{doerr2018probabilistic}, which have the potential to enhance learning accuracy and address the issue of underestimating state inference uncertainty, commonly encountered in MF algorithms \cite{turner+sahani:2011a}. 
%Moreover, by leveraging the advantages of the EnKF model-based filtering algorithm, EnVI can achieve superior computational efficiency compared to existing NMF algorithms (e.g. VCDT \cite{ialongo2019overcoming}) that rely on inference networks.
% [\cblue{need to show in the experiments}].
% Last but not least, it is crucial to point out that the proposed EnVI stands out compared to the existing EnKF-based dynamical systems learning methods (see e.g., \cite{chen2023reduced, ishizone2020ensemble} and the references therein) in two fundamental aspects. 

Last but not least, our method offers an effective means of mitigating the risk of overfitting. This is mainly due to the fact that our method leverages the Bayesian non-parametric model and the variational inference framework to derive the ELBO, as presented in Eq.~\eqref{eq:EnVI_lower_bound}, from which the additional regularization terms for the initial state and state transition function can effectively mitigate the overfitting issue. As a comparison, in the state-of-the-art EnKF-based dynamical system learning method, autodifferentiable EnKF (AD-EnKF) \cite{chen2022autodifferentiable,chen2023reduced}, the transition function in the SSM is modeled using a deterministic parametric model, specifically a neural network, and the optimization objective function is the logarithm of marginal likelihood of the model, i.e.
\begin{equation}
\label{eq:loss_AD-EnKF}
    \mathcal{L}_{\text{AD-EnKF}} = \log p(\y_{1:T}) = \sum_{t=1}^T \log p(\y_t \vert \y_{1:t-1}).
\end{equation}
Maximizing this objective function solely w.r.t. the model parameters can easily lead to overfitting. We defer further discussions of this issue to Section \ref{sec:experimental_results}.

% \newpage
% \begin{itemize}
%     \item Connections and differences with the existing works:
%     \begin{itemize}
%         \item existing variational inference methods for GPSSMs \cite{ialongo2019overcoming,doerr2018probabilistic,eleftheriadis2017identification,lin2023towards,lindinger2022laplace,curi2020structured}
%         \item EM algorithm based methods \cite{turner2010state,schon2011system,ghahramani1999learning}
%         \item Maximum likelihood based methods \cite{chen2022autodifferentiable} 
%         \item EnKF-based ELBO construction method \cite{ishizone2020ensemble}
%     \end{itemize} 
%     \item Robustness to the transition process noise
%     \begin{itemize}
%         \item need more empirical results
%     \end{itemize}
% \end{itemize}

%
%
%%%%%%%%%%%%%%%%%%%%%%%%%%%%%%%%%%%%
\section{OEnVI: Online Implementation of EnVI} \label{sec:online_EnVI}
In this section, we further explore online setting where data is processed sequentially, one sample at a time. It is within this context that the simultaneous inference of states and nonlinear dynamics in GPSSMs presents significant challenges \cite{polyzos2021ensemble,zhao2022streaming,zhao2020variational}.  The good news is that our EnVI algorithm readily lends itself to online learning scenarios. Specifically, at each time step $t$, we can naturally maximize the corresponding objective function, denoted as $L_{\ell_t}$, given by
\begin{equation}
L_{\ell_t} = \mathbb{E}_{q(\vu)} \left[ \log p(\y_t \vert \vu, \y_{1:t-1}) \right] - \operatorname{KL}(q(\vu) \| p(\vu)),
\label{eq:OEnVI_obj}
\end{equation}
in terms of both the model parameters and variational parameters. Detailed steps for implementing the online EnVI, termed OEnVI, are summarized in Algorithm \ref{alg:OEnKF}. \cblue{It is worth noting that this algorithm is designed for learning a time-invariant dynamical system, as defined in Section~\ref{subsec:GPSSM}.}

%%%%% ------------------------------------------------------------
\begin{algorithm}[t!]
\caption{OEnVI: Online EnVI (Step $t$)}
\label{alg:OEnKF}
\begin{algorithmic}[1]
\Statex {\bf Input}:  $\btheta = \{\btheta_{gp}, \mathbf{Q}, \mathbf{R}\}, ~\bzeta, ~\y_{t}$, $\x_{t-1}^{1:N}$
\While{\textit{iterations not terminated}}
\State $\vu \sim q(\vu)$, $L_{\ell_t} = 0$
\State Get prediction samples using Eq.~\eqref{eq:EnKF_predictive_GP} \label{eq:EnKF_forecast_alg} 
\State Get empirical moments $\bar{\mathbf{m}}_t, \bar{\mathbf{P}}_t$ using Eq.~\eqref{eq:predictive_distri_enkf_ip}
\State Get Kalman gain: $\bar{\mathbf{G}}_t \!=\! \bar{\mathbf{P}}_t \bm{C}^{\top} (\bm{C} \bar{\mathbf{P}}_t \bm{C}^{\top} \!\!+\! \mathbf{R})^{-1}$
\State  Get updated samples using Eq.~\eqref{eq:updated_samples}
\State Evaluate the objective function
$$
L_{\ell_t} = \log p(\y_t \vert\vu, \y_{1:t-1}) - \operatorname{KL}(q(\vu) \| p(\vu)) 
$$  
\State Maximize $L_{\ell_t}$ and update $\btheta$, $\bzeta$ using Adam \cite{kingma2015adam}
\EndWhile
\Statex {\bf Output}: 
EnKF particles $\x_{t}^{1:N}$,  model parameters $\btheta$, and variational parameters $\bzeta$.
\end{algorithmic}
\end{algorithm}
%
%%
%%%%------------------------------------------------------------

\begin{remark}
Analogous to EnVI, the computational complexity of OEnVI scales as $\mathcal{O}(d_x M^2)$ under practical parallel computing environments.
\end{remark}

An interesting insight for maximizing the objective function of OEnVI in Eq.~\eqref{eq:OEnVI_obj} (as well as the objective function of EnVI in Eq.~\eqref{eq:EnVI_lower_bound}) is that it essentially encourages successful data reconstruction while simultaneously ensuring that the filtering distribution $p(\x_t \vert \vu, \y_{1:t})$ and the prediction distribution $p(\x_t \vert \vu, \y_{1:t-1})$ do not deviate too far from each other, apart from the regularization of $q(\vu)$. This result is supported by the following proposition.
\begin{proposition}\label{prop:2}
The log-likelihood, $\log p(\y_t \vert \vu, \y_{1:t-1})$, essentially is the difference between the data reconstruction error, represented by $\mathbb{E}_{p(\x_t \vert \vu, \y_{1:t})} \left[ \log p(\y_t \vert \x_t)\right]$, and the KL divergence between the filtering distribution $p(\x_t \vert \vu, \y_{1:t})$ and the prediction distribution $p(\x_t \vert \vu, \y_{1:t-1})$. That is, 
\begin{equation}
\begin{aligned}
  \log p(\y_t \vert \vu, \y_{1:t-1}) = & - \operatorname{KL}\left[ p(\x_t \vert \vu, \y_{1:t}) \| p(\x_t \vert \vu, \y_{1:t-1}) \right] \\
  & + \mathbb{E}_{p(\x_t \vert \vu, \y_{1:t})} \left[ \log  p(\y_t \vert \x_{t})\right],  
\end{aligned}
\end{equation}
Thus, an alternative objective function for OEnVI can be expressed as 
\begin{equation}
\label{eq:OEnVI_alternative}
\begin{aligned}
         L_{\ell_t} &= \mathbb{E}_{q(\vu)}\left[ \mathbb{E}_{p(\x_t \vert \vu, \y_{1:t})} [ \log  p(\y_t \vert \x_{t})] \right]\\
        &  \quad - \mathbb{E}_{q(\vu)}\left[ \operatorname{KL} \left[ p(\x_t \vert \vu, \y_{1:t}) \| p(\x_t \vert \vu, \y_{1:t-1}) \right] \right] \\
        & \quad -\operatorname{KL} \left[ q(\vu) \| p(\vu) \right]
\end{aligned}
\end{equation}
\end{proposition}
\begin{proof}
    The proof can be found in Appendix \ref{appendix_prop2_proof}
\end{proof}
\noindent This insight sheds light on the interplay between data reconstruction and the alignment of filtering and prediction distributions in the EnVI and OEnVI algorithms.

Up to this point,  it is worth noting that in contrast to the existing inference network-based variational algorithms \cite{eleftheriadis2017identification,doerr2018probabilistic,ialongo2019overcoming,curi2020structured,lin2023towards,krishnan2017structured,zhao2020variational,zhao2022streaming}, OEnVI is a simple and straightforward extension of EnVI, which benefits from eliminating the dependence on the additional parametric variational distributions. Previous works have typically employed inference networks that take the entire sequence of observations $\y_{1:T}$ as input, necessitating a significant amount of data for offline training. While it is conceivable to constrain the input length of the inference network to a shorter sequence $l$, such as $\y_{t-l:t}$, this approach still leads to prolonged training times and higher computational requirements for optimizing the inference network parameters \cite{zhao2020variational}. Moreover, storing historical inputs $\y_{t-l:t}$ adds to the storage overhead, making it problematic in situations where historical data duplication and storage are not permissible. In sharp contrast, OEnVI successfully overcomes the aforementioned challenges related to the optimization of inference networks. As a result, it facilitates more efficient learning and inference processes, contributing to potential improved overall performance.
Furthermore, OEnVI offers a principled objective, see Eq.~\eqref{eq:OEnVI_alternative}, by simultaneously minimizing the KL divergence, accounting for data reconstruction error balance, and applying regularization to the transition function to mitigate model overfitting.
%This stands in contrast to the existing approaches that sorely minimize the KL divergence between filtering and prediction distributions in the context of online learning \cite{dowling2023real}.

\section{Experiments and Results}
\label{sec:experimental_results}
\vspace{-.05in}    
This section presents a comprehensive numerical study of the proposed EnVI and OEnVI. Section \ref{subsec:LGSSM} showcases the filtering performance. In Section \ref{subsec:Learning_kink_func}, we present the system dynamics learning performance. The series forecasting performance of EnVI on various real datasets is illustrated in Section \ref{subsec:systemID_data_forecasting}. Finally, Section \ref{subsec:OEnVI_performance} provides a comprehensive demonstration of the performance of the OEnVI online algorithm.  More details regarding the experimental setup can be found in supplementary material \cite{lin2023ensemble}, and the accompanying source code is publicly available online\footnote{\url{https://github.com/zhidilin/gpssmProj}}.

\begin{figure*}[t!]
    \centering
    \subfloat{\includegraphics[width=.96\linewidth]{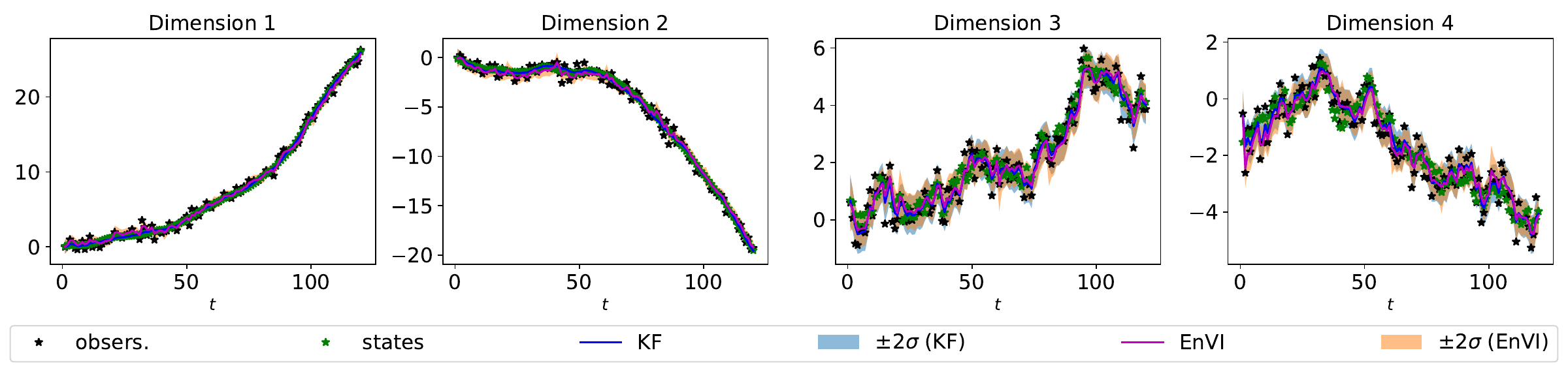}} \\ \vspace{-.07in}    
    \subfloat{\includegraphics[width=.96\linewidth]{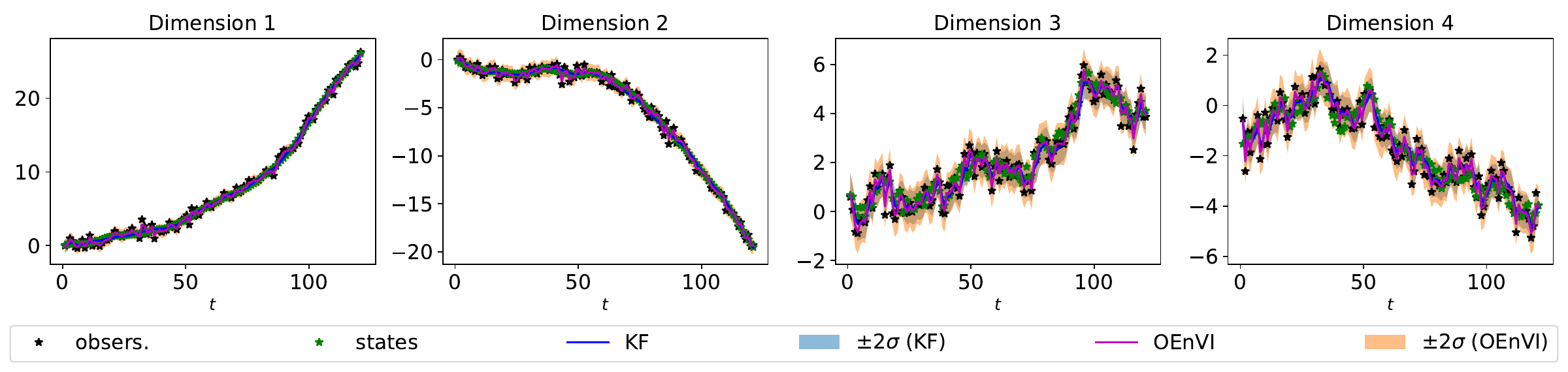}}  \vspace{-.07in}    
    \caption{EnVI (\textbf{top}) \& OEnVI (\textbf{bottom}) on state inference in linear Gaussian SSM. The RMSE of the latent state estimates for KF, EnVI, and OEnVI are 0.5252, 0.6841, and 0.7784, respectively; the RMSE between the observations and the latent states is 0.9872. \vspace{-.12in}}
    \label{fig_maintext:LGSSMs_EnVIs}
\end{figure*}

\subsection{Learning and Inference in Linear Gaussian SSMs} \label{subsec:LGSSM}
We investigate the learning and inference capacity of EnVI and OEnVI, by using a linear Gaussian state-space model (LGSSM) where exact inference of latent state is applicable (i.e. KF). Specifically, we use the following LGSSM,  a car tracking example given in the textbook \cite{sarkka2013bayesian}, to generate observation data, $\y_{1:T}$, for training EnVI and OEnVI,
\begin{subequations}
\label{eq:LGSSMs}
\begin{align}
\mathbf{x}_t & =\mathbf{H} \mathbf{x}_{t-1}+\mathbf{v}_{t-1}, & \mathbf{v}_{t-1} & \sim \mathcal{N}(\mathbf{0}, \mathbf{Q}), \\
\mathbf{y}_t & =\bm{C} \mathbf{x}_t+\mathbf{e}_t, & \mathbf{e}_t & \sim \mathcal{N}(\mathbf{0}, \mathbf{R}),
\end{align}
\end{subequations}
where the state and the observation are both four dimensional, and $\boldsymbol{C} = \bm{I}_{4\times 4}, \mathbf{R} = \sigma_{\mathrm{R}}^2 \bm{I}_{4\times 4}$ with $\sigma_{\mathrm{R}} = 0.5$;
\begin{equation}
\mathbf{H}=\left(\begin{array}{cccc}
1 & 0 & \Delta t & 0 \\
0 & 1 & 0 & \Delta t \\
0 & 0 & 1 & 0 \\
0 & 0 & 0 & 1
\end{array}\right)
\end{equation}
and
\begin{equation}
\mathbf{Q}=\left(\begin{array}{cccc}
\frac{q_1^{\mathrm{c}} \Delta t^3}{3} & 0 & \frac{q_1^{\mathrm{c}} \Delta t^2}{2} & 0 \\
0 & \frac{q_2^{\mathrm{c}} \Delta t^3}{3} & 0 & \frac{q_2^{\mathrm{c}} \Delta t^2}{2} \\
\frac{q_1^{\mathrm{c}} \Delta t^2}{2} & 0 & q_1^{\mathrm{c}} \Delta t & 0 \\
0 & \frac{q_2^{\mathrm{c}} \Delta t^2}{2} & 0 & q_2^{\mathrm{c}} \Delta t
\end{array}\right)
\end{equation}
with $\Delta t = 0.1, q_1^{\mathrm{c}} = q_2^{\mathrm{c}} = 1$. 

We begin by generating $T=120$ training observations.  For EnVI, we employ $1000$ epochs/iterations for training, but convergence is typically achieved approximately $300$ iterations. In OEnVI, the parameters $\btheta$ and $\bzeta$ are updated once per time step $t$. Both EnVI and OEnVI employ $15$ inducing points, a setting that will be used for subsequent experiments unless otherwise specified. 
We report the state inference results, which are depicted in Fig.~\ref{fig_maintext:LGSSMs_EnVIs}.  It can be observed that the state inference performance of EnVI and OEnVI is comparable to that of the KF in terms of state-fitting root mean square error (RMSE), despite being trained solely on noisy observations without any physical model knowledge.  

Another finding is that though OEnVI incurs a lower training cost compared to EnVI, this advantage comes at the expense of inadequate learning of the latent dynamics, leading to a less accurate estimation of the latent states when compared to EnVI. The discrepancy is evident in Fig.~\ref{fig_maintext:LGSSMs_EnVIs}, where OEnVI exhibits larger estimation RMSE and greater estimation uncertainty for the latent states in comparison to EnVI and KF. Notably, EnVI relies on offline training, resulting in an uncertainty quantification that closely approaches the optimal estimate, the KF estimate. Nevertheless, it is essential to mention that, with continuous online data arrival, OEnVI can eventually achieve a comparable state estimation performance as EnVI. We observed that after observing $360$ data points, OEnVI achieves a latent state RMSE estimation of $0.6512$. Further details on this aspect of the results are provided in Supplement~\ref{supp_subsec:OEnVI_LGSSMs}, \cblue{where we also show and discuss the inference performance under different emission coefficient matrices $\bm{C}$. }

\begin{table*}[t!]
	\centering
	\caption{Comparison of our method with other competitors on the kink function dataset. Shown are mean and standard errors over five repetitions of the \textbf{fitting MSE (lower is better)} and the \textbf{log-density (higher is better)} of the kink function varying the emission noise variance $\sigma_{\mathrm{R}}^2$. } \vspace{-.05in}
	\setlength{\tabcolsep}{2mm}{
		\centering
		\begin{tabular}{ r  lll}
			\toprule
			Method & $\sigma_{\mathrm{R}}^2=0.008$ (MSE $\mid$ Log-Likelihood) & $\sigma_{\mathrm{R}}^2=0.08$ (MSE $\mid$ Log-Likelihood) & $\sigma_{\mathrm{R}}^2=0.8$ (MSE $\mid$ Log-Likelihood)  \\
			\midrule
			{vGPSSM \cite{eleftheriadis2017identification}}    
                        & $1.0410 \!\pm\! 0.7426 \mid \!-\!27.5981 \! \pm \! 19.7817$  
                        & $1.6390 \!\pm\! 0.6783  \mid  \!-\!30.9557 \!\pm \!16.9218 $ 
                        & $1.9584 \!\pm\! 0.9655 \mid \!-\!56.5997\!\pm\! 37.8221$      \\
                {VCDT \cite{ialongo2019overcoming}}  
                        & $0.2057  \!\pm\! 0.2219 \mid \!-\!1.058 \!\pm\! 1.5005$  
                        & $0.1934 \!\pm\! 0.0140  \mid \!-\!0.5867 \!\pm\! 0.2610$ 
                        & $ 1.4035 \!\pm\! 0.6470  \mid \!-\!3.8092\!\pm\! 0.6588$   \\
                {AD-EnKF \cite{chen2022autodifferentiable}} 
                        & $0.0285\!\pm\! 0.0318 \mid \!-\!3.6282 \! \pm \! 6.3514$   
                        & $1.5246 \! \pm \! 0.9734 \mid \!-\!242.2795 \!\pm\!194.6741$ 
                        & $1.3489 \! \pm \! 0.3102 \mid \!-\!267.7068\!\pm\!62.0488$  \\
			{EnVI (ours)} 
                        & $\bm{0.0046} \!\pm\! \bm{0.0025} \mid \bm{1.1060} \!\pm\! \bm{0.0381}$  
                        & $\bm{0.0536} \!\pm\! \bm{0.0232} \mid \bm{0.1025} \!\pm\! \bm{0.1075}$ 
                        & $\bm{0.5315} \!\pm\! \bm{0.1542} \mid \bm{\!-\!1.0439} \!\pm\! \bm{0.1714}$ \\
			\bottomrule
		\end{tabular}
	}
 \vspace{-.12in}
 \label{tab:synthetic_dataset_MSE_LL}
\end{table*}
\begin{figure*}[t!]
    \centering
    \subfloat[vGPSSM]{\includegraphics[width=.48\textwidth]{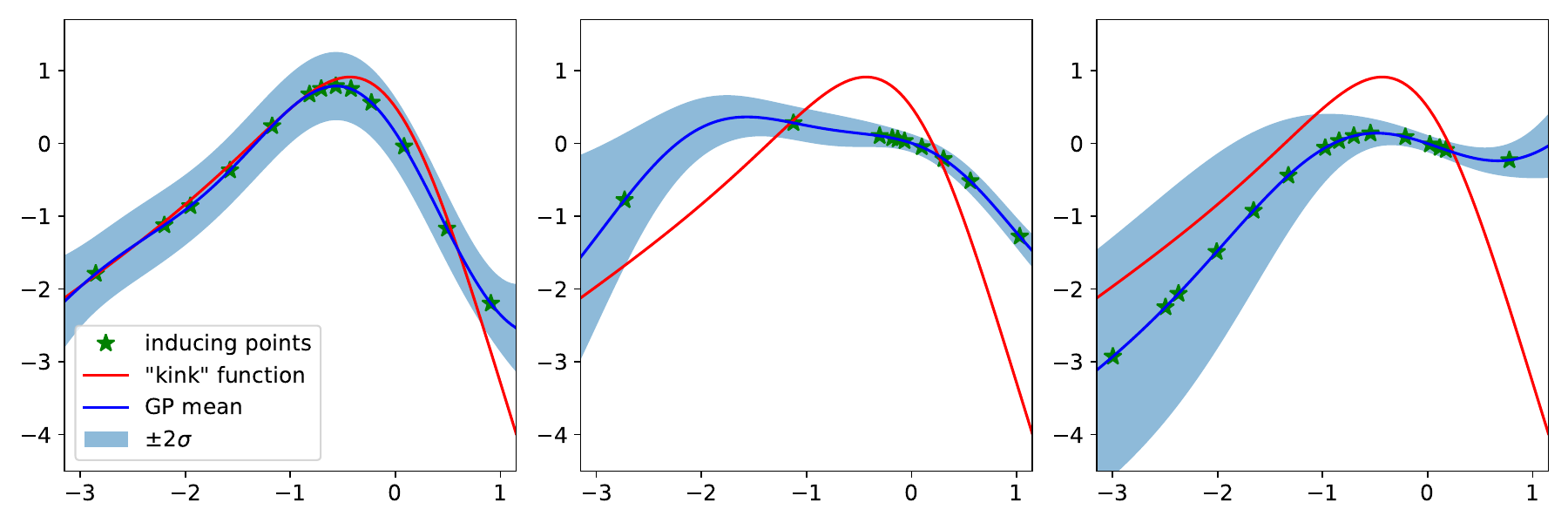} \label{subfig:kink_envi_vgpssm}} \hfill 
    \subfloat[VCDT]{\includegraphics[width=.48\textwidth]{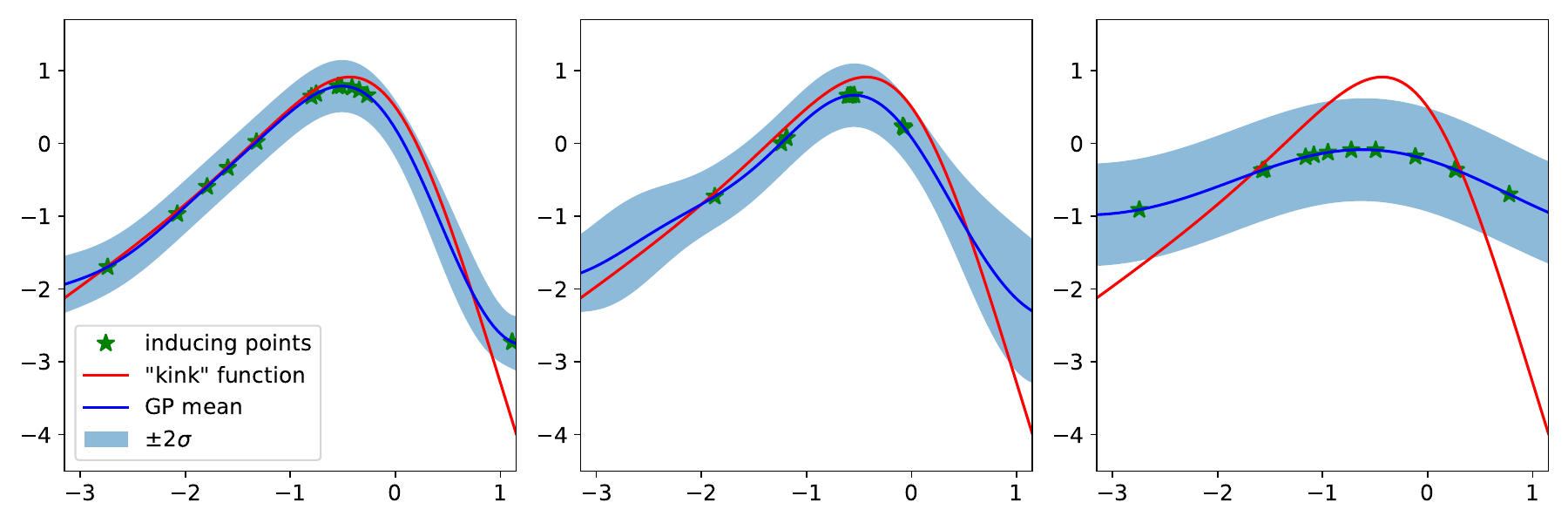} \label{subfig:kink_envi_vcdt}}
    \vspace{-.05in}
    
    \subfloat[AD-EnKF]{\includegraphics[width=.48\textwidth]{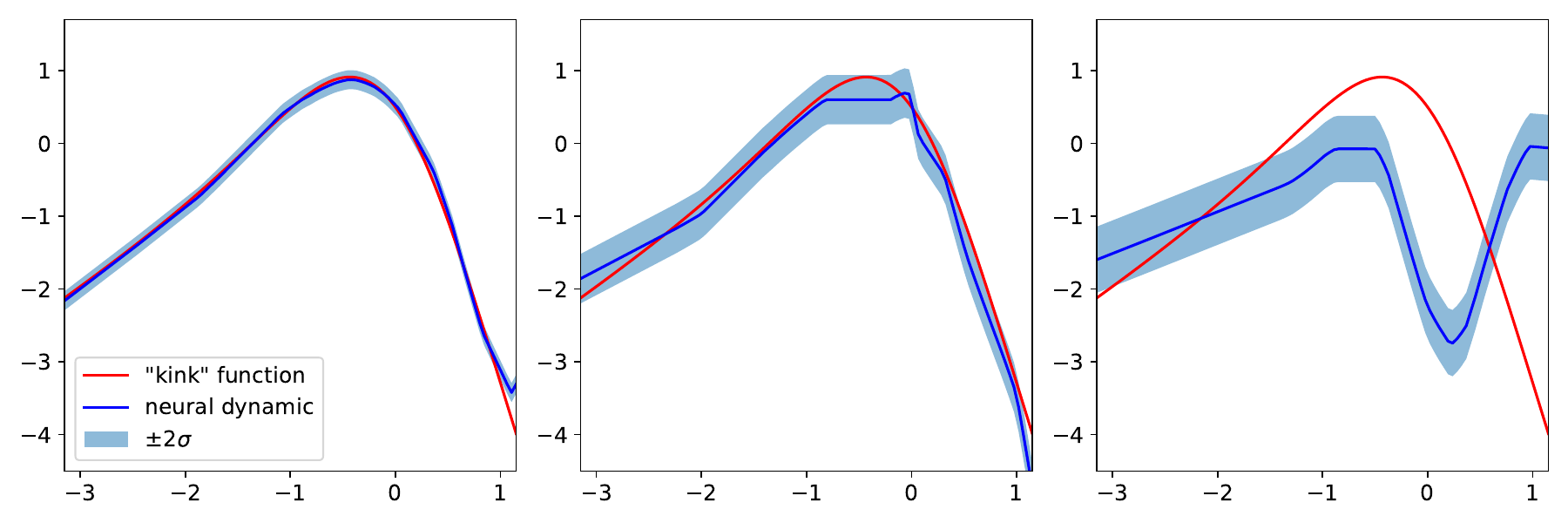} \label{subfig:kink_envi_AD-EnKF}} \hfill 
    \subfloat[EnVI]{\includegraphics[width=.48\textwidth]{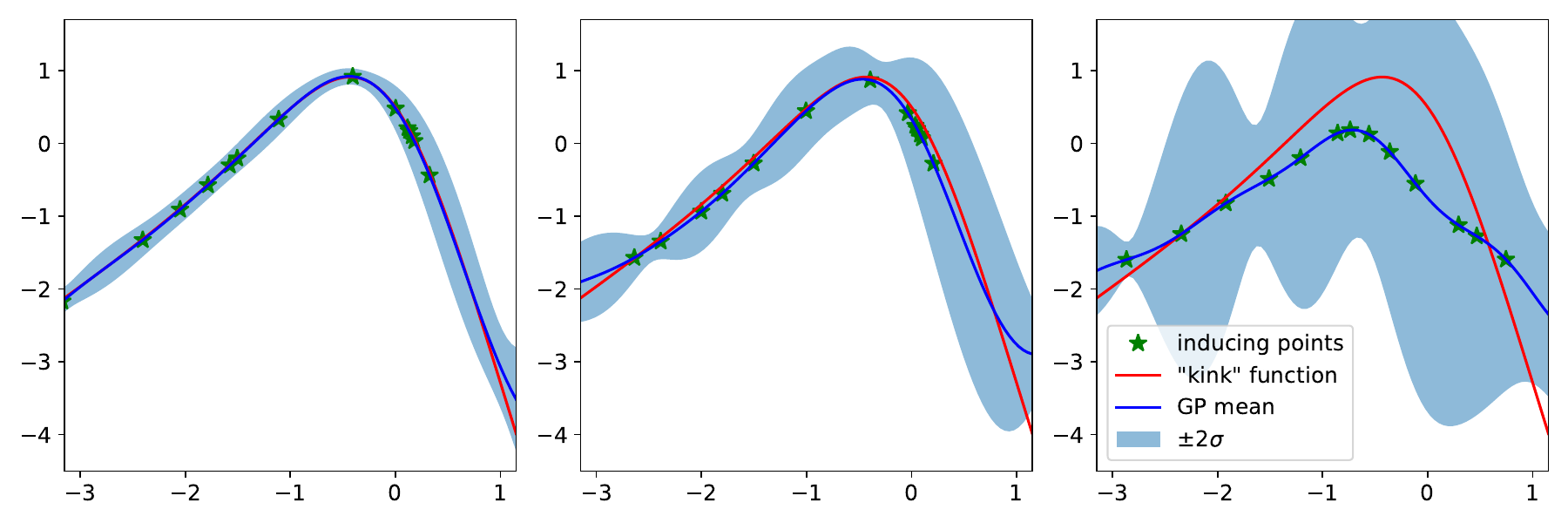} \label{subfig:kink_envi_EnVI}} 

    \caption{Kink transition function learning performance (mean $\pm$ $2 \sigma$) using various methods across different levels of emission noise ($\sigma_{\mathrm{R}}^2 \in \{0.008, 0.08, 0.8\}$, from left to right).}
    \vspace{-.12in}
    \label{fig:kink_EnVI}
\end{figure*}

\vspace{-.1in}
\subsection{System Dynamics Learning} 
\label{subsec:Learning_kink_func}     
% \vspace{-.01in}
This subsection demonstrates the superior capability of EnVI in learning latent dynamics for GPSSMs. To evaluate its performance, we utilize a 1-D synthetic dataset called the \textit{kink} function dataset, which is generated from a dynamical system described by Eq.~\eqref{eq:kink_function}, where the nonlinear, smooth and time-invariant function $f(\x_t)$ is called the ``kink'' function,
\begin{subequations}	
	\label{eq:kink_function}
	\begin{align}
		& \x_{t+1}  = \underbrace{0.8 + \left(\x_t+0.2\right)\left[1 - \frac{5}{1+ \exp(-2\x_t)}\right]}_{\triangleq \text{ ``kink function'' }f(\x_t)} + \mathbf{v}_{t},\\
		& \y_t = \x_t + \mathbf{e}_t, \ \ \ \mathbf{v}_{t} \sim \cN(0, \sigma_{\mathrm{Q}}^2), \ \ \ \mathbf{e}_t\sim \cN(0, \sigma_{\mathrm{R}}^2).
	\end{align}
\end{subequations}
It is worth mentioning that this specific dynamical system has been extensively employed in GPSSM literature to evaluate the accuracy of the learned GP transition posterior \cite{lindinger2022laplace}. 

In showcasing the superior performance of EnVI, we compare it against several prominent competing methods, namely vGPSSM \cite{eleftheriadis2017identification}, VCDT \cite{ialongo2019overcoming}, and AD-EnKF \cite{chen2022autodifferentiable}. Our implementation of VCDT incorporates an inference network to address the linear growth in the number of variational parameters. The vGPSSM method adheres to the original paper's implementation \cite{eleftheriadis2017identification}, while the AD-EnKF is utilized as per the default software package available online\footnote{\url{https://github.com/ymchen0/torchEnKF}}. 
The training data sequences are generated by fixing $\sigma_\mathrm{Q}^2$ at 0.05 and systematically vary $\sigma_\mathrm{R}^2$ within the range of $\{0.008, 0.08, 0.8\}$.  As a result, we generate three sets of $T=600$ observations each for training. To ensure a fair comparison in the latent space, we adhere to Ialongo \textit{et al}. \cite{ialongo2019overcoming} and keep the emission model fixed to the true generative ones for all methods, while allowing the transition to be learned. Further details, including the description of the setup for the aforementioned algorithms, are provided in Supplement~\ref{supp_subsec:EnVI_kink} and the accompanying source code. The result is depicted in Table~\ref{tab:synthetic_dataset_MSE_LL} and visualized in Fig.~\ref{fig:kink_EnVI}. We observe that EnVI consistently excels in system dynamic learning and exhibits superior learning robustness compared to existing methods.
We next conduct two ablation studies.

\textbf{\textit{EnVI} vs. \textit{Inference Network-Based Methods}}. 
Based on the numerical results presented in Table \ref{tab:synthetic_dataset_MSE_LL}, we can find that the EnVI exhibits superior dynamic learning performance compared to vGPSSM and VCDT, both of which rely on an inference network. Specifically, as illustrated in Fig.~\ref{fig:kink_EnVI}, vGPSSM faces more challenges in dynamic learning, while VCDT, categorized under the NMF paradigm, only performs well under more minor noise conditions ($\sigma_{\mathrm{R}}^2 = 0.008$ and $\sigma_{\mathrm{R}}^2 = 0.08$). In contrast, EnVI effectively learns the GP transition well, even in the high noise setting with $\sigma_{\mathrm{R}}^2 = 0.8$.  

The primary reason for this discrepancy is the increased model and computational complexity arising from additional inference network parameters, which hinder effective training. Furthermore, the inference network-based methods are prone to convergence into various unfavorable local optima, demonstrating reduced robustness. As a consequence, the learning performance of such methods often fluctuates significantly, see Table \ref{tab:synthetic_dataset_MSE_LL}. In contrast, EnVI inherits the benefits of EnKF and avoids the need to optimize additional variational parameters from the inference network, making it more amenable to optimization and demonstrating enhanced robustness. 

Our experimental findings consistently demonstrate that EnVI exhibits rapid convergence compared to vGPSSM and VCDT, owing to its streamlined parameterization. For instance, as shown in Fig.~\ref{fig:MSE_kink}, when considering $\sigma_{\mathrm{R}}^2 = 0.008$, EnVI achieves convergence after 300 iterations, whereas both vGPSSM and VCDT require many more iterations. \cblue{Moreover, EnVI also exhibits a noticeable reduction in piratical runtime per iteration compared to the two competitors, which need to optimize additional inference network parameters.
This underscores the efficiency improvement brought by streamlining the inference process using EnKF.
}

\begin{figure*}[t!]
    \centering
    \subfloat{\includegraphics[width=.48\linewidth]{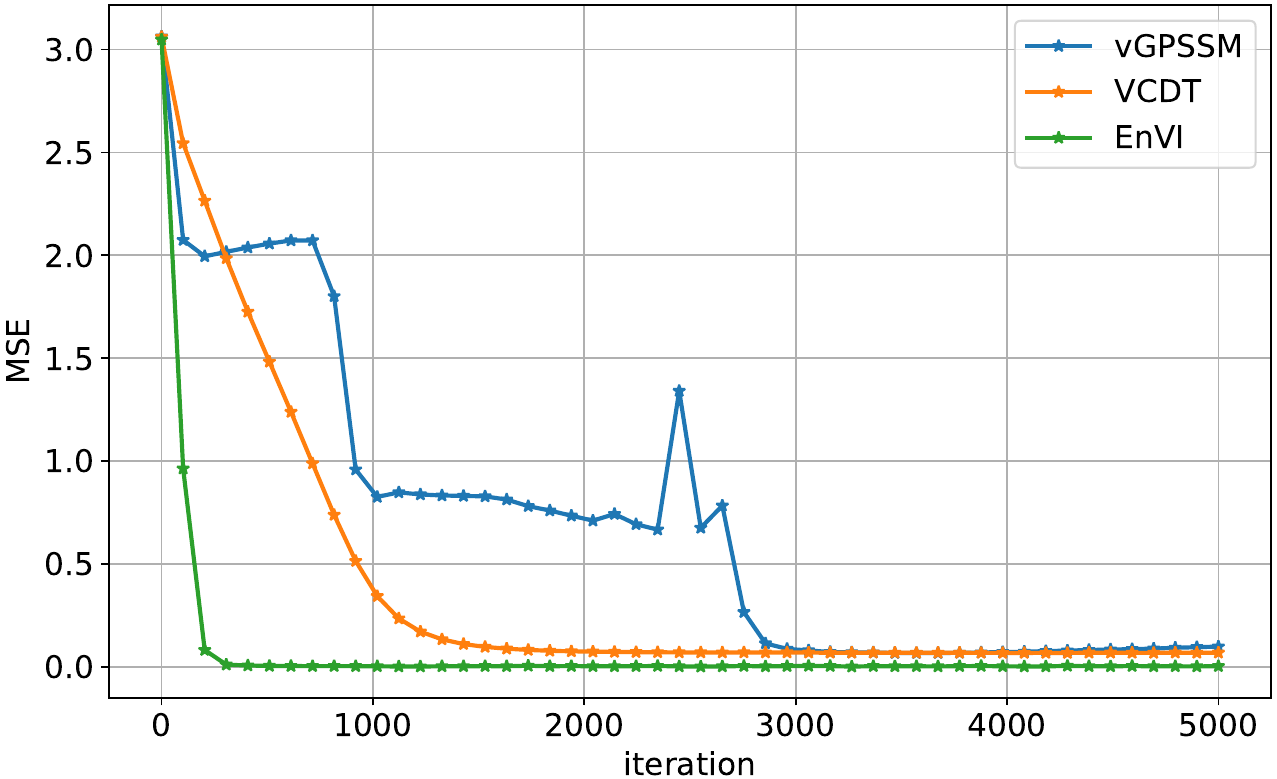}} \hfill
    \subfloat{\includegraphics[width=.485\linewidth]{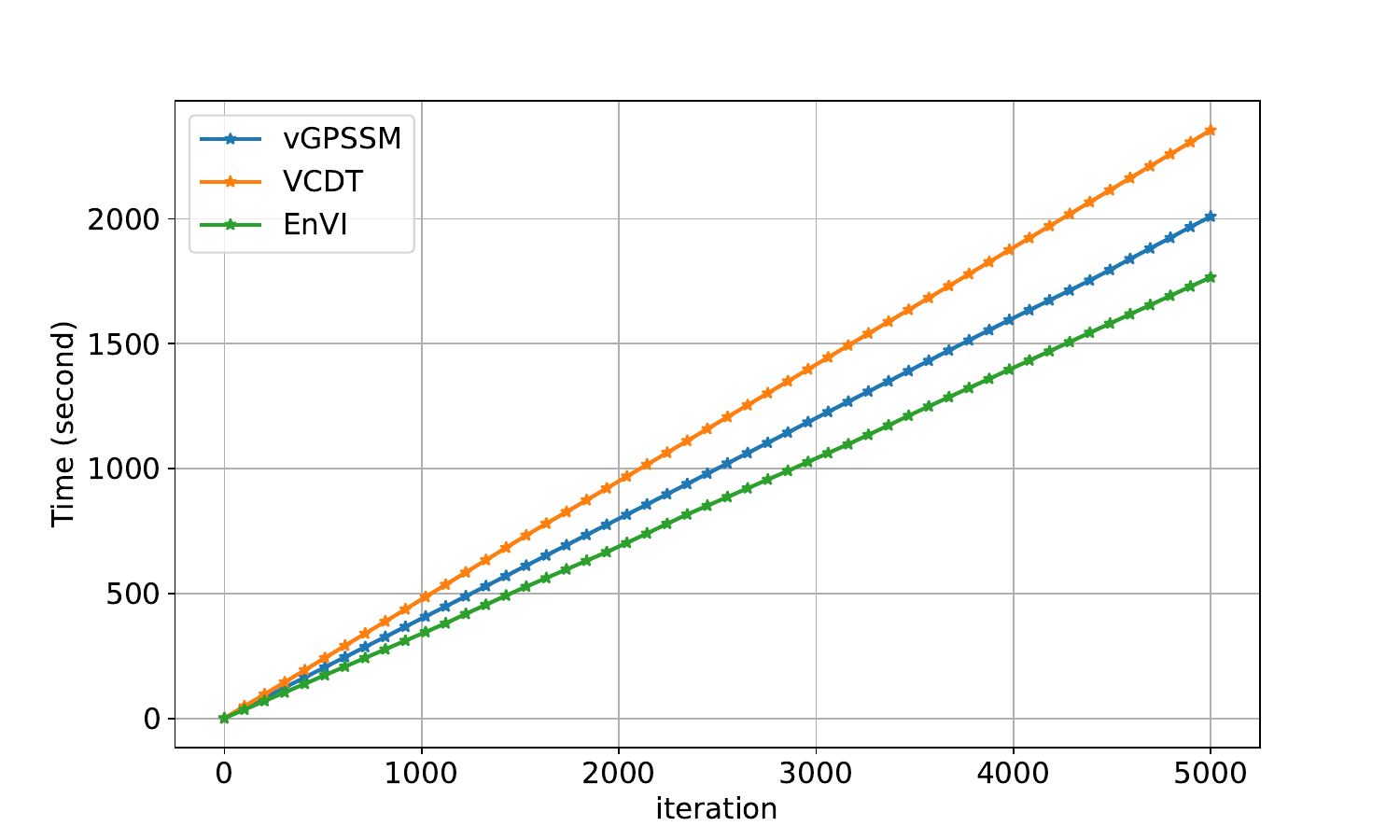}} 
    \vspace{-.1in}
    \caption{
    Kink function learning performance against the training iterations. EnVI \cblue{exhibits} rapid convergence compared to vGPSSM and VCDT.
    \vspace{-.15in}
    }
    \label{fig:MSE_kink}
\end{figure*}

\textbf{\textit{EnVI} vs. \textit{AD-EnKF}}.
The primary difference between EnVI and AD-EnKF lies in their data-driven modules. AD-EnKF employs a parametric model, specifically a neural network (so it is also known as a deep SSM (DSSM)), while EnVI utilizes a non-parametric GP. Consequently, the dynamics learned by AD-EnKF exhibit a tendency to be over-confident, as depicted in Fig.~\ref{fig:kink_EnVI}, as it does not account for uncertainties from the learned transition function. Moreover, the absence of regularization during the training of the neural network in AD-EnKF (as indicated in Eq.~\eqref{eq:loss_AD-EnKF}) renders the method susceptible to overfitting and being trapped in suboptimal solutions during the training process. As a result, the performance of different repetitions can vary significantly, as shown in Table \ref{tab:synthetic_dataset_MSE_LL}.

% \begin{table*}[t!]
% 	\centering
% 	\caption{Comparison of our method with other competitors on the kink function dataset. Shown are mean and standard errors over five repetitions of the \textbf{log-density (higher is better)} of the kink function varying the emission noise variance $\sigma_{\mathrm{R}}^2$. }
% 	\setlength{\tabcolsep}{3mm}{
% 		\centering
% 		\begin{tabular}{ r  ccc}
% 			\toprule
% 			Model & $\sigma_{\mathrm{R}}^2=0.008$ & $\sigma_{\mathrm{R}}^2=0.08$ & $\sigma_{\mathrm{R}}^2=0.8$  \\
% 			\midrule
% 			{vGPSSM \cite{eleftheriadis2017identification}}  & $-27.5981 \! \pm \! 19.7817$ & $\!-30.9557 \pm\!16.9218 $  & $-56.5997\!\pm\!37.8221$   \\
%                 {VCDT \cite{ialongo2019overcoming}}  & $-1.058 \!\pm\! 1.5005$   & $-0.5867 \!\pm\! 0.2610$  & $ -3.8092\!\pm\! 0.6588 $    \\
%                 {AD-EnKF \cite{chen2022autodifferentiable}}  & $-3.6282 \! \pm \! 6.3514$ & $\!-242.2795 \pm\!194.6741 $  & $-267.7068\!\pm\!62.0488$\\
% 			{EnVI (proposed)} & $\bm{1.1060} \!\pm\! \bm{0.0381}$   & $\bm{0.1025} \!\pm\! \bm{0.1075}$  & $\bm{-1.0439} \!\pm\! \bm{0.1714}$\\
% 			\bottomrule
% 		\end{tabular}
% 	}\label{tab:synthetic_dataset_LL}
% \end{table*}

%
\vspace{-.1in}
\subsection{Time Series Data Forecasting} \label{subsec:systemID_data_forecasting}  \vspace{-.03in}
This subsection further demonstrates the series prediction performance of the proposed EnVI algorithm on five public real-world system identification datasets\footnote{\url{https://homes.esat.kuleuven.be/~smc/daisy/daisydata.html}}, which consist of one-dimensional time series of varying lengths between $296$ to $1024$ data points. In addition to the comparison methods discussed in Section \ref{subsec:Learning_kink_func}, EnVI is also compared with several other competitors, including two NMF class methods, PRSSM \cite{doerr2018probabilistic}, ODGPSSM \cite{lin2022output}, and two inference network-based methods, DKF \cite{krishnan2017structured} and CO-GPSSM \cite{lin2023towards}, as depicted in Table \ref{tab:systemidentifcation}. For each method, the first half of the sequence in every dataset is utilized as training data, with the remaining portion designated for testing. Standardization of all datasets is conducted based on the training sequence, and the latent state dimension, $d_x$, is consistently set to $4$ for all datasets. Table \ref{tab:systemidentifcation} reports the series prediction results, wherein the RMSE is averaged over 50-step ahead forecasting.

Table~\ref{tab:systemidentifcation} reveals that EnVI outperforms almost all methods across the five datasets. Specifically, EnVI demonstrates superior performance among the MF and NMF methods. Compared to PRSSM and ODGPSSM \cite{doerr2018probabilistic,lin2022output}, which assume equality between variational and prior distributions of latent states, EnVI employs EnKF to filter latent states, leading to an enhanced system dynamics learning performance, and consequently, improving the series predictions. Compared to the inference network-based methods, like VCDT \cite{ialongo2019overcoming} and the MF class methods vGPSSM \cite{eleftheriadis2017identification} and CO-GPSSM \cite{lin2023towards},
EnVI eliminates the need to optimize inference network parameters. From an optimization solution perspective,  optimizing the inference network leads to a vast solution space for the variational distribution. Consequently, despite their adequate approximation capabilities for the true posterior distribution, these inference network-based methods often fall short of realizing their theoretical potential in empirical performance due to numerous bad local optimums \cite{lindinger2022laplace}.  In contrast,  EnVI imposes model-based constraints on the variational distribution by the EnKF, narrowing the solution space and yielding significantly improved and robust empirical performance.

Compared to the DSSM methods, EnVI offers performance advantages due to its non-parametric GP model. In contrast to DKF \cite{krishnan2017structured}, which utilizes neural networks to model variational distributions and nonlinear SSMs, EnVI employs GPs with much less model parameters, making it particularly suitable for small datasets. While AD-EnKF \cite{chen2022autodifferentiable} outperforms DKF, its deterministic neural network modeling approach and the absence of regularization in its objective function cause it to lag behind EnVI in forecasting performance.

%%%%%%%-------------------------------------------------------------
\begin{table*}[ht!]
	\centering
	\caption{Prediction performance (RMSE) of the different models on the system identification datasets. Mean and standard deviation of the prediction results are shown across five seeds. The lowest RMSE is highlighted in bold.} \vspace{-.05in}
	\setlength{\tabcolsep}{2.8mm}{
		\centering
		\begin{tabular}{c| r | ccccc }
			\toprule
			% \hline
                Category & Method & \multicolumn{1}{c}{Actuator} &  \multicolumn{1}{c}{Ball Beam} &  \multicolumn{1}{c}{Drive} &  \multicolumn{1}{c}{Dryer} &  \multicolumn{1}{c}{Gas Furnace}\\
   			\midrule
			\multirow{2}{1.5cm}{\centering DSSMs} 
			& \textbf{DKF} \cite{krishnan2017structured}
			& $1.204 \pm 0.250$   
			& $0.144 \pm 0.005$  
			& $0.735 \pm 0.001$   
			& $1.465 \pm 0.087$   
			& $5.589 \pm 0.066$    \\
   			& \textbf{AD-EnKF} \cite{chen2022autodifferentiable}
			&  $0.705 \pm 0.117$  
			&  $0.057 \pm 0.006$
			&  $0.756 \pm 0.114$ 
			&  $0.182 \pm 0.053$ 
			&  $1.408 \pm 0.090$  \\
			\midrule
			\multirow{2}{1.5cm}{\centering MF-based Methods}
			& \textbf{vGPSSM}  \cite{eleftheriadis2017identification}
			&  $1.640 \pm 0.011$ 
			&  $0.268 \pm 0.414$
			&  $0.740 \pm 0.010$
			&  $0.822 \pm 0.002$
			&  $3.676 \pm 0.145$   \\
			& \textbf{CO-GPSSM} \cite{lin2023towards}
			&  $0.803 \pm 0.011$
			&  $0.079 \pm 0.018$
			&  $0.736 \pm 0.007$
			&  $0.366 \pm 0.146$
			&  $1.898 \pm 0.157$ \\
   % 			& \textbf{CO-TGPSSM} \cite{lin2023towards}
			% &   
			% &   
			% &  
			% &   
			% &   \\
			\midrule
			\multirow{4}{1.5cm}{\centering NMF-based Methods} 
			& \textbf{PRSSM}  \cite{doerr2018probabilistic}
                &  $0.691 \pm 0.148$
                &  $0.074 \pm 0.010$
                &  $\bm{0.647} \pm \bm{0.057}$
                &  $0.174 \pm 0.013$  
                &  $1.503 \pm 0.196$   \\ 
			& \textbf{ODGPSSM}  \cite{lin2022output}
                &  $0.666 \pm 0.074$
                &  $0.068 \pm 0.006$
                &  $0.708 \pm 0.052$ 
                &  $0.171 \pm 0.011$  
                &  $1.704 \pm 0.560$ \\
   			& \textbf{VCDT}  \cite{ialongo2019overcoming}
			&  $0.815 \pm 0.012$
			&  $0.065 \pm 0.005$
			&  $0.735 \pm 0.005$
			&  $0.667 \pm 0.266$
			&  $2.052 \pm 0.163$ \\
   			& \textbf{EnVI} (ours)   
			&  $\bm{0.657}\pm \bm{0.095}$
			&  $\bm{0.055} \pm \bm{0.002}$
			&  $0.703 \pm 0.050$
			&  $\bm{0.125} \pm \bm{0.017}$
			& $\bm{1.388} \pm \bm{0.123}$ \\
                % \midrule
			% \multirow{1}{2.5cm}{\centering Online Method} 
			% & \textbf{SVMC}  
			% &   
			% &   
			% &   
			% &   
			% &   \\
			% \midrule
			% \multirow{1}{2.5cm}{\centering \textbf{Proposed}} 
			% & \textbf{EnVI}   
			% & $\bm{0.1007} \pm \bm{0.0591}$
			% & $0.2227\pm 0.0417$
			% & $\textbf{0.7708}\pm\textbf{0.1302}$
			% & $1.5288\pm0.1400$
			% & $0.4156\pm0.0532$ \\
			\bottomrule
	\end{tabular}}
	\label{tab:systemidentifcation}
\vspace{-.12in}
\end{table*}
%%%%%----------------------------------------------------------------------
\begin{figure*}[t!]
    \centering
    \subfloat[\cblue{True} and \cred{inferred} latent trajectory using OEnVI]{\includegraphics[width=.33\textwidth]{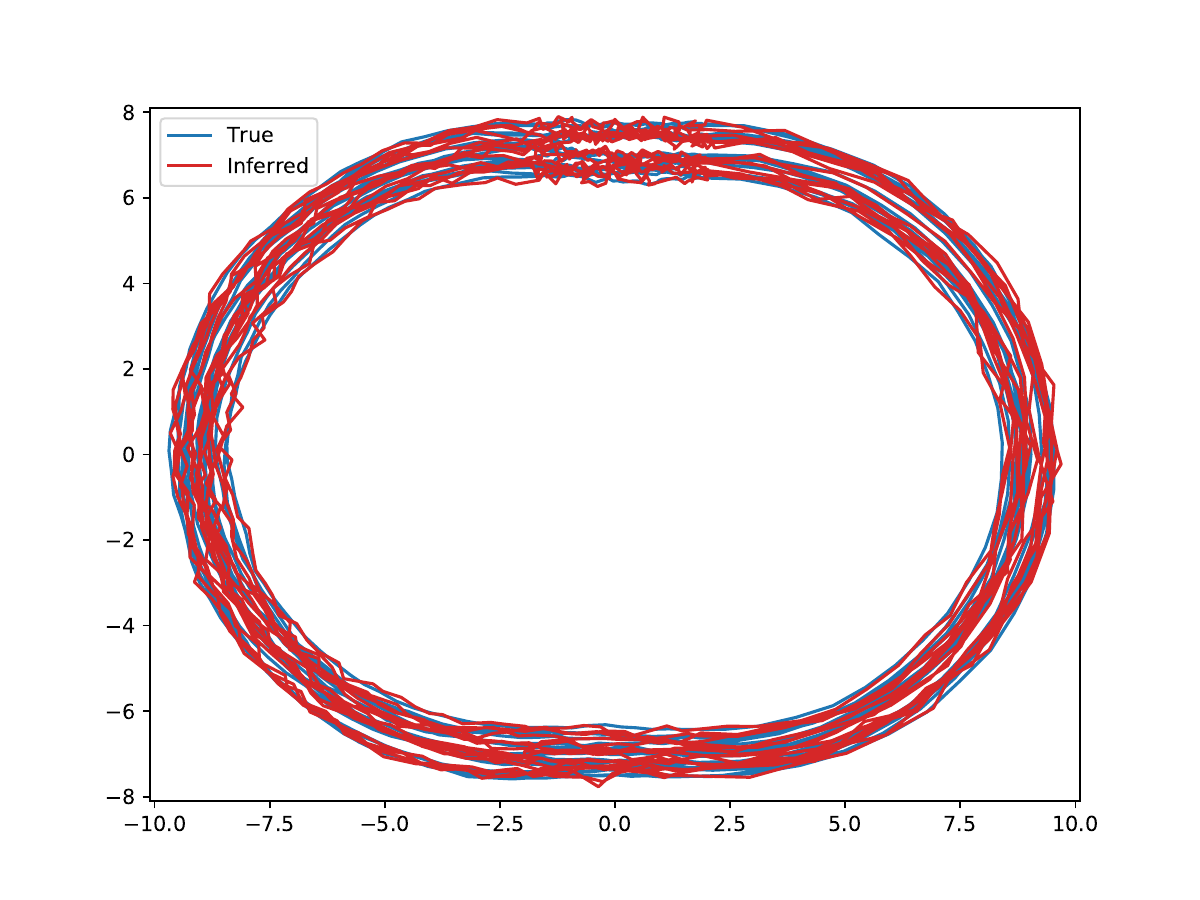} \label{fig:NASCAR_a}} 
    \subfloat[\cblue{True} and \cred{inferred} latent trajectory using SVMC]{\includegraphics[width=.33\textwidth]{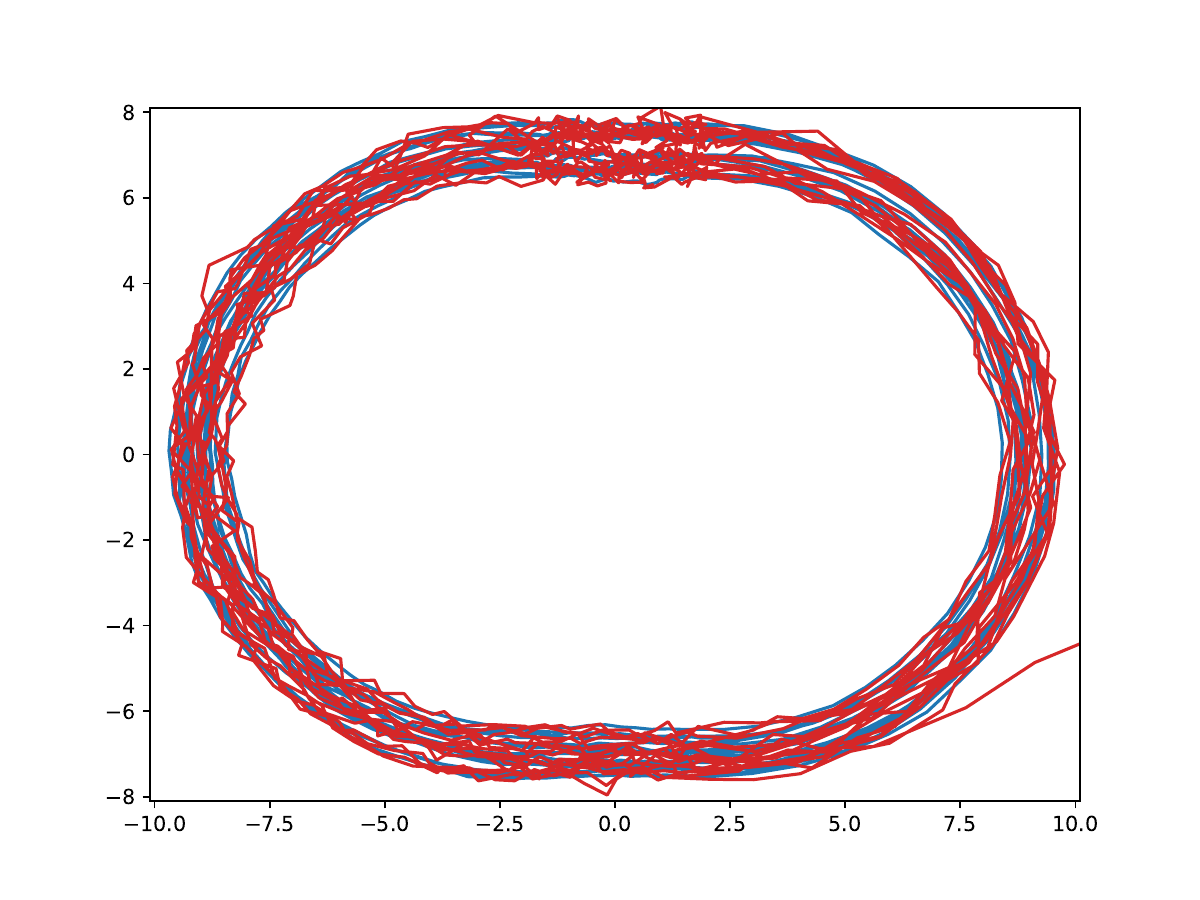} \label{fig:NASCAR_b}} 
    \subfloat[\cblue{True} and \cred{inferred} latent trajectory using VJF]{\includegraphics[width=.33\textwidth]{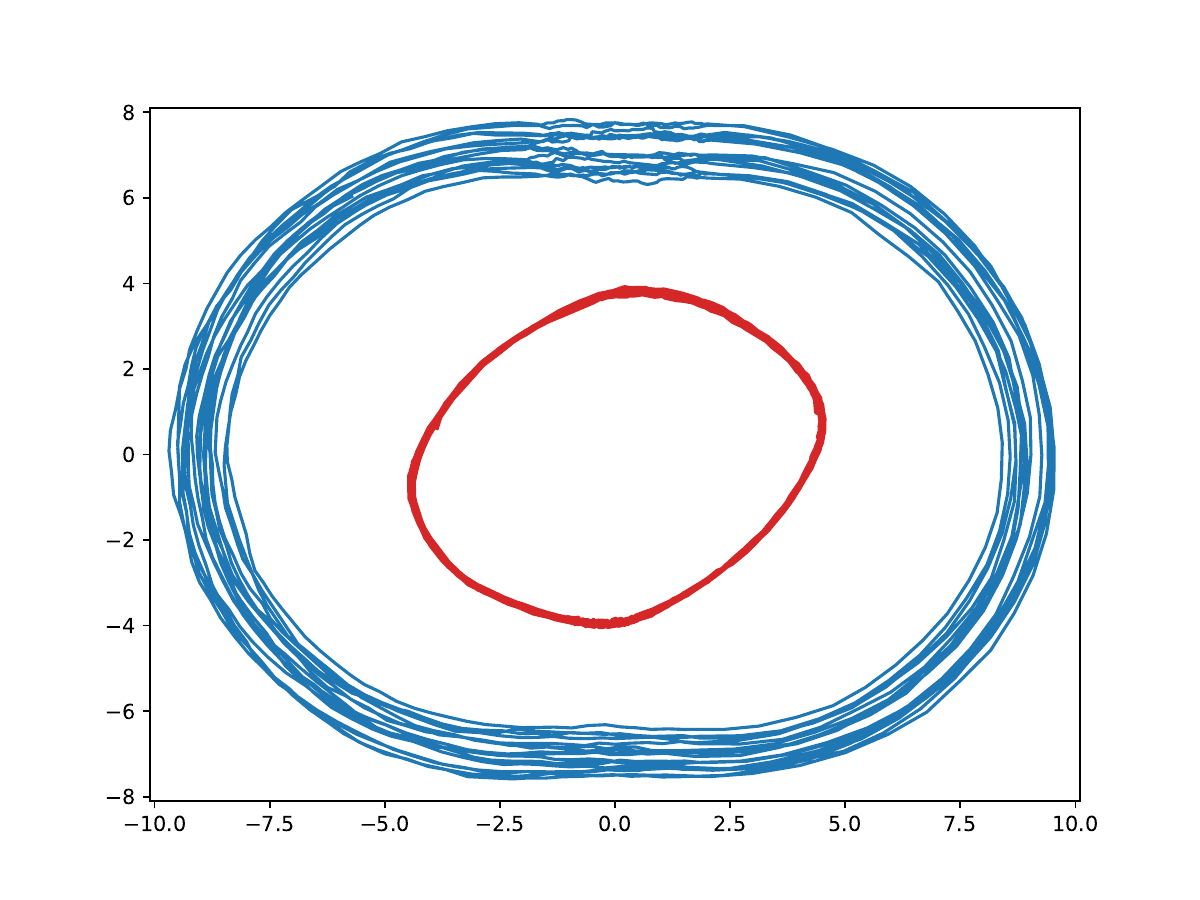} \label{fig:NASCAR_c}} 

    \subfloat[Filtering and prediction using OEnVI (ours)]{\includegraphics[width=.33\textwidth]{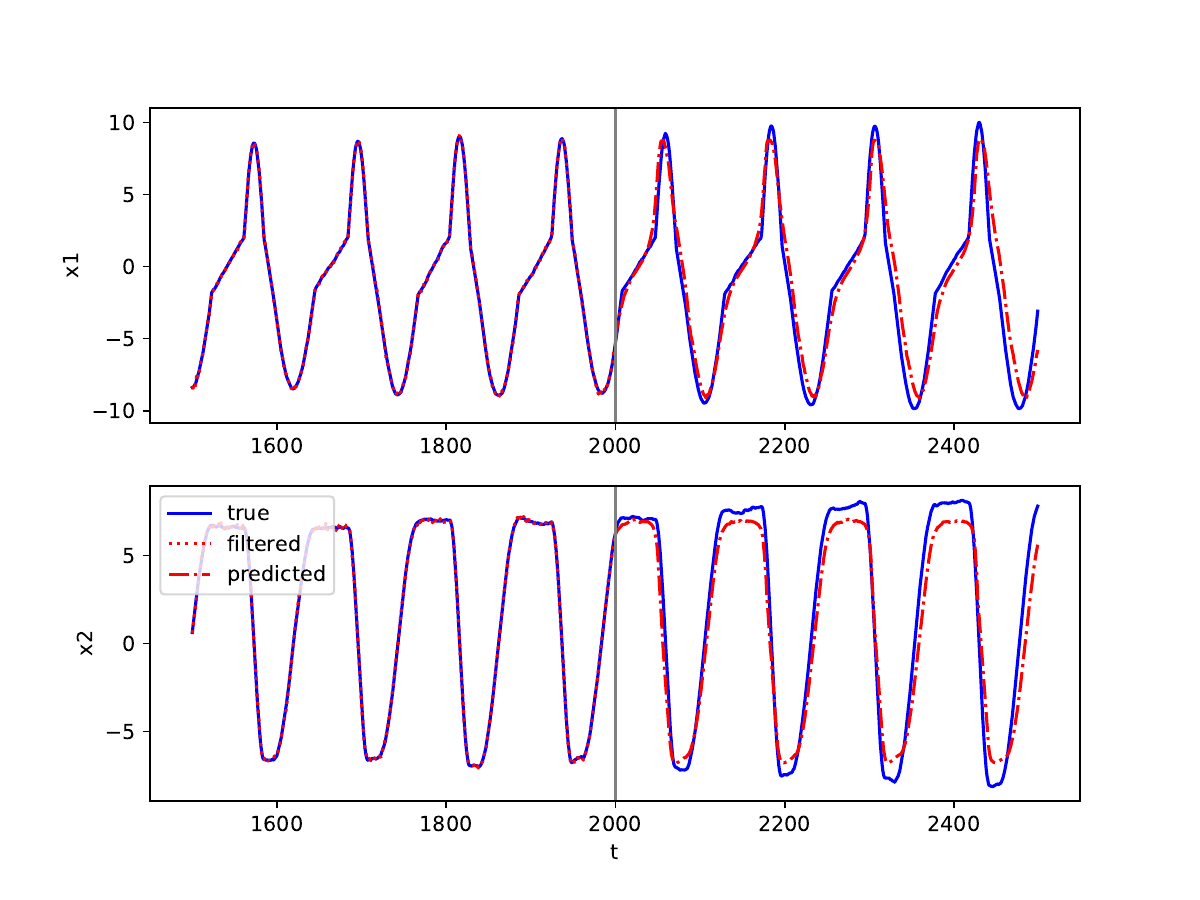} \label{fig:NASCAR_d}} 
    \subfloat[Filtering and prediction using SVMC \cite{zhao2022streaming}]{\includegraphics[width=.33\textwidth]{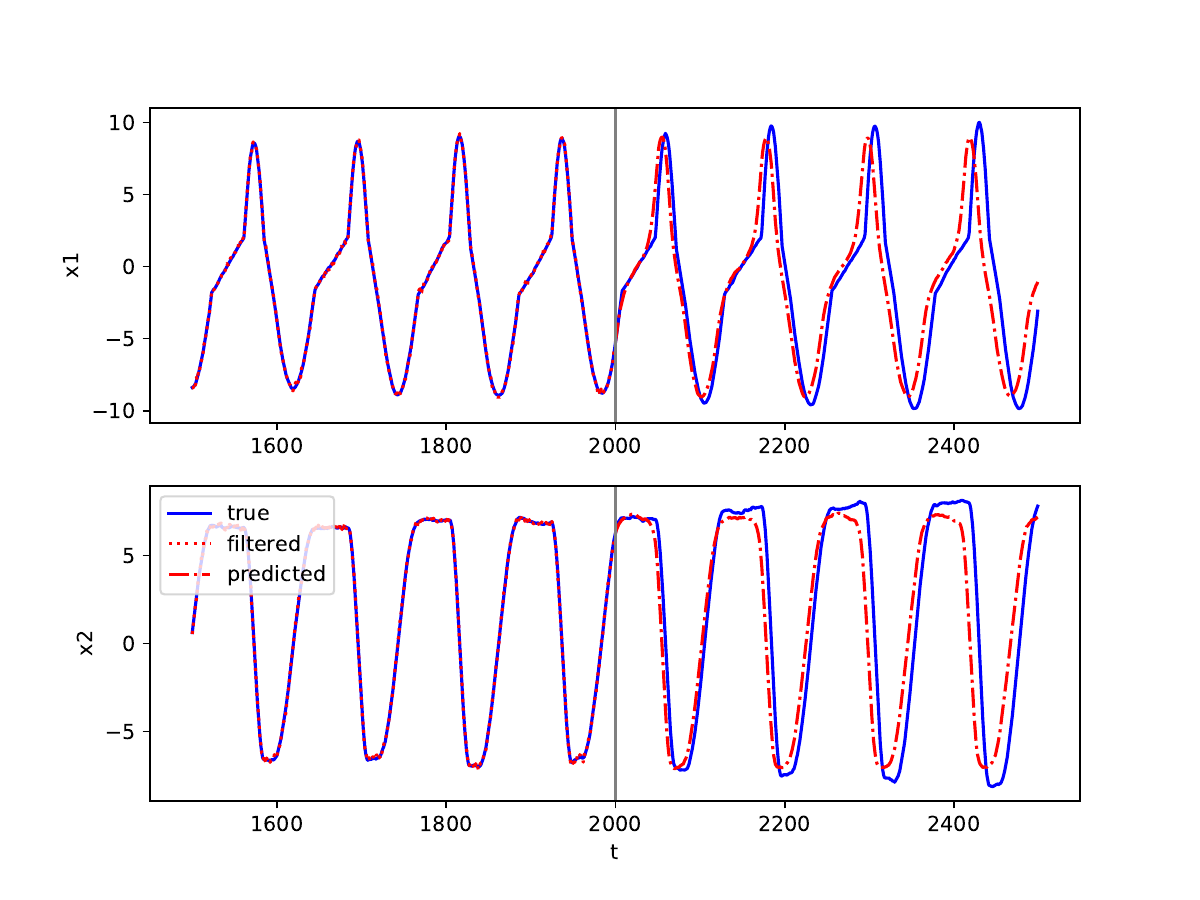}
    \label{fig:NASCAR_e}} 
    \subfloat[Filtering and prediction using VJF \cite{zhao2020variational}]{\includegraphics[width=.33\textwidth]{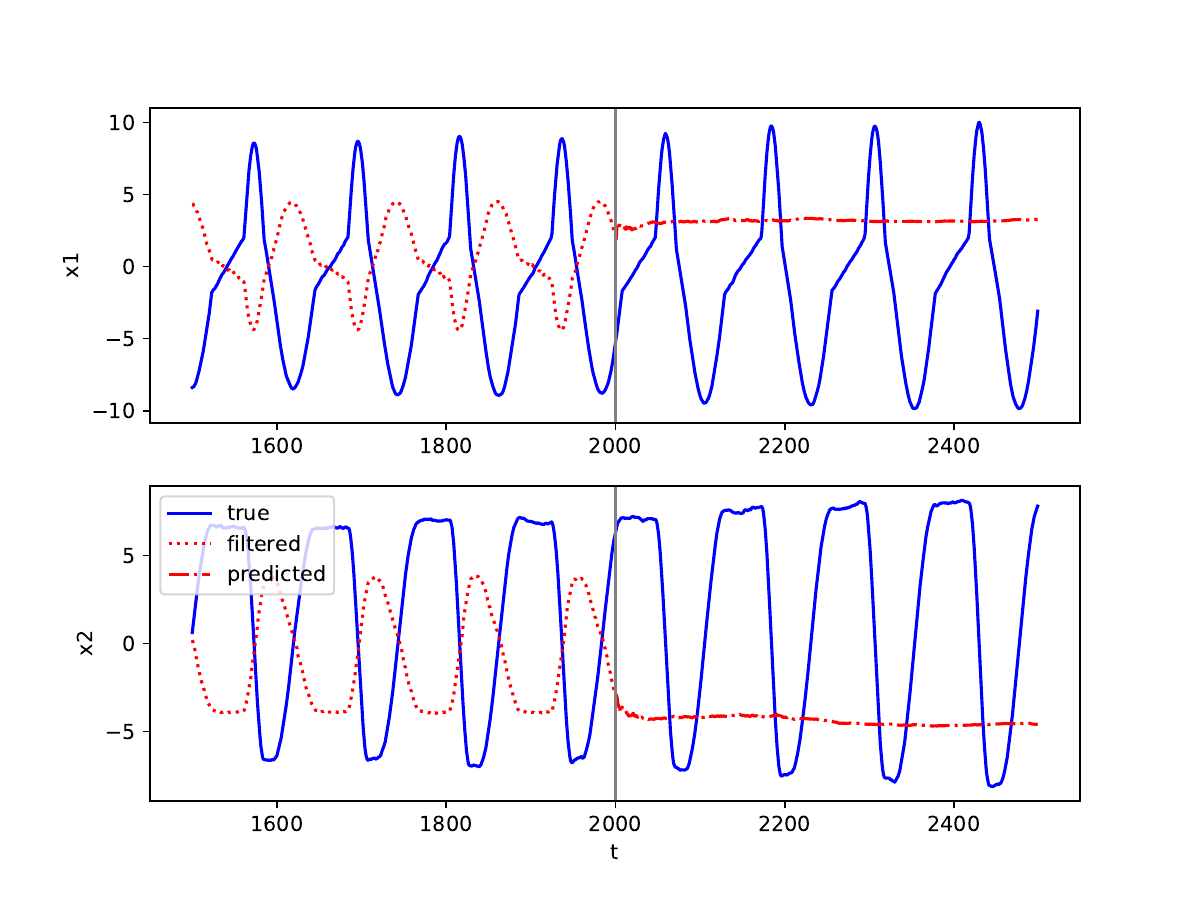} \label{fig:NASCAR_f}} 
    
    \caption{Online NASCAR$^\circledR$ dynamics learning results of the three online algorithms. The prediction RMSE values for OEnVI, SVMC, and VJF are 1.8780, 4.6682, and 10.8499, respectively.}
    \label{fig:NASCAR_OEnVI}
    \vspace{-.05in}
\end{figure*}
%%%%%----------------------------------------------------------------------

\vspace{-.1in}
\subsection{Online Learning and Inference Using OEnVI} 
\label{subsec:OEnVI_performance}  
We next evaluate the performance of OEnVI.  We compare OEnVI with two very recent competitive online learning algorithms, specifically SVMC \cite{zhao2022streaming}, which models the variational distribution of latent states using sample particles, and VJF \cite{zhao2020variational}, where the variational distribution of latent states is modeled by an inference network. More numerical results of OEnVI about learning kink dynamical function can be found in Supplement~\ref{supp_subsec:OEnVI_more}.

We evaluate the three online learning methods using NASCAR$^\circledR$ data, a dataset previously utilized in \cite{zhao2022streaming}, showcasing dynamics akin to recurrent switching linear dynamical systems \cite{linderman2017bayesian}. The latent state trajectory faithfully reproduces the layout of a NASCAR$^\circledR$ track, as depicted in Fig.~\ref{fig:NASCAR_a}.  We train these three methods with $2000$ observations and test them with $500$ observations, both generated from $\y_t = \bm{C} \x_t + \mathbf{e}_t$, where $\bm{C}$ is a 10-by-2 matrix generated randomly \cblue{(strictly following the settings of SVMC \cite{zhao2022streaming}),} and $\mathbf{e}_t \!\sim\! \cN(\bm{0}, 0.1^2\bm{I}_{2\times2})$.  The SVMC and VJF algorithms are implemented using the code publicly provided online\footnote{\url{https://github.com/catniplab/svmc}}\footnote{\url{https://github.com/catniplab/vjf}}.  

Figs.~\ref{fig:NASCAR_a}--\ref{fig:NASCAR_c} depict the true states (in blue) and the latent states (in red) inferred by the three methods. The results clearly indicate that EnVI and SVMC swiftly captured the real state and maintained accuracy, while VJF faced challenges, primarily due to its difficulty in optimizing the parameters in the inference network. The detailed comparison with SVMC and VJF, illustrated in Figs.~\ref{fig:NASCAR_d}--\ref{fig:NASCAR_f}, highlights the superior accuracy of OEnVI in both inference and prediction of latent states. These empirical findings emphasize the efficacy of OEnVI, particularly in its use of EnKF to approximate the variational distribution, demonstrating its advancement over VJF using inference networks, and SVMC using PF methods.

\section{Conclusion}
\label{sec:conclusion}   
In this paper, we have introduced EnVI, a novel NMF algorithm tailored for GPSSMs, which integrates EnKF into a variational inference framework. Additionally, we have presented an extended online version, OEnVI. Both algorithms eliminate the necessity for heavy parameterization like inference networks and shape the variational distribution over latent states through the model-based EnKF. Leveraging the inherent differentiable nature along with the modern automatic differentiation tools, the proposed EnVI and OEnVI can enhance efficiency and algorithmic robustness while improving learning and inference performance compared to existing methods. Detailed analysis and fresh insights for the proposed algorithms are provided to enhance their interpretability. Empirical experiments conducted on diverse real and synthetic datasets, evaluating filtering, prediction, and dynamics learning performance, unequivocally support the effectiveness of the proposed methods. 
\cblue{Future research will explore  efficient learning and inference techniques for complex time-varying dynamical systems.}
% \cblue{Looking ahead, our focus is on further reducing the computational complexity of the GP transition function and enhancing its modeling capacity. This effort aims to render the GPSSM compatible with more high-dimensional datasets, aligning it with the capabilities of EnKF and enabling the application of EnVI in more practical, high-dimensional dynamic systems.}
%
%
% \cblue{In this paper,  we developed a novel variational inference framework for GPSSMs, leveraging variational inference and EnKF, which enables flexible and accurate Bayesian joint filtering. We demonstrated that the integration of EnKF can eliminate the need for inference networks and significantly reduces the number of variational parameters. Furthermore, by leveraging automatic differentiation tools, the EnKF integration enhances the robustness and efficiency of GPSSM training. In addition, we extend the proposed algorithm to an online setting and provide comprehensive analysis and insights into both proposed methods. Extensive testing on various real and simulated datasets demonstrates that the variational inference algorithm integrated with EnKF outperforms existing methods in terms of both performance and efficiency.}

% \section*{Acknowledgments}
% \noindent This should be a simple paragraph before the References to thank those individuals and institutions who have supported your work on this article.

\appendices
\section{Model Evidence Lower Bound Evaluations}
\label{appendix_elbo_proof} 
With the model joint distribution, Eq.~\eqref{eq:joint_dist_ips}, and the  variational distribution, Eq.~\eqref{eq:generic_vi_dist_enKF}, we can write down the ELBO according to the general definition given in Eq.~\eqref{eq:ELBO_general}, i.e.,
\begin{equation}
\label{eq:ELBO_NMF_enkf_appendix}
    \begin{aligned}
	\!\! & \mathcal{L} = \mathbb{E}_{q(\vu, \vx)} \left[\log \frac{p( \vu, \vx, \vy)} {q(\vu, \vx)}\right]\\
	\!\! & = \mathbb{E}_{q(\vu, \vx)}  \! \left[\log \frac{p(\mathbf{x}_{0})  p(\vu) \prod_{t=1}^{T}  p(\mathbf{x}_{t} \vert {\vu}, \x_{t-1}) p(\mathbf{y}_{t} \vert \mathbf{x}_{t})} {q(\vu) q(\x_0) \prod_{t= 1}^T q(\x_{t} \vert \vu, \x_{t-1}) }\right]  \\
		& = \underbrace{ \mathbb{E}_{q(\vu, \vx)}  \left[ \sum_{t=1}^T \log p(\y_{t} \vert \x_t) \right]}_{\text{term 1}}   \!
		-  \underbrace{\operatorname{KL}[q(\x_0) \| p(\x_0)]}_{\text{term 2}} \\ 
        &  \quad \!-\!\underbrace{\operatorname{KL}[q(\vu) \| p(\vu)]}_{\text{term 3}}  - \underbrace{\mathbb{E}_{q(\vu, \vx)}  \left[ \sum_{t=1}^T \log \frac{q(\x_{t} \vert \vu, \x_{t-1})}{p(\x_t \vert \vu, \x_{t-1})}\right]}_{\text{term 4}}.
    \end{aligned}
\end{equation}
where the KL divergence terms can be analytically computed in closed form \cite{theodoridis2020machine} due to the Gaussian nature of the prior and variational distributions. The evaluation of the first and fourth terms is typically intractable.
We examine the difference between term 1 and term 4 in Eq.~\eqref{eq:ELBO_NMF_enkf_appendix} and can have the following lemma. 
\cblue{
\begin{lemma}
\label{lemma:equivalence}
    Under the approximations that:
    \begin{itemize}
        \item[1)] $p(\x_{t-1} \vert \vu, \y_{1:t-1}) \approx p(\x_{t-1} \vert \vu, \y_{1:t})$,
        \item[2)] $q(\x_{t} \vert \vu, \x_{t-1}) \approx p(\x_{t} \vert \vu, \x_{t-1}, \y_{1:t})$, 
    \end{itemize}
    computing the difference between term 1 and term 4 in the ELBO (Eq.~\eqref{eq:ELBO_NMF_enkf_appendix}) yields the expected log-likelihood, i.e.,
    \begin{equation}
    \operatorname{term~1} - \operatorname{term~4} = \mathbb{E}_{q(\vu)} \left[ \log p(\y_{1:T} \vert \vu) \right]
    \label{eq:sum_term1_term4}
    \end{equation}
\end{lemma}
}

\begin{proof}
According to the ELBO given in Eq.~\eqref{eq:ELBO_NMF_enkf_appendix}, we have:
\begin{align}
    & \text{term 1} \!-\! \text{term 4}  \nonumber \\
    & \!=\! \mathbb{E}_{q(\vu, \vx)} \! \left[ \sum_{t=1}^T \log   \frac{p(\y_{t} \vert \x_t) p(\x_t \vert \vu, \x_{t-1})}{q(\x_t \vert \vu, \x_{t-1})} \right] \nonumber \\
    & \!=\! \mathbb{E}_{q(\vu, \vx)} \! \left[ \sum_{t=1}^T \log   \frac{p(\y_{t} \vert \x_t) p(\x_t \vert \vu, \x_{t-1}) \cblue{p(\x_{t-1} \vert \vu, \y_{1:t-1})} }{q(\x_t \vert \vu, \x_{t-1}) \cblue{p(\x_{t-1} \vert \vu, \y_{1:t-1})} } \right]  \nonumber \\
    & \! \approx \! \mathbb{E}_{q(\vu, \vx)} \! \left[ \sum_{t=1}^T \log   \frac{p(\y_{t} \vert \x_t) p(\x_t \vert \vu, \x_{t-1}) \cblue{p(\x_{t-1} \vert \vu, \y_{1:t-1})}}{  \color{red} 
    \underbracket{\color{black} p(\x_t \vert \vu, \x_{t-1}, \y_{1:t})}_{\cblue{\text{assumption 2)}}}
    \underbracket{\color{black} p(\x_{t-1} \vert \vu, \y_{1:t})}_{\cblue{\text{assumption 1)}}}} \right]  \nonumber \\
    & \!=\! \mathbb{E}_{q(\vu, \vx)} \left[ \sum_{t=1}^T \log   \frac{p(\y_{t}, \x_t, \x_{t-1} \vert \vu, \y_{1:t-1})}{p(\x_t, \x_{t-1} \vert  \vu, \y_{1:t-1}, \y_t) } \right] \nonumber \\
    & \!=\! \mathbb{E}_{q(\vu)} \left[ \sum_{t=1}^T \log  p(\y_t \vert \vu, \y_{1:t-1}) \right], \label{subeq:expected_ll} 
    % & \!=\! \mathbb{E}_{q(\vu)} \left[ \log  p(\y_{1:T} \vert \vu) \right],
\end{align}
where the last line of Eq.~\eqref{subeq:expected_ll} is derived straightforwardly by applying Bayes' theorem.
\end{proof} \vspace{-.07in}
According to Lemma \ref{lemma:equivalence}, and the ELBO given in Eq.~\eqref{eq:ELBO_NMF_enkf_appendix}, we immediately get the following approximated ELBO:
\begin{align}
    \! \! \mathcal{L} \!\approx\! \mathbb{E}_{q(\vu)} \! \left[ \sum_{t=1}^T \log  p(\y_t \vert \vu, \y_{1:t-1}) \right] & \! - \operatorname{KL}[q(\x_0) \| p(\x_0)]  \label{eq:envi_elbo_appendix}\\ 
    & \! - \operatorname{KL}[q(\vu) \| p(\vu)],
\end{align}
where the log-likelihood, $\log p(\y_t \vert \vu, \y_{1:t-1})$ in  Eq.~\eqref{eq:envi_elbo_appendix} can be analytically evaluated using EnKF, see Eq.~\eqref{eq:evaluation_log_likelihood}, due to the Gaussian prediction distribution, see Eq.~\eqref{eq:predictive_distri_enkf_ip}, and the linear emission model.

\section{Proof of Proposition \ref{prop:2}}
\label{appendix_prop2_proof}
\begin{proof}
With the filtering distribution $p(\x_t \vert \vu, \y_{1:t})$,  we have the log-likelihood term 
% \begin{subequations}
%     \begin{align}
%          - & \operatorname{KL}\left[ p(\x_t \vert \vu, \y_{1:t}) \| p(\x_t \vert \vu, \y_{1:t-1}) \right] \\
%         & = - \mathbb{E}_{p(\x_t \vert \vu, \y_{1:t})} \left[ \log \frac{p(\x_t \vert \vu, \y_{1:t})}{p(\x_t \vert \vu,  \y_{1:t-1})}\right]\\
%         & = \mathbb{E}_{p(\x_t \vert \vu, \y_{1:t})} \left[ \log \frac{p(\x_t \vert \vu,  \y_{1:t-1})}{p(\x_t \vert  \vu, \y_{1:t})}\right]\\
%         & = \mathbb{E}_{p(\x_t \vert \vu, \y_{1:t})} \left[ \log \frac{p(\x_t \vert \vu, \y_{1:t-1}) p(\y_t \vert \vu, \y_{1:t-1})}{ p(\x_t \vert \vu, \y_{1:t-1}) p(\y_t \vert  \x_{t})}\right]\\
%         & = \mathbb{E}_{p(\x_t \vert \vu, \y_{1:t})} \left[ \log \frac{ p(\y_t \vert \vu,  \y_{1:t-1})}{ p(\y_t \vert \x_{t})}\right]\\
%         & =  \log   p(\y_t \vert \vu, \y_{1:t-1}) - \mathbb{E}_{p(\x_t \vert \vu, \y_{1:t})} \left[ \log  p(\y_t \vert \x_{t})\right],
%     \end{align}
% \end{subequations}
\begin{subequations}
    \begin{align}
    & ~\log  p(\y_t \vert \vu, \y_{1:t-1}) \\
    & = \mathbb{E}_{p(\x_t \vert \vu, \y_{1:t})} \left[  \log p(\y_t \vert \vu, \y_{1:t-1})  \right] \\
    & = \mathbb{E}_{p(\x_t \vert \vu, \y_{1:t})} \left[  \log \frac{ p(\x_t \vert \vu, \y_{1:t-1}) p(\y_t \vert \x_t) }{ p(\x_t \vert \vu, \y_{1:t-1}, \y_t) }\right] \label{subeq:log-like-term}\\
    & = - \mathbb{E}_{p(\x_t \vert \vu, \y_{1:t})} \left[  \log \frac{ p(\x_t \vert \vu, \y_{1:t}) }{ p(\x_t \vert \vu, \y_{1:t-1}) p(\y_t \vert \x_t) }\right] \\
    & = - \operatorname{KL}\left[ p(\x_t \vert \vu, \y_{1:t}) \| p(\x_t \vert \vu, \y_{1:t-1}) \right] \nonumber \\
    & \quad ~ + \mathbb{E}_{p(\x_t \vert \vu, \y_{1:t})} \left[ \log  p(\y_t \vert \x_{t})\right], 
    \end{align}
\end{subequations}
which completes the proof. Here Eq.~\eqref{subeq:log-like-term} is obtained straightforwardly by applying Bayes' theorem. This result sheds light on the interplay between data reconstruction and the alignment of filtering and prediction distributions in the EnVI and OEnVI algorithms.
\end{proof}

\section{Mean-Field and Non-Mean-Field Approximations} \label{appendix:MF_NMF_definition}
\begin{definition}
\label{def:MF_NMF}
% \vspace{-.05in}
If the variational distribution, $q(\vf, \vu, \vx)$, is factorized such that the transition function values and the latent states are independent, i.e.,
\begin{equation}
q(\vf, \vu) q(\x_0) \prod_{t=1}^T q(\x_{t} \vert \f_t) = q(\vf, \vu) q(\vx),
\label{eq:mf_assumption}
\end{equation}
the factorization is known as a mean-field approximation in the GPSSM literature. Conversely, if the variational distribution, $q(\vx, \vf, \vu)$,  explicitly builds the dependence between the latent states and the transition function values, as shown in Eq.~(\ref{eq:generic_vi_dist}), it is a non-mean-field approximation.
\end{definition}

\bibliographystyle{IEEEtran}  
\bibliography{ref_vEnKF}

% % \bf{If you include a photo:}\vspace{-33pt}
% \end{IEEEbiography}

% % \bf{If you will not include a photo:}\vspace{-33pt}

\vfill

\newpage
\renewcommand{\appendixname}{Supplement}
\appendices
% \appendix
\onecolumn 

\section{Variational Bayes} \label{appx:miscellanies}
\subsection{Evidence Lower Bound (ELBO)} \label{appx_subsec:ELBO}
In Bayesian statistics, the model marginal likelihood $p(\vy \vert \btheta)$ is a fundamental quantity for model selection and comparison \cite{courts2021gaussian}. By maximizing the logarithm of $p(\vy \vert \btheta)$ w.r.t. the model parameters $\btheta$, the goodness of data fitting and the model complexity are automatically balanced, in accordance with Occam's razor principle \cite{cheng2022rethinking}.
However,  $p(\vy \vert \btheta)$ is obtained by integrating out all the latent variables $\{\vx, \vf\}$ in the joint distribution, see Eq.~(\ref{eq:joint_dist}), which is analytically intractable.  Thus, the posterior distribution of the latent variables, $p(\vx, \vf \vert \vy) = \frac{p(\vy, \vx, \vf)}{p(\vy \vert \btheta)},$ cannot be expressed in a closed-form expression, either.  This intractability issue has been addressed in variational Bayesian methods by adopting a variational distribution \cite{theodoridis2020machine}, $q(\vx, \vf)$, to approximate the intractable $p(\vx, \vf \vert \vy)$. With the newly introduced variational distribution $q(\vx, \vf)$, we have
\begin{equation}
    \begin{aligned}
         \log p(\vy \vert \btheta) = \log \frac{p(\vy, \vx, \vf)}{p(\vx, \vf \vert \vy)}  & =\int \int q(\vx, \vf) \log \frac{p(\vy, \vx, \vf) q(\vx, \vf)}{p(\vx, \vf \vert \vy) q(\vx, \vf)} \mathrm{d} \vx \mathrm{d} \vf \\
         & = \underbrace{\mathbb{E}_{q(\vx, \vf)} \! \left[ \log \frac{p(\vy, \vx, \vf)}{q(\vx, \vf)}\right]}_{\text{ELBO: } \mathcal{L}} + \underbrace{\mathbb{E}_{q(\vx, \vf)} \! \left[ \log \frac{q(\vx, \vf)}{p(\vx, \vf \vert \vy)}\right]}_{\text{KL divergence}}
    \end{aligned}
\end{equation}

\section{More Experiment Results} \label{supp_sec:more_experimental_results}
\subsection{OEnVI: Online Learning Results on Linear Gaussian SSMs} \label{supp_subsec:OEnVI_LGSSMs}

\begin{itemize}
    \item The definition of RMSE:
    \begin{equation}
        \text{RMSE} = \sqrt{ \frac{1}{T} \sum_{t=1}^T \sum_{d=1}^{d_x}(\hat{\x}_t^{(d)} - \x_t^{(d)})^2 }
    \end{equation}
    where $\hat{\x}_t$ represents the estimation of $\x_t$.
    \item The baseline in Table \ref{supp_tab:LGPSSM} and Table~\ref{supp_tab:LGPSSM_2} is the RMSE between the noisy observations $\y_{1:T}$ and the latent states $\x_{1:T}$.
    
    \item \cblue{
    Table \ref{supp_tab:LGPSSM}, and  Figs.~\ref{fig:supp_OEnVI_1} and \ref{fig:supp_OEnVI_2} present the results obtained using our proposed methods, EnVI and OEnVI, applied to learning the linear Gaussian SSM with the identity coefficient matrix $\bm{C} = \bm{I}_{4\times 4}$.
    }
    
    \item \cblue{
    The results in Table~\ref{supp_tab:LGPSSM_2} and Fig.~\ref{fig_maintext:LGSSMs_EnVIs_random_C} show that our proposed methods, EnVI and OEnVI, can still accurately filter the true latent states. However, it is noteworthy that the matrix $ \bm{C} \in \mathbb{R}^{d_y \times d_x} $ should be ``full-column rank'' if the main task is to perform filtering; otherwise, there is a high probability that some latent dimensions cannot be inferred accurately. Intuitively speaking, if we aim to recover latent state $\x_t$, it is only possible when the corresponding observation $\y_t$ contains ``sufficient information'' about $\x_t$.  
    For the task of predicting the sequence $ \y_t $ for $ t = T+1, T+2, \ldots $, it is typically not necessary to focus as much on the physical meaning and accuracy of the latent states.
}
\end{itemize}

\begin{table*}[!ht]
\centering
\caption{State inference performance (RMSE) of EnVI and OEnVI (with emission coefficient matrix $\bm{C} = \bm{I}_{4 \times 4}$)}
\setlength{\tabcolsep}{2.32mm}{
    \centering
    \begin{tabular}{|l|c|c|c|c|c|c|c|c|c|c|}
            \hline  Time Slot & \textbf{0 -- 120} & \textbf{120 -- 240} & \textbf{240 -- 360} & \textbf{360--480} & \textbf{480--600} & \textbf{600--720} & \textbf{720--840} & \textbf{840--960} & \textbf{900-1000} & \textbf{0--1000} \\
            \hline \textbf{KF} & 0.5252 & 0.5149 & 0.5228 & 0.5658 & 0.5246 & 0.4907 & 0.4983 & 0.5202 & 0.4846 & 0.5199 \\
            \hline \textbf{EnVI} & 0.6841 & / & / & / & / & / & / & / & / & 0.7182 \\
            \hline \textbf{OEnVI} & 0.7784 & 0.7130 & 0.6512 & 0.6487 & 0.6786 & 0.6515 & 0.5958 & 0.6713 & 0.6418 & 0.6739 \\
            \hline \textbf{Baseline} & 0.9872 & 0.9811 & 0.9510 & 1.0077 & 1.0215 & 0.9967 & 1.0077 & 1.0314 & 1.0222 & 0.9974 \\
            \hline
       \end{tabular}}
       \label{supp_tab:LGPSSM}
\end{table*}

\begin{table*}[!ht]
\centering
\caption{State inference performance (RMSE) of EnVI and OEnVI (with emission coefficient matrix $\bm{C} \in \mathbb{R}^{4 \times 4}$ generated randomly )}
\setlength{\tabcolsep}{2.32mm}{
    \centering
    \begin{tabular}{|l|c|c|c|c|c|c|c|c|c|c|}
            \hline  Time Slot & \textbf{0 -- 120} & \textbf{120 -- 240} & \textbf{240 -- 360} & \textbf{360--480} & \textbf{480--600} & \textbf{600--720} & \textbf{720--840} & \textbf{840--960} & \textbf{900-1000} & \textbf{0--1000} \\
            \hline \textbf{KF} & 0.7799 & 1.0302 & 1.1086 & 1.3682 & 1.4355 & 1.1277 & 0.7517 & 1.1239 & 1.0911 & 1.1062 \\
            \hline \textbf{EnVI} & 1.8413  & / & / & / & / & / & / & / & / &  2.3356\\
            \hline \textbf{OEnVI} & 1.4508 & 1.5902 & 1.3763 & 3.0673 & 1.2875 & 2.1743 & 1.2470 & 3.6032 & 2.8966 & 2.1546 \\
            \hline \textbf{Baseline} & 24.10  & 100.49 & 211.82 & 272.88 & 303.85 & 422.38 & 547.60 & 741.17 & 847.08 & 428.24 \\
            \hline
       \end{tabular}}
       \label{supp_tab:LGPSSM_2}
\end{table*}

\begin{figure*}[ht!]
    \centering
    \subfloat{\includegraphics[width=.96\linewidth]{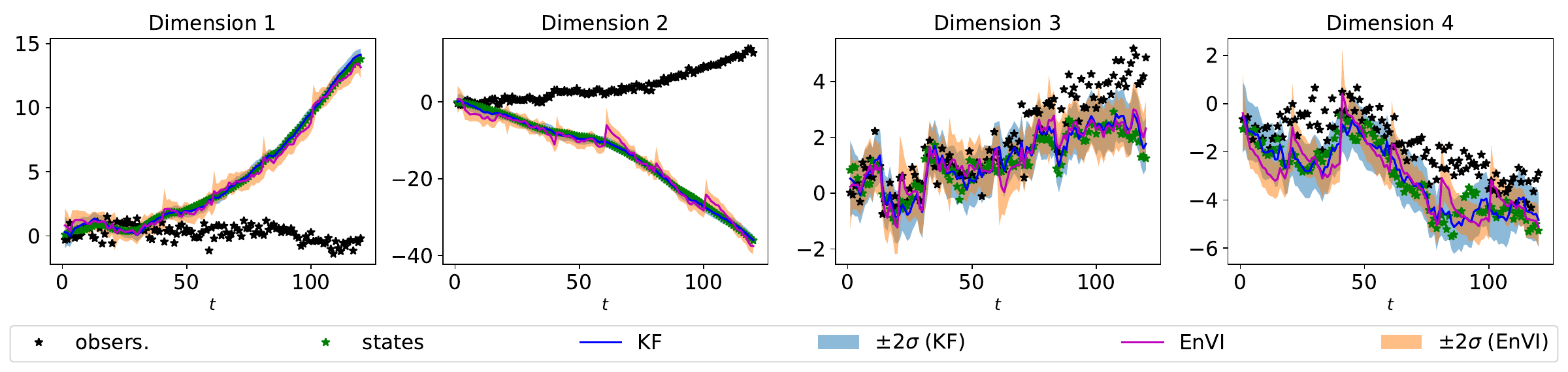}} \\ \vspace{-.07in}    
    \subfloat{\includegraphics[width=.96\linewidth]{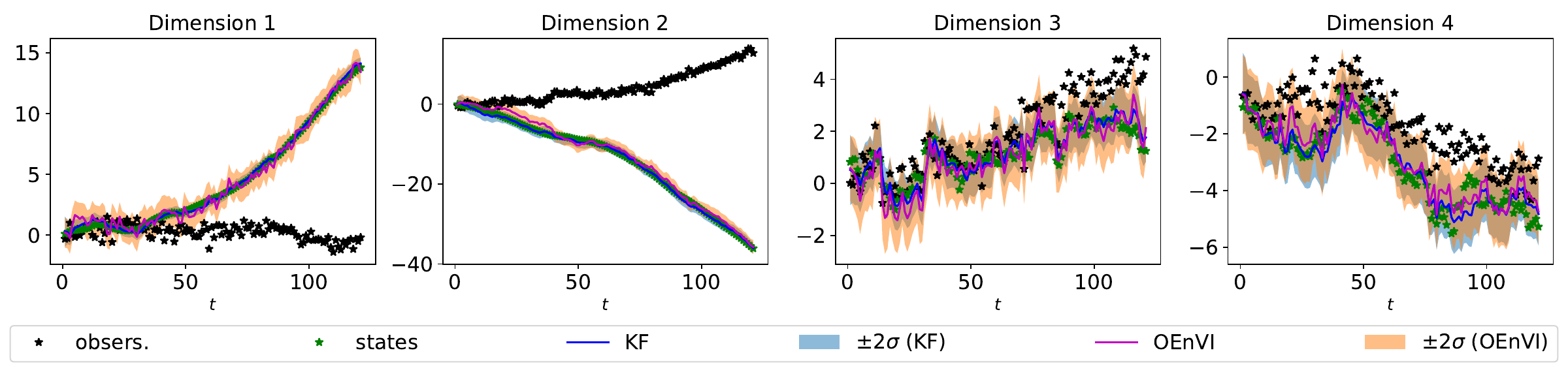}}  \vspace{-.07in}    
    \caption{\cblue{EnVI (\textbf{top}) \& OEnVI (\textbf{bottom}) on state inference in linear Gaussian SSM, with the emission coefficient matrix $\bm{C} \in \mathbb{R}^{4 \times 4}$ generated randomly. The RMSE of the latent state estimates for KF, EnVI, and OEnVI are 0.7799, 1.8413, and 1.4508, respectively; the RMSE between the observations and the latent states is 24.10.} \vspace{-.12in}}
    \label{fig_maintext:LGSSMs_EnVIs_random_C}
\end{figure*}

\begin{figure*}[!ht]
    \centering
    \subfloat[Online learning result from $t=0$ to $t=120$]{\includegraphics[width=.9\linewidth]{figs/carTrack/carTrack_0_to_120_EnVI.pdf}} 
    
    \subfloat[Online learning result from $t=120$ to $t=240$]{\includegraphics[width=.9\linewidth]{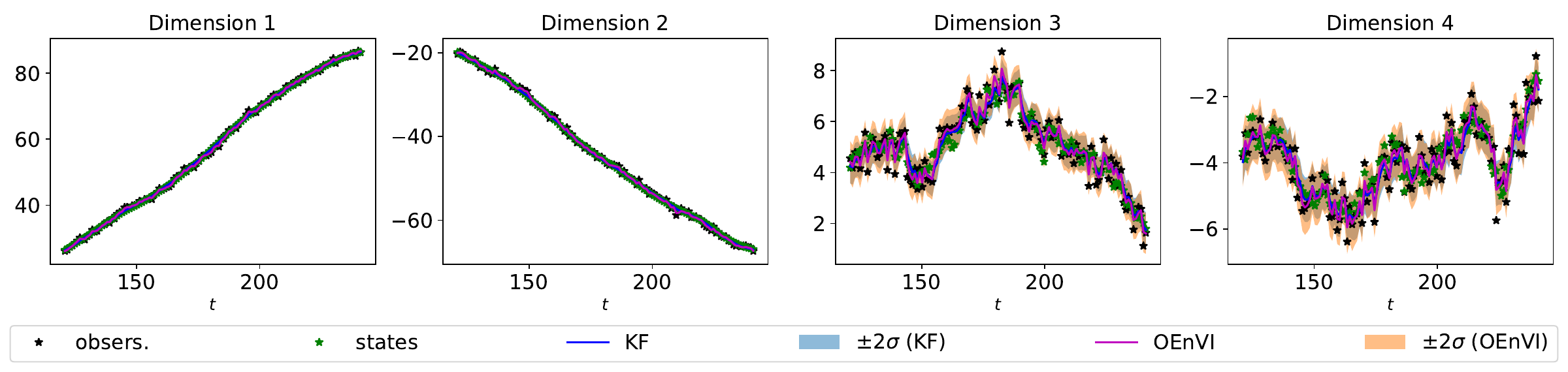}} 
    
    \subfloat[Online learning result from $t=240$ to $t=360$]{\includegraphics[width=.9\linewidth]{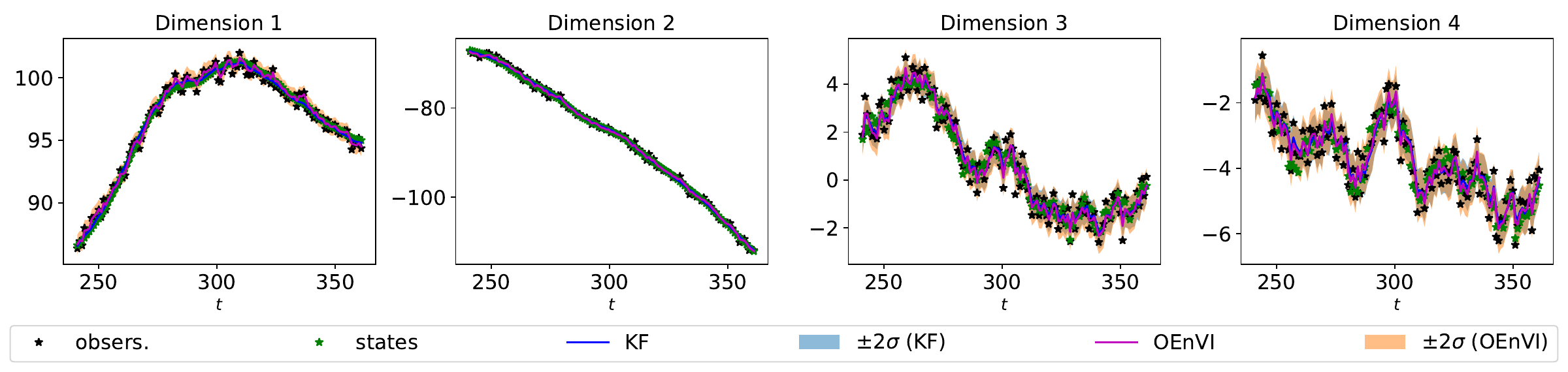}} 
    
    \subfloat[Online learning result from $t=360$ to $t=480$]{\includegraphics[width=.9\linewidth]{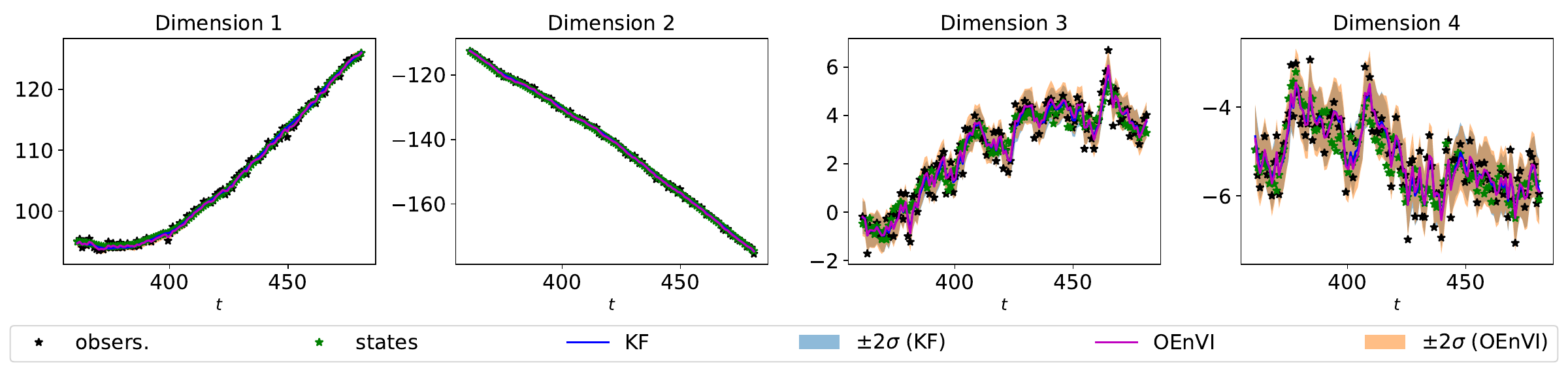}}
    
    \subfloat[Online learning result from $t=480$ to $t=600$]{\includegraphics[width=.9\linewidth]{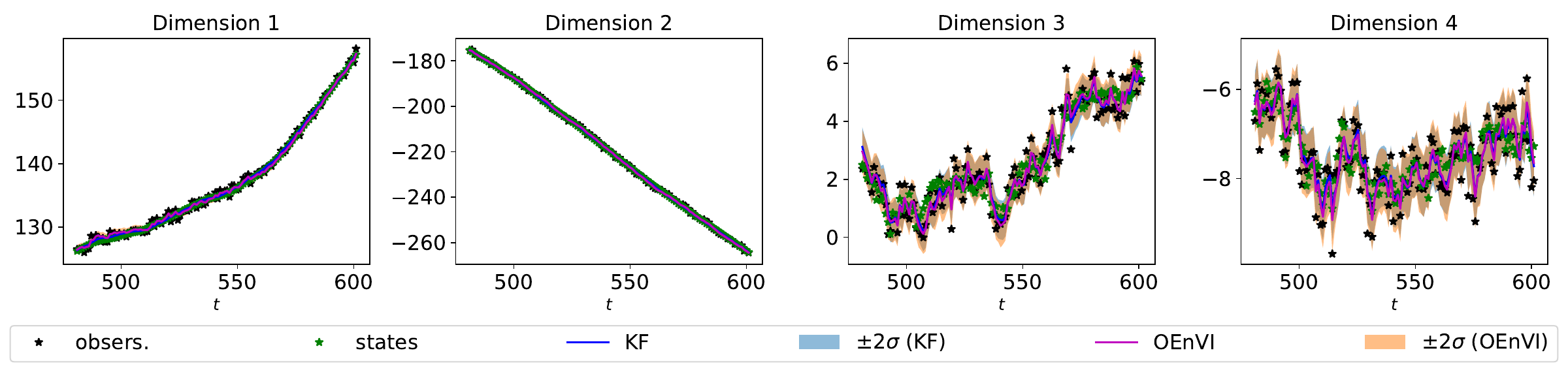}} 

    \caption{OEnVI on learning and inference in linear Gaussian SSMs with $\bm{C} = \bm{I}_{4\times 4}$}
    \label{fig:supp_OEnVI_1}
\end{figure*}
\begin{figure*}[t!]
    \centering
    \subfloat[Online learning result from $t=600$ to $t=720$]{\includegraphics[width=.9\linewidth]{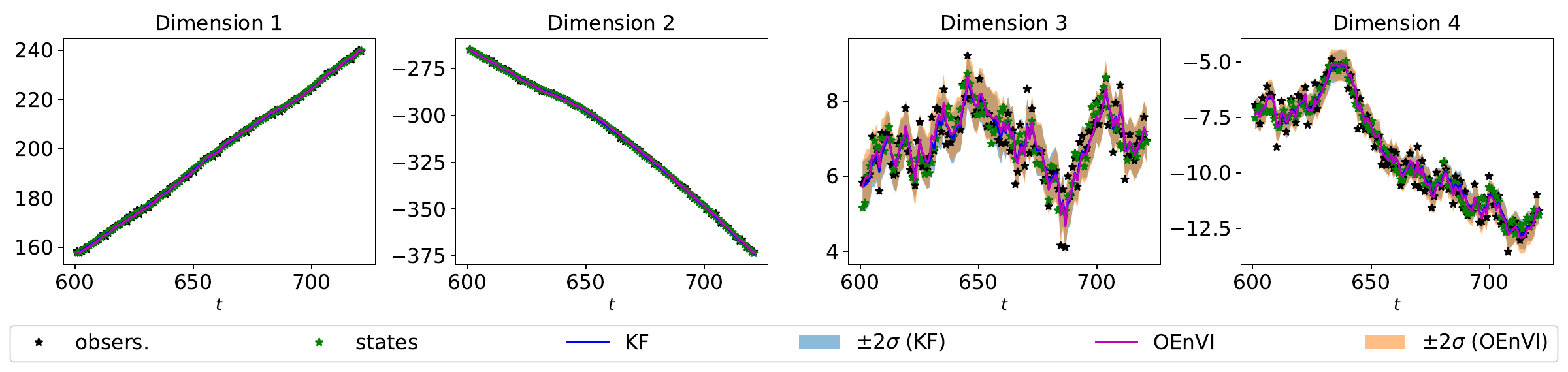}} 
    
    \subfloat[Online learning result from $t=720$ to $t=840$]{\includegraphics[width=.9\linewidth]{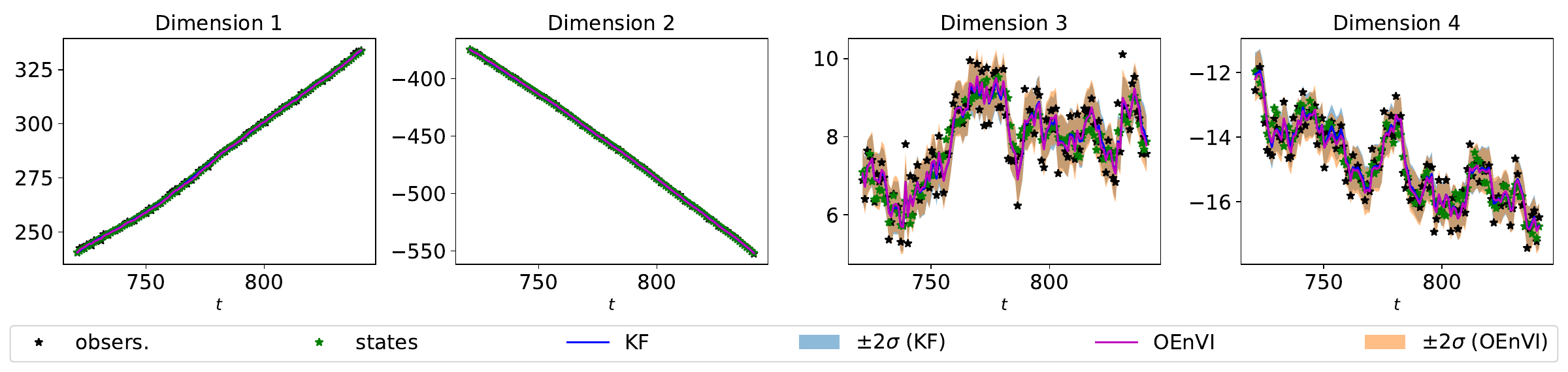}} 
    
    \subfloat[Online learning result from $t=840$ to $t=960$]{\includegraphics[width=.9\linewidth]{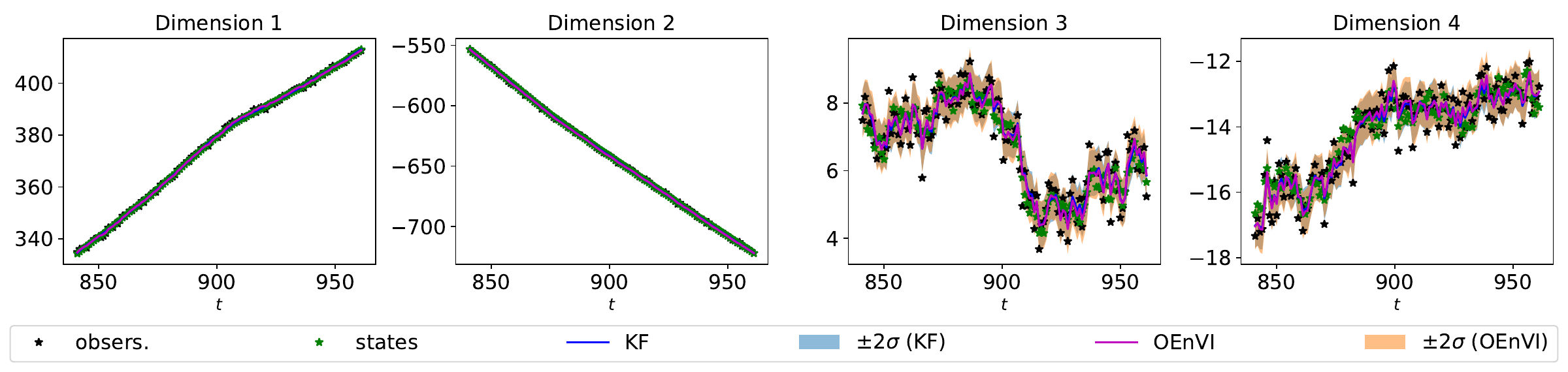}} 
    
    \subfloat[Online learning result from $t=900$ to $t=1000$]{\includegraphics[width=.9\linewidth]{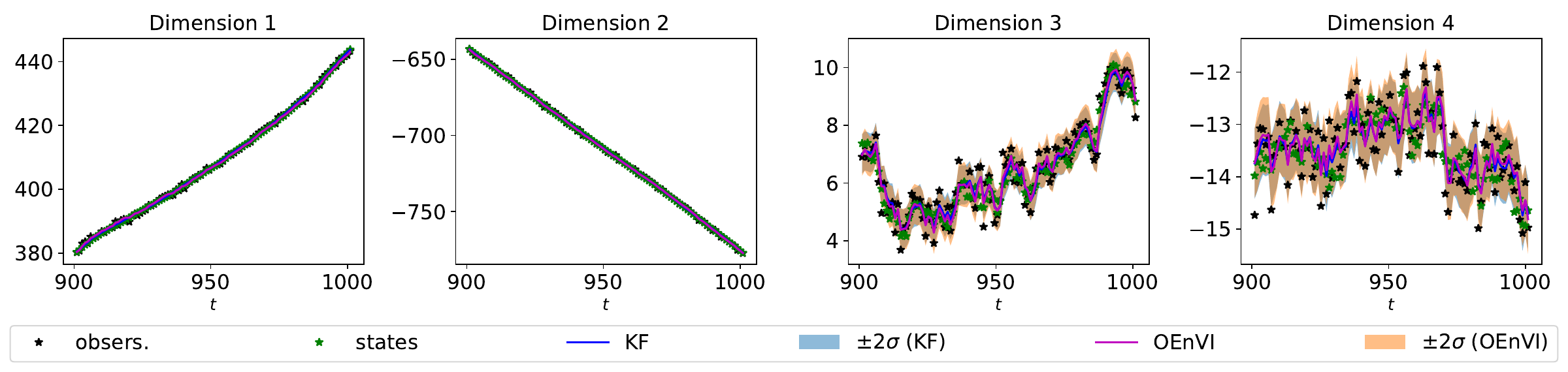}} 
    
    \subfloat[Online learning result from $t=0$ to $t=1000$]{\includegraphics[width=.9\linewidth]{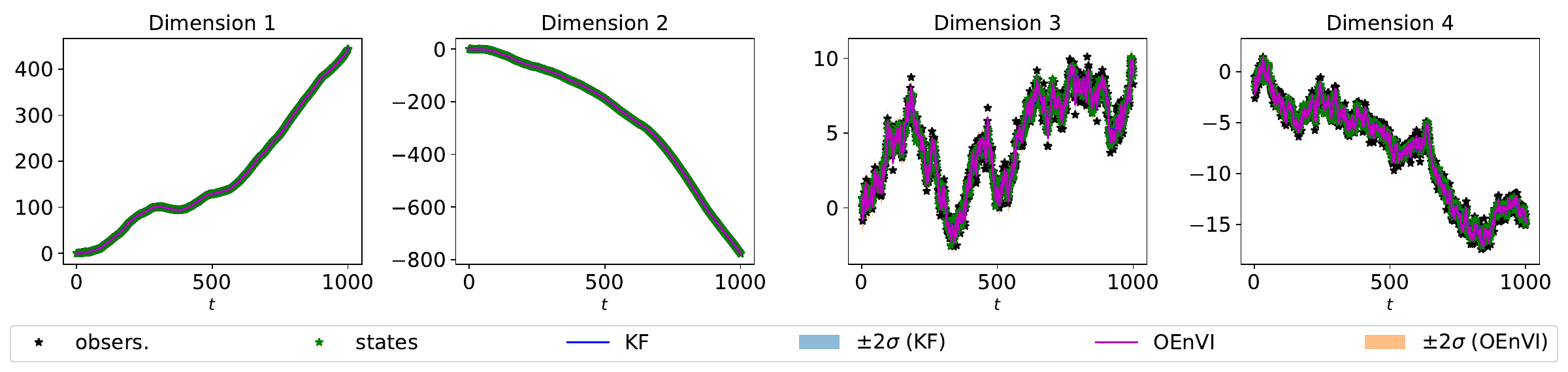}} 
    \caption{OEnVI on learning and inference in linear Gaussian SSMs with $\bm{C} = \bm{I}_{4\times 4}$}
    \label{fig:supp_OEnVI_2}
\end{figure*}

\clearpage
\subsection{Learning System Dynamics on Kink Function} \label{supp_subsec:EnVI_kink}
\begin{itemize}
    \item The kink function is depicted in Fig.~\ref{fig:kink_function}
    \item The MSE and Log-likelihood in Table \ref{tab:synthetic_dataset_MSE_LL} are evaluated as follows:
    \begin{align}
        & \text{MSE} = \frac{1}{T} \sum_{t=1}^T \sum_{d=1}^{d_x}(\hat{\f}_t^{(d)} - \f_t^{(d)})^2 \\
        &  \text{Log-likelihood} = \frac{1}{T} \sum_{t=1}^T \log \cN(\f_t \mid \bm{\xi}_t, \bm{\Xi}_t)
    \end{align}
\end{itemize}

\begin{figure}[!ht]
    \centering
    \includegraphics[width =0.7\textwidth]{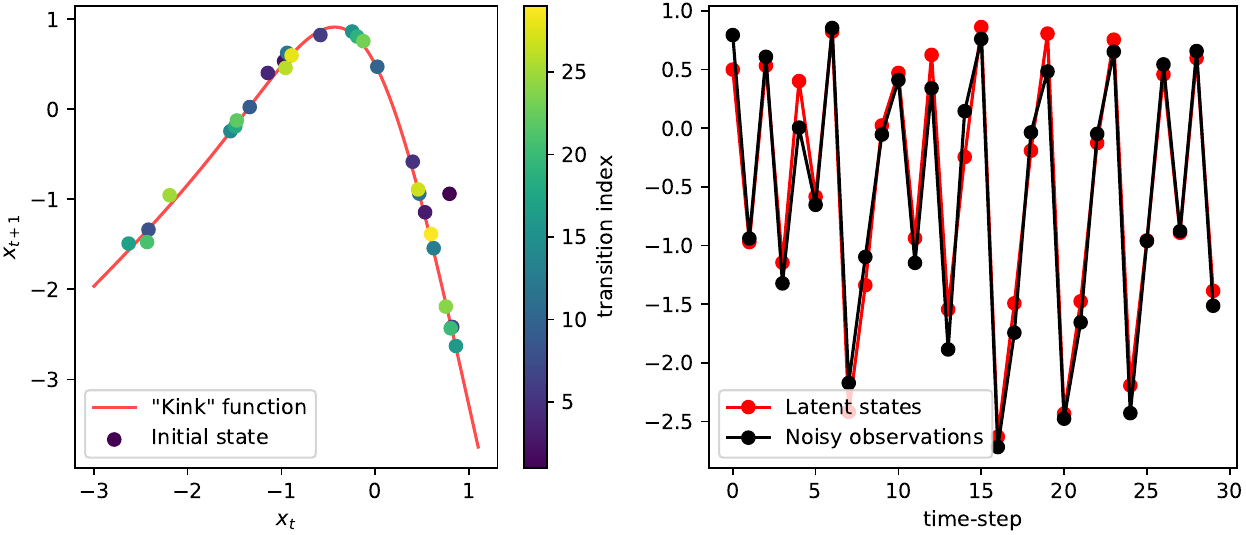}
    \caption{The "kink" dynamical function is used to generate $50$ latent states and corresponding observations, with $\sigma_{\mathrm{Q}}^2 = 0.01$ and $\sigma_{\mathrm{R}}^2 = 0.1$.}
    \label{fig:kink_function}
\end{figure}

\clearpage
\subsection{Time Series Data Forecasting using EnVI}
In addition to the overall prediction performance outlined in Table \ref{tab:systemidentifcation}, we provide a specific prediction of EnVI below.
\begin{figure*}[ht!]
    \centering
    \subfloat[Actuator]{\includegraphics[width=.5\linewidth]{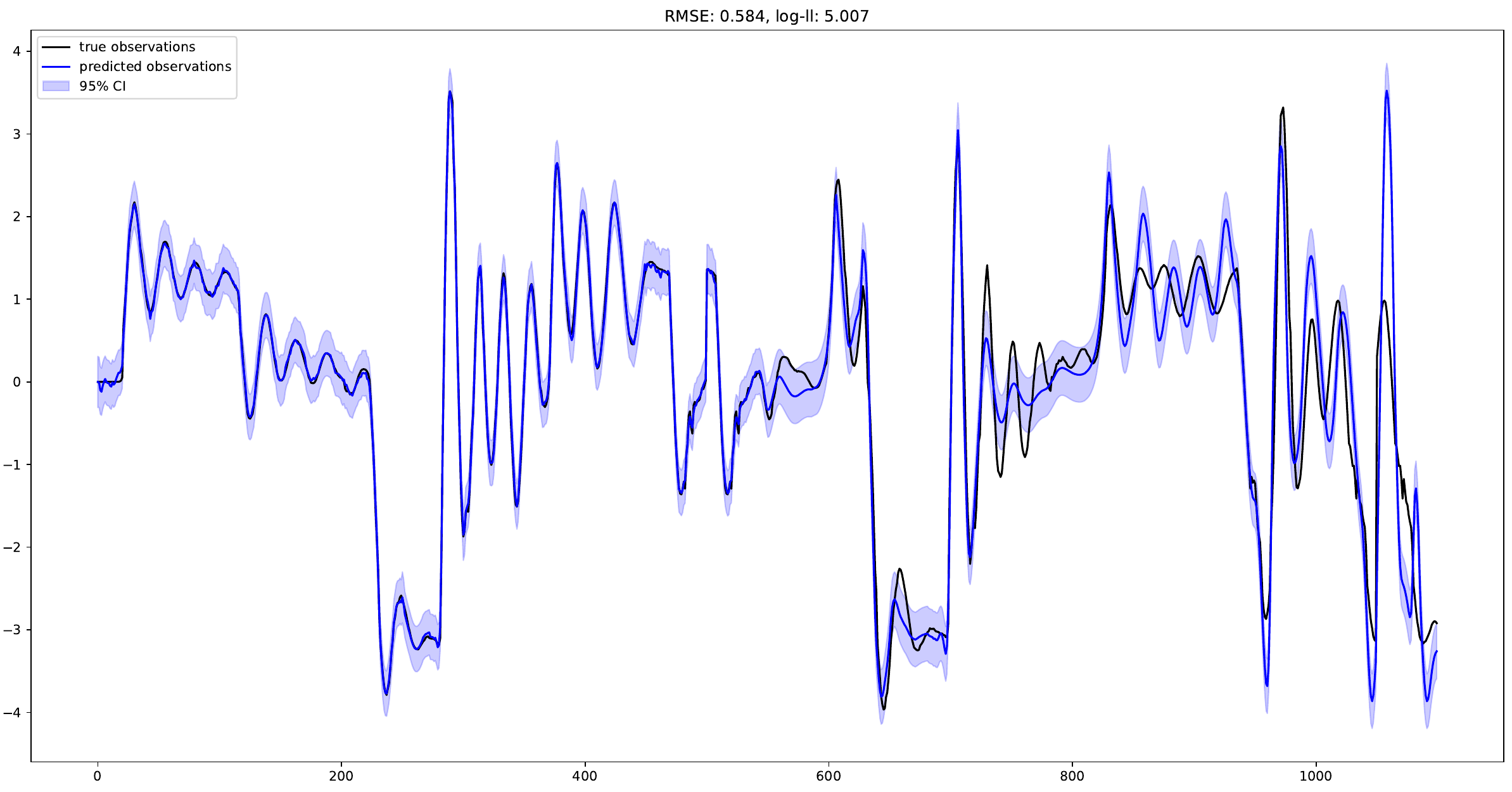}} 
    \subfloat[Ball Beam]{\includegraphics[width=.5\linewidth]{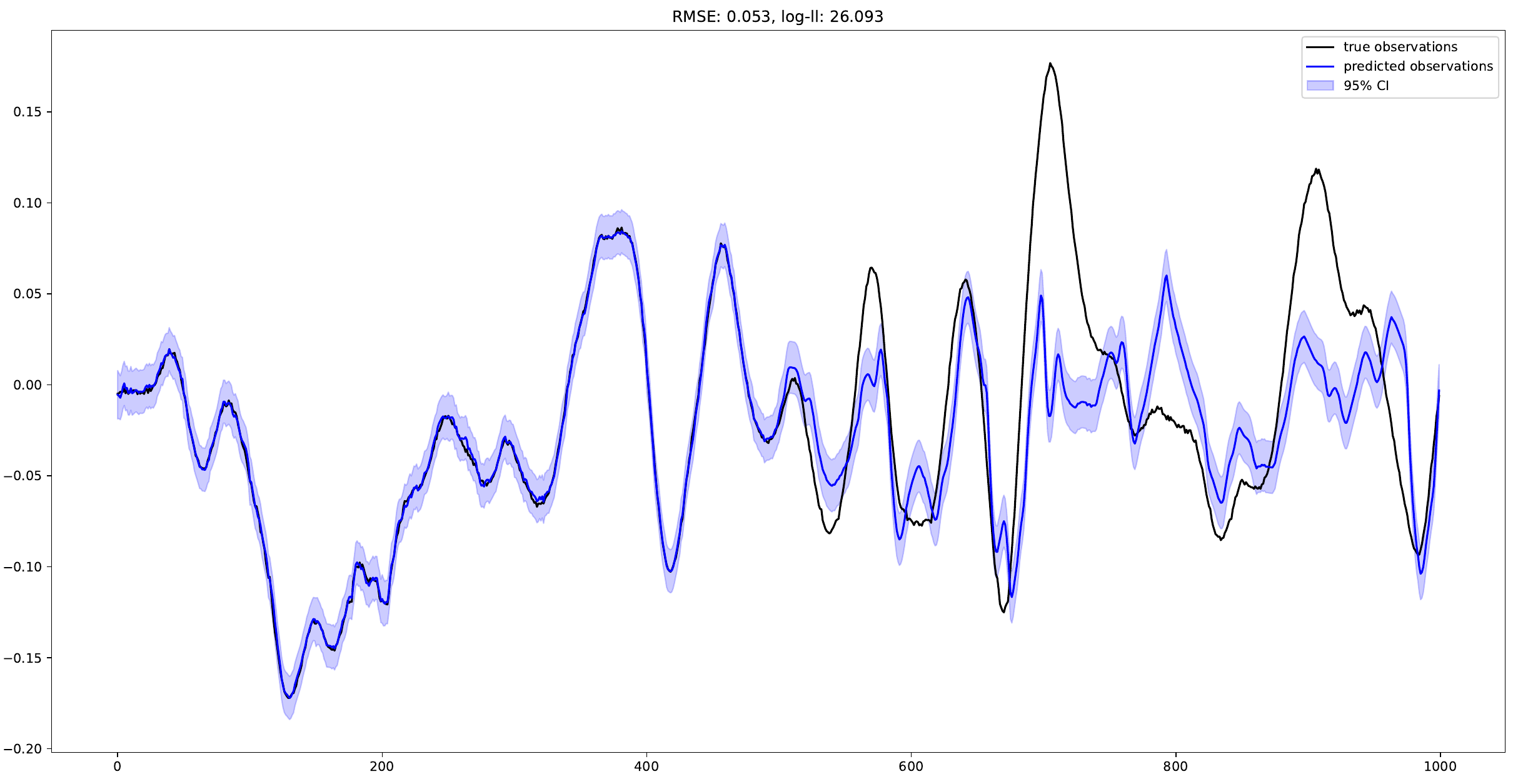}} 

    \subfloat[Drive]{\includegraphics[width=.5\linewidth]{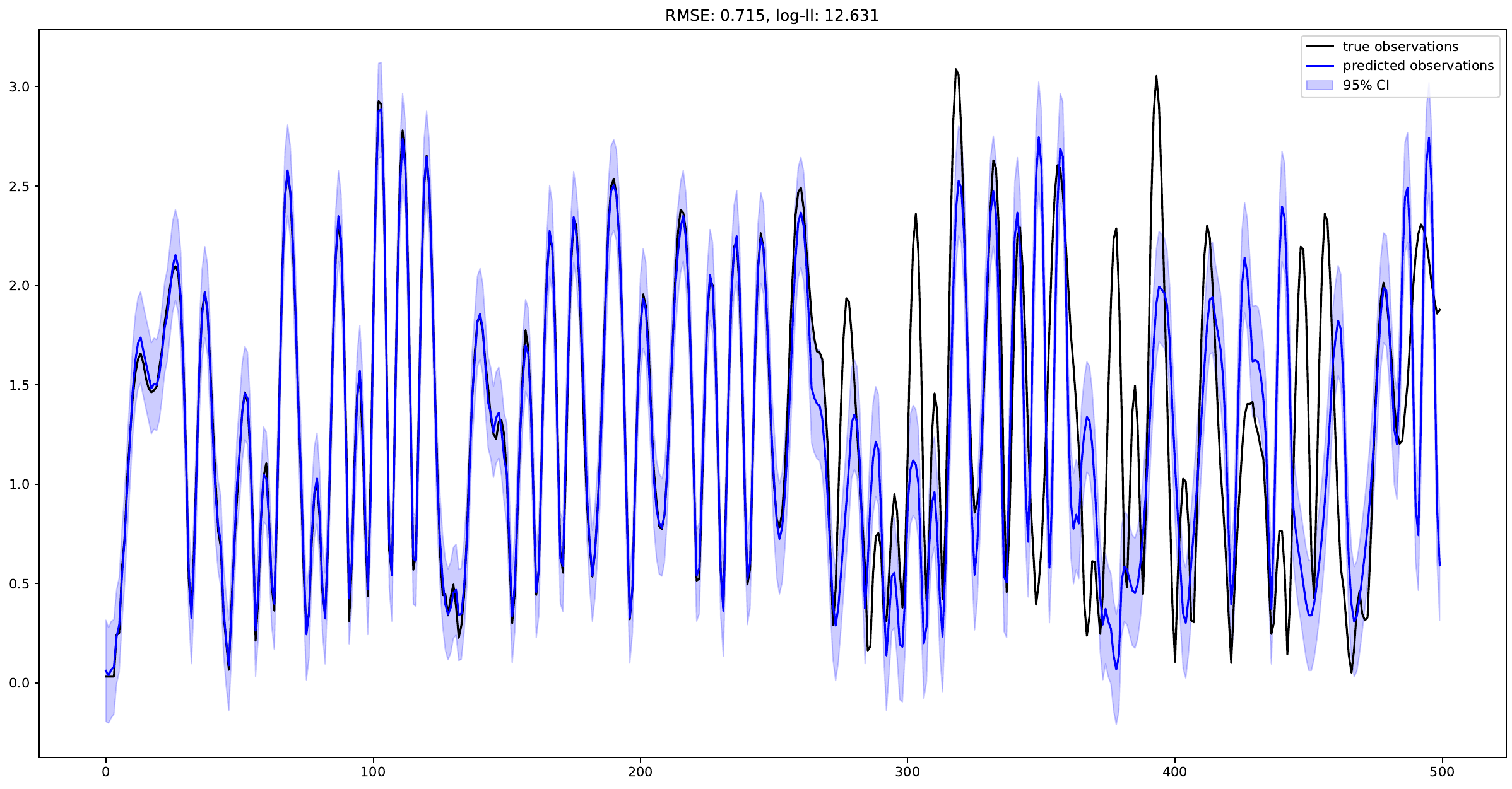}}
    \subfloat[Dryer]{\includegraphics[width=.5\linewidth]{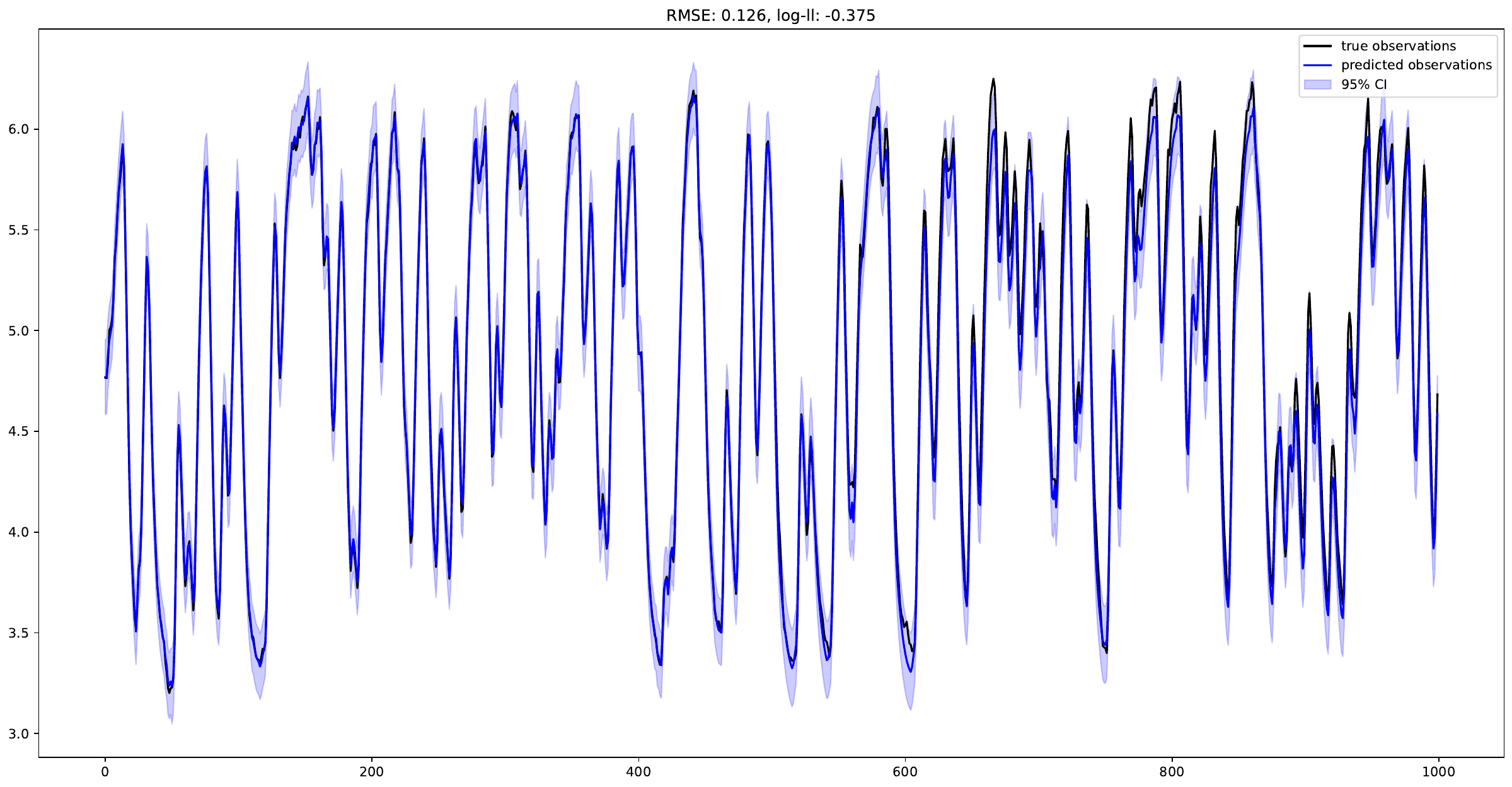}} 

    \subfloat[Gas Furnace]{\includegraphics[width=.5\linewidth]{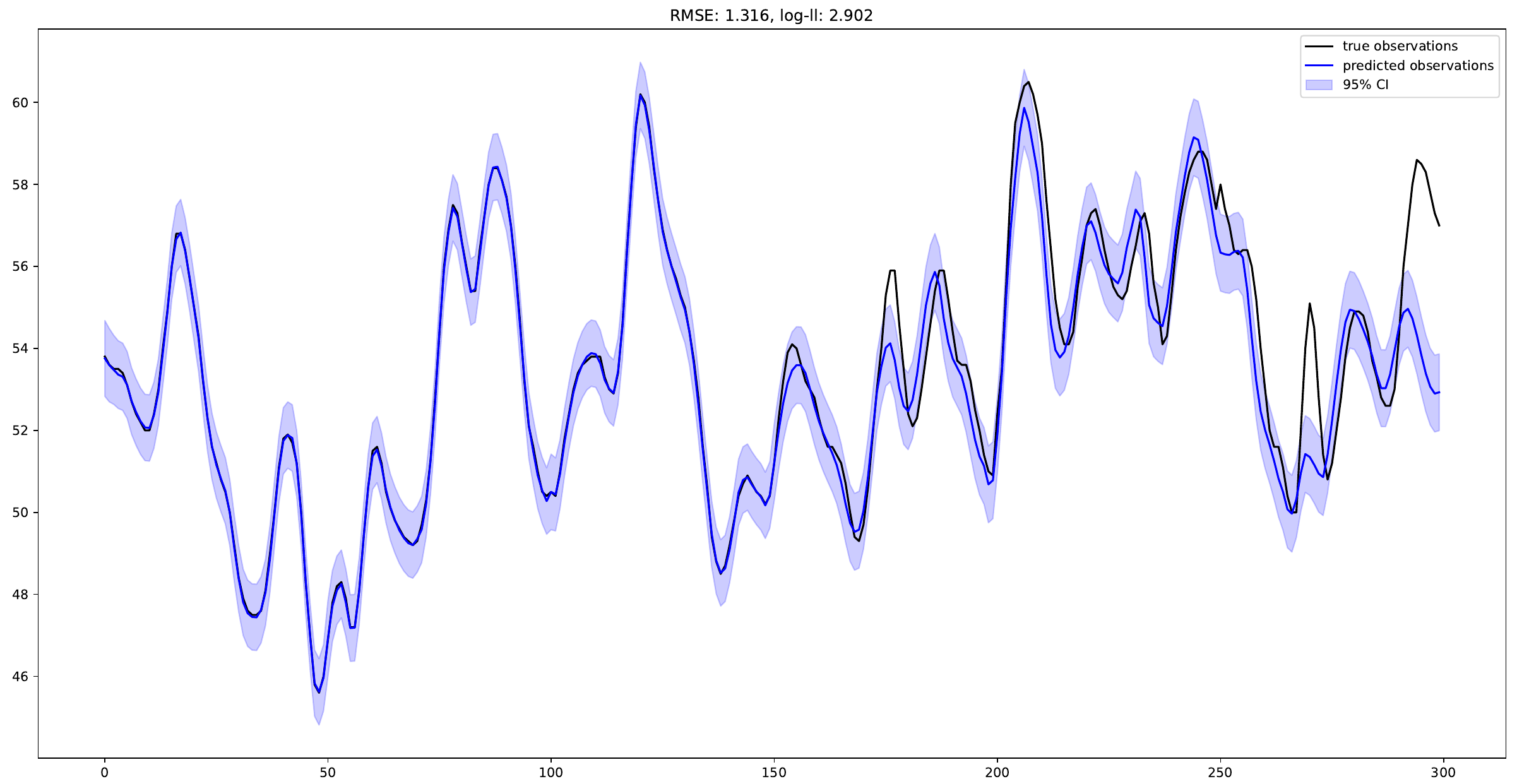}} 
    \caption{Learning and prediction results of EnVI. The initial half of the sequence is generated by passing the filtered $\x_t$ through the emission model, while the subsequent half of the sequence represents the prediction outcome. This prediction is derived from the filtered $\x_t$ obtained from the final step of the training sequence, serving as the initial state.}
\end{figure*}

\clearpage
\subsection{Online Learning and Inference with OEnVI} \label{supp_subsec:OEnVI_more}
We report the learning results of OEnVI on the kink function dataset. As illustrated in Fig. \ref{fig:kink-OEnVI}, after sequentially training with $600$ data points, OEnVI demonstrates comparable performance to EnVI, which undergoes offline training consisting of over $400$ iterations with a full batch of data with length $T\!=\!600$. 

\begin{figure}[ht!]
    \centering
    \includegraphics[width=0.5\linewidth]{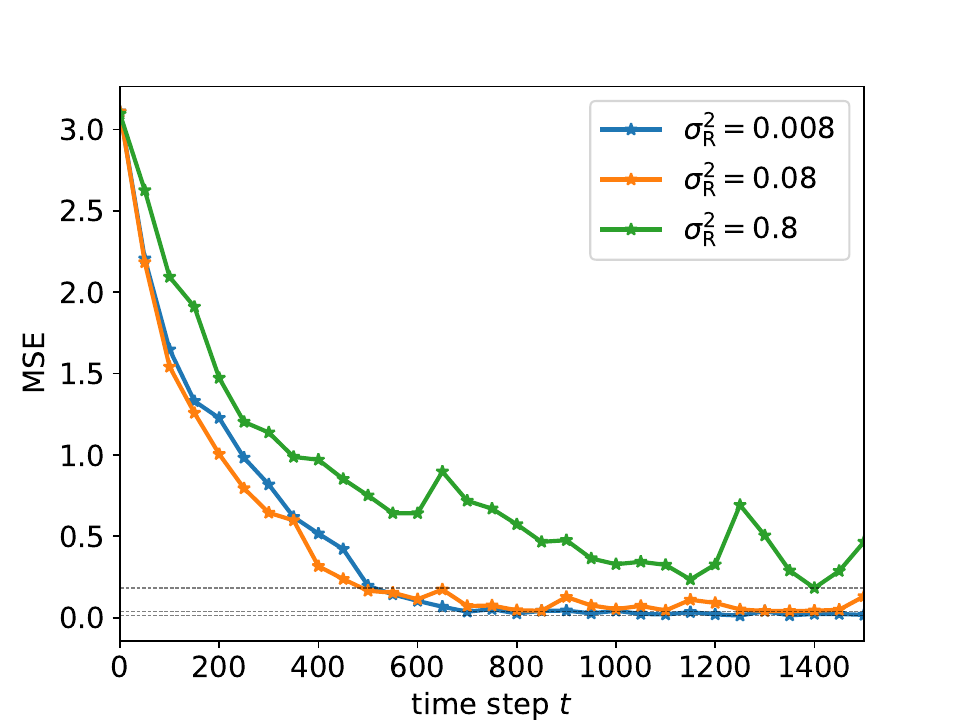}
    \caption{Kink transition function learning using OEnVI}
    \label{fig:kink-OEnVI}
\end{figure}

\end{document}